\documentclass{article}

\usepackage{mathrsfs,amsthm,amsmath,amssymb,bbm,pdfsync,prettyref,hyperref,graphicx}
\usepackage{epstopdf}
\usepackage{array}

\usepackage{fancyhdr,setspace}

\usepackage[marginratio=1:1,height=600pt,width=460pt,tmargin=100pt]{geometry}
\usepackage[compress]{cite}

\usepackage{authblk}

\usepackage{algorithm2e,tikz}
\usepackage{algpseudocode}

\usepackage{hyperref}
\hypersetup{
  colorlinks = false,
}

\usepackage[caption=false]{subfig}

\usepackage{stackengine}
\usepackage{multirow}

\newcolumntype{C}[1]{>{\centering\arraybackslash$}p{#1}<{$}}

\newcommand{\R}{\mathbb{R}}

\newcommand{\Z}{\mathbb{Z}}

\newcommand{\E}{\mathbb{E}}

\newtheorem{theorem}{Theorem}[section]
\newtheorem{assump}{Assumption}

\newtheorem{lemma}[theorem]{Lemma}
\newtheorem{corollary}[theorem]{Corollary}
\newtheorem{proposition}[theorem]{Proposition}

\title{
Classification Logit Two-sample Testing by Neural Networks
}

\author{
Xiuyuan Cheng\thanks{Department of Mathematics, Duke University. Email: xiuyuan.cheng@duke.edu}
~~~~~~~~~~~~~~~~
Alexander Cloninger\thanks{Department of Mathematics and Halicio{\u g}lu Data Science Institute, University of California at San Diego. Email: acloninger@ucsd.edu}
}

\date{\vspace{-20pt}}

\begin{document}

\maketitle

\begin{abstract}
The recent success of generative adversarial networks and variational learning suggests training a classifier network may work well in addressing the classical two-sample problem. 
Network-based tests have the computational advantage that the algorithm scales to large samples. 
This paper proposes a two-sample statistic
which is the difference of the logit function, provided by a trained classification neural network, evaluated on the testing set split of the two datasets. 
Theoretically, we prove the testing power to differentiate two
sub-exponential densities given that the network is sufficiently parametrized.
When the two densities lie on or near to low-dimensional manifolds embedded in possibly high-dimensional space, the needed network complexity is reduced to only
scale with the intrinsic dimensionality.
Both the approximation and estimation error analysis are based on
a new result of near-manifold integral approximation.
In experiments, the proposed method demonstrates better performance than previous network-based tests using classification accuracy as the two-sample statistic, and compares favorably to certain kernel maximum mean discrepancy tests on synthetic datasets and hand-written digit datasets.
  
\end{abstract}

\section{Introduction}

The two-sample test problem studies the comparison of two unknown distributions $p$ and $q$ from finitely observed data samples 
\cite{lehmann2006testing}.
As a central problem in statistics, it is widely encountered in general data analysis of biomedical data, audio and imaging data, etc. 
\cite{borgwardt2006integrating, chwialkowski2015fast, jitkrittum2016interpretable,  lopez2016revisiting,cheng2017two},
and particularly, in the machine learning application 
of training and evaluating generative models
\cite{li2015generative,li2017mmd, goodfellow2014generative,arjovsky2017wasserstein,nowozin2016fgan,
lloyd2015statistical,sutherland2016generative,chwialkowski2016kernel,liu2016kernelized,jitkrittum2017linear}.
Many existing tests are based on certain estimators of a distance or divergence 
between $p$ and $q$. 
Important examples include Maximum Mean Discrepancy (MMD), 
especially
kernel-based MMD \cite{anderson1994two,gretton2012kernel}
and distance of Reproducing Kernel Hilbert Space (RKHS) Mean Embedding \cite{chwialkowski2015fast,jitkrittum2016interpretable},
and divergence based methods 
which may involve non-parametric estimation of density difference
or density ratio 
 \cite{sriperumbudur2009integral, wornowizki2016two, sugiyama2013density, kanamori2012f}.
While these methods have been intensively studied and theoretically well-understood,
their application is often 
restricted to data with small dimensionality and/or small sample size due to model and computational limitations.

The powerful expressiveness of neural networks and the recent progress in neural network optimization
suggest the natural idea of using a network  for the two-sample problem \cite{lopez2016revisiting}. 
In the training of generative adversarial networks (GAN) \cite{goodfellow2014generative,arjovsky2017wasserstein, nowozin2016fgan},
in each iteration a discriminative network (D-net) is trained to distinguish 
the density $q$ produced by a generative network 
from the data density $p$ which is only accessible via observed samples, that is, a two-sample problem.
Strictly speaking, the task of D-net is a goodness-of-fit problem as the model density $q$ is analytically given
\cite{chwialkowski2016kernel,liu2016kernelized,jitkrittum2017linear}. 
Since batch sampling is commonly used in training GAN and other generative networks, 
the ability of a trained network, such as the D-net in GAN,
to detect the difference between two densities from finite samples is crucial for such applications.

Recent work has also shown that the D-net does not have to be arbitrarily large in order to approximate complicated functions, and that the architecture does not have to scale with the number of points in order to bound the generalization error \cite{schmidt2019deep}, or for constructing GANs \cite{chen2020statistical}.
These works approach neural networks from an approximation theoretic perspective and yield constructive approximations of smooth functions.
Other important directions in this line include studying the effects of data dimensionality \cite{shaham2018provable, chen2019nonparametric}, the effects of depth of the network \cite{yarotsky2017error,bolcskei2019optimal}, and showing that the space of fixed size networks is not closed under $L^p$ norm \cite{petersen2018topological}.
High dimensional data that are lying near low-dimensional structures have not been considered in these works. 
We review and discuss all these connections in more detail in Section \ref{sec:review}.

The current paper studies the method of training a neural network for the two-sample problem,
and the proposed test statistic is the log ratio of the class probabilities averaged over samples,
 which can be computed once a classifier network is trained.
Our contributions include:

\begin{itemize}
    \item We introduce a network-based two-sample test statistic using classification logit, and the algorithm inherits the scalable computational efficiency of neural networks.
    
    \item
    Theoretical guarantee of testing power is proved for sub-exponential densities $p$ and $q$ in $\R^D$, 
and the needed network complexity is reduced to be intrinsic when $p$ and $q$ lie on or near to a low-dimensional manifold embedded in $\R^D$ even when $D$ is high.
A key part of the analysis is a result of near-manifold integral approximation, and we give that in a  more general form which can be of independent interest.
    
    \item
    Numerical experiments show that the proposed test compares favorably to kernel MMD tests and earlier neural network test based on classification accuracy, on both synthetic manifold data and hand-written digits datasets. 
\end{itemize}

The softmax classification loss that we use corresponds to the Jensen-Shannon divergence (JSD) between two densities, which belongs to a general class of $f$-divergences.
Thus our method and techniques can extend to a broad class of classification networks \cite{nowozin2016fgan}.
Since JSD is a prototypical case of $f$-divergence, we focus on softmax classifier network in this paper.

\subsection{Related works}\label{sec:review}

{\bf Classification and two-sample testing}.
The relations between two-sample testing, 
divergence estimation and binary classification has been pointed out in earlier statistical literature \cite{friedman2004multivariate, sriperumbudur2009integral,reid2011information}.
\cite{ramdas2016classification} studied  Fisher LDA classifier used for testing mean shift of Gaussian distributions. 
Discriminative approach has also been used to detect and correct covariant shifts
\cite{bickel2007discriminative,bickel2009discriminative}.
Training classifier provides an estimator of density ratio, 
as has been pointed out in \cite{menon2016linking} and in the formulation of learning generative models \cite{mohamed2016learning}. 
While distribution divergence estimation has been studied and used for two-sample problems \cite{kanamori2011f,kanamori2012f,sugiyama2013density, wornowizki2016two},
the use of neural network as a divergence estimator for two-sample testing was less investigated.
In terms of theoretical guarantee of test power,
the analysis in \cite{lopez2016revisiting} 
assumes a non-zero population test statistic when $p\neq q$
but the expression is not specified, 
along with other approximations.
Theoretical analysis of neural network two-sample testing power remains limited.

{\bf MMD and kernel MMD tests.}
MMD,
also known as the Integral Probability Metric (IPM),
encloses a wide class of two-sample statistics 
such as Kolmogorov-Smirnov statistic, 
Wasserstein metric, etc. 
Particularly, kernel-based MMD \cite{anderson1994two,gretton2012kernel}
has been widely applied due to its non-parametric form, 
and recently in training moment matching networks (MMN) \cite{li2015generative,li2017mmd}
and evaluating generative models \cite{sutherland2016generative}. 
To optimize kernel choice, 
\cite{gretton2012kernel} considers selecting kernel bandwidth from data,
\cite{cheng2017two} studies anisotropic kernels.
Optimizations of kernel through 
convex combination of multiple kernels \cite{gretton2012optimal},
adapting reference locations in the Mean Embedding test \cite{jitkrittum2016interpretable},
and neural network parametrization \cite{liu2020learning}
have been introduced which maximize estimated testing power. 
The combination of neural network feature learning and kernel MMD has been studied in \cite{li2017mmd},
where the training is typically more costly than that of a classifier network.
Compared to kernel methods, neural networks are often algorithmically more scalable,
and the current paper studies the theoretical guarantee of testing power,
and compares performance in practice.

{\bf Relation to goodness-of-fit test and GAN.}
The goodness-of-fit problem differs from the two-sample problem in that one of the two densities is analytically accessible. 
Using the explicit formula of $q$, methods based on kernel Stein discrepancy have been developed in
 \cite{chwialkowski2016kernel,liu2016kernelized,jitkrittum2017linear}
 and applied to generative model evaluation.
 However, the computation of the {\it score function} $\nabla \log q$ may be difficult for certain generative models, 
including many generative networks. 
Meanwhile, in many generative models including MMN and GAN the goodness-of-fit is evaluated by batch sampling, i.e. the two-sample setting:
Kernel MMD is used in MMN,
and GAN, Wasserstein GAN \cite{arjovsky2017wasserstein} and $f$-GAN \cite{nowozin2016fgan}
estimate density divergence by a trained network (the D-net).
Since the success of GAN training relies on the discriminative power of the D-net,
the efficiency of using a neural network for the two-sample test
is important  for the training and evaluating of such  models.

{\bf Approximation with neural networks.} The expressiveness of neural networks and their ability to approximate functions to error $\epsilon$ considers the minimum architecture needed to approximate a family of functions, independent of whether an optimization scheme converges to the proposed constructed network.
The number of parameters needed in a network to bound the point-wise error by $\epsilon$ is known to scale as $\epsilon^{-D/r}$
where $D$ is dimension of the space and $r$ is the smoothness of the target function \cite{mhaskar1996neural, yarotsky2017error, shen2019nonlinear}.  It has also been established that if the function is in a Korobov-2 space (i.e., smooth mixed derivatives up to order  
${\partial^2_{x_1} ... \partial^2_{x_D}} f$), 
then the complexity can be reduced to $\epsilon^{-1/2}|\log(\epsilon)|^{3D/2}$ \cite{montanelli2019new}.
Similarly, when the data lies on a lower dimensional manifold of dimension $d<D$, the complexity of the networks scales as $\epsilon^{-d/2}$ for twice differentiable functions when the network depth is bounded \cite{shaham2018provable, schmidt2019deep, shen2019deep},
and regression of H\"older functions on such manifolds using deep ReLU networks
was studied in \cite{chen2019nonparametric} which proved estimation convergence rate depending on $d$.  These methods are mostly applied to regression problems, and consider data distributions that are absolutely continuous in $\mathbb{R}^D$ or lying on a low-dimensional manifold, but have not considered the case where data is concentrated near low-dimensional structures.

\section{Log-ratio test by network classifier}

\subsection{Preliminary: two-sample problem and MMD test}

Formally, the two-sample problem asks to test the null hypothesis 
${\cal H }_0: p  = q$ 
given datasets \[
X=\{x_i\}_{i=1}^{n_X}, \quad x_i \sim p ~ i.i.d., 
\quad 
Y=\{y_j\}_{j=1}^{n_Y}, \quad y_j \sim q ~ i.i.d.,
\]
and $X$ independent from $Y$. 
It is also of application interest to provide indication of where $q$ differs from $p$. 
Similarly, $K$-sample problem can be considered where $K > 2$, namely to differentiate distributions among
$K$ data sets.
We focus on the two-sample problem in the paper, and the neural network classifier method extends to the $K > 2$ situations.
Throughout the paper, we assume that the distributions have density functions, and use $p$ and $q$ to stand for both the distribution and the density function.

Most two-sample testing method is based upon a statistic \[
\hat{T}= \hat{T}(X,Y),
\]
 which is computed from the two datasets,
and a test threshold $\tau$.
The null hypothesis ${\cal H }_0$ is rejected if $\hat{T} > \tau$.
To control the false discovery under the null,
the threshold $\tau$ is usually set to the smallest value s.t. $\Pr [  \hat{T} > \tau | {\cal H }_0 ] \le \alpha$, 
where $\alpha \in (0,1)$ is a pre-specified number called the {\it significance level} of the test
(typically $\alpha = 0.05$).
Algorithm-wise, $\tau$ is given either by theory (the probabilistic distribution of $\hat{T}$ under ${\cal H }_0$) 
or computed from data.
Specifically, 
{\it permutation test} is a standard procedure to determine $\tau$ \cite{higgins2003introduction,gretton2012kernel}.
For a given test, the {\it test power} is measured by 
 $\Pr [  \hat{T} > \tau | {\cal H }_1 ] $,
 namely the probability of true discovery when $p$ and $q$ indeed differ. 
 The test is called asymptotically consistent if the test power $\to 1$ as the number of samples $n_X, n_Y \to \infty$
 and usually the ratio $\frac{n_X}{n_Y} \to$ a nonzero constant.

A class of widely used two-sample statistics is the MMD test, also named as IPM,
the population statistic of which takes the general form as 
\begin{equation}\label{eq:def-MMD}
D(p,q) = \sup_{f \in{ \cal F}} \int f (p-q),
\end{equation}
where ${\cal F}$ is certain restricted family of functions. 
E.g., in kernel MMD, ${\cal F}$ is the $L^2$-unit ball in the RKHS.
Many finite-sample estimators of \eqref{eq:def-MMD} have been developed.
The proposed network logit test is not an MMD,
because the training objective is an (empirical) $f$-divergence rather than an IPM. 
However, the test statistic resembles the form of MMD evaluated at the sup $f$ in $\cal F$.
We discuss the connection in Section \ref{subsec:witness}, and compare with kernel MMD tests in experiments.

\subsection{Classification logit test statistic}\label{subsec:stat-T}

The proposed test statistic is computed in the following way.
Given $X$ and $Y$ as above, without loss of generality suppose $n = n_X+n_Y$ is even integer.
Same as in \cite{lopez2016revisiting},
we split the dataset 
$ {\cal D} = \{ (x_i, 0) \}_{i=1}^{n_X} \cup \{ (y_j, 1) \}_{j=1}^{n_Y} = \{ (z_i, l_i) \}_{i=1}^n$, $l_i \in \{0,1\}$, 
 into two halves, 
 i.e. ${\cal D}  = {\cal D}_{\text{tr}} \cup  {\cal D}_{\text{te}}  $, $|{\cal D}_{\text{tr}}| = |{\cal D}_{\text{te}} |= \frac{n}{2}$,
 ${\cal D}_{\text{te}} = X_{\text{te}} \cup  Y_{\text{te}}$ and similarly for the training set. 
The method consists of two phases of training and testing:

\begin{itemize}
\item Training. 
 A binary classification neural network  is trained on $ {\cal D}_{\text{tr}}$ using softmax loss
 (equivalent to applying logistic regression to the output of last layer before the loss layer), 
 which gives estimated class probabilities as 
 \[
 \Pr [l= 0 | z] = \frac{ e^{u_\theta(z)}}{e^{u_\theta(z)} + e^{v_\theta(z)}},
\quad 
\Pr [l = 1 | z] = \frac{ e^{v_\theta(z)}}{e^{u_\theta(z)} + e^{v_\theta(z)}},
\] 
where $u_\theta(z)$ and $v_\theta(z)$ are activations in the last hidden layer of the network, $\theta$ denoting the network parametrization. 
We define \[
f_\theta := u_\theta- v_\theta, \] 
which is the {\it logit}.
The mathematical formulation of network training is detailed in Section \ref{subsec:training-opt}.

\item Testing.
After $f_\theta$ is parametrized by a trained neural network,
 the test statistic is computed as
\begin{equation}\label{eq:def-hatT-logratio}
\hat{T} 
= \frac{1}{ |X_{\text{te}}|} \sum_{x\in X_{\text{te}}} f_\theta(x) -  \frac{1}{ |Y_{\text{te}}| }\sum_{y \in Y_{\text{te}}} f_\theta(y),
\end{equation}
which can be equivalently written as 
$
\hat{T}  = \int f_\theta(x) (\hat{p}_\text{te}(x) - \hat{q}_\text{te}(x) )dx
$
where $\hat{p}_\text{te}$ and $\hat{q}_\text{te}$ stand for the empirical measure of $X_{\text{te}}$ and $Y_{\text{te}}$ respectively.

{\bf Determination of $\tau$}.
Once the logit function $f_\theta$ is evaluated on each testing sample,
 the test threshold $\tau$ can be computed by a {\it permutation test} 
\cite{higgins2003introduction}:
randomly permute the $ |{\cal D}_{\text{te}} |$ many labels $l_i$ on the test set,
and recompute the test statistics for $m_{\text{perm}}$ times, typically a few hundreds.
Then $\tau$  is set to be the (1-$\alpha$)-quantile of the empirical distribution
so as to control the type-I error to be at most $\alpha$. Note that this permutation test does not require retraining the network upon each permutation 
nor re-evaluation of the neural network on testing samples.

\end{itemize}

The above classifier logit test can be used with other classifiers than neural networks, 
e.g., logistic regression,
which is equivalent to restricting $f_\theta$ to be a linear mapping or the network to have only one linear layer.
A main advantage of neural network is the enlarged expressiveness of the class of functions $f_\theta$ that can be represented or approximated.

{\bf Computational complexity.}
Given $n$ data samples,  
evaluating the network output on each sample takes a fixed amount of flops,
and thus computing the test statistic takes $O(n)$ operations.
The permutation test to determine $\tau$ adds negligible cost since  $f_\theta$  has been evaluated at each test sample,
and permuting the class labels only reorders these computed values.
The training phase is certainly more expensive,
though theoretically the overall complexity is $O(n)$ 
assuming that training is terminated after a fixed number of epochs.
Note that the computation can be conducted by  batch sampling 
so the algorithm scales to large sample size and also to multiple sample problems.

\subsection{Witness function}\label{subsec:witness}

Given the logit function $f$, the empirical test statistic is written as $\hat{T}=\int f(\hat{p}-\hat{q})$,
the subscripts being omitted without causing confusion,
and the population test statistic is 
\begin{equation}\label{eq:def-Tf}
T[f] := \int f (p-q),
\end{equation}
which is of the same form as the MMD discrepancy \eqref{eq:def-MMD} evaluated at the $\sup$-achieved $f$. 
In the literature of kernel MMD \cite{gretton2012kernel},
such $f$ is named the {\it witness function}, 
as it provides an indication of where $q$ differs from $p$.  
The indicator of differential regions can be of more application interest than the hypothesis test itself.
The witness function for kernel MMD is expressed via the reproducing kernel. 

In our setting, the density differential indicator is provided by the logit function $f_\theta$ of the trained classifier network.
We follow the MMD literature and  call $f_\theta$
 the (empirical) witness function of the proposed logit test.
As has been pointed out in \cite{menon2016linking},
the training of the classifier can be interpreted as minimizing a Bregman divergence between the estimated logit $f_\theta$
and the true log ratio $f^* = \log \frac{p}{q}$. 
We will see in Section \ref{subsec:training-opt} that $f^*$ is also the unconstraint optimizer of the population training objective.
Thus 
the trained $f_\theta$ serves as an estimator of density log ratio $f^*$, and
the proposed statistic $\hat{T}$ can also be viewed as estimating the symmetric KL divergence
$ \int (p-q) \log \frac{p}{q} = \text{KL} (p || q) + \text{KL} (q || p)$.
We call $f^*$
the population witness function of the logit test.

The witness function plays an important role in the ability of the test to distinguish two densities. 
When $p \neq q$, once a witness function $f$ with $T[f] > 0$ is obtained (from the training set),
the two-sample test (on the test set) using $\hat{T}$ will be asymptotically consistent
 as a direct result of Central Limit Theorem (CLT). 
The difference in test power thus depends on the quality of $f$, e.g., how large is the bias $T[f]$ compared to the variance of $\hat{T}$.
Consequently,
the efficiency of the network classification test lies in how well the neural network can express a good witness function and how it can be identified via the optimization, which is the central question of our analysis. 
Apart from theory, 
we also experimentally compare the witness function of different MMD tests in Section \ref{sec:experiment}.

\subsection{Identification of $f_\theta$ by neural network training}\label{subsec:training-opt}

Mathematically, the training of classification neural network optimizes the following objective 
\begin{equation}\label{eq:training-loss-unnormalized}
\sum_{x \in X_{\text{tr}}}
 \log D(x) 
 + 
\sum_{y \in Y_{\text{tr}}}
 \log (1-D(y)),
  \quad 
 D(x):= \frac{e^{f(x)}}{1+e^{f(x)}},
\end{equation}
called the (negative) empirical training loss. 
\footnote{``Loss'' usually refers to minimization, in this paper we use $L$ and $L_n$ to denote the population and empirical negative softmax loss which is to maximize by the optimization. 
}
After normalizing by number of samples,
where we assume same of samples in $X$ and $Y$, and then $|X_{\text{tr}}| = |Y_{\text{tr}}|$, for simplicity throughout the paper,
the optimization of the empirical loss (up to a additive constant) can be written as
\begin{eqnarray}
& &  \max_{f\in{\cal F}_{\Theta}}
L_{n, \text{tr}} [f]
 =  \frac{1}{2}  \left(
\frac{1}{|X_{\text{tr}}|} \sum_{x \in X_{\text{tr}}}
 \log D(x) 
 + 
 \frac{1}{|Y_{\text{tr}}|} \sum_{y \in Y_{\text{tr}}}
 \log (1-D(y))
+ 2 \log 2
 \right) \nonumber
 \\ 
& &~~~~
 =   \frac{1}{2}  \left(
 \int  \hat{p}_{\text{tr}}(x) 
 \log D(x) dx
 + \int  \hat{q}_{\text{tr}} (x) \log (1-D(x))
+ 2 \log 2
 \right), 
 \label{eq:training-loss-finite-sample}
\end{eqnarray}
where ${\cal F}_{\Theta}$ denotes the class of functions that can be expressed as the difference of the outputs in the last hidden layer of the classification network,
and 
$\hat{p}_\text{tr}$ and $\hat{q}_\text{tr}$ stand for the empirical measure of $X_{\text{tr}}$ and $Y_{\text{tr}}$ respectively.

This training objective is the same as that of the the D-net in the standard GAN\cite{goodfellow2014generative}.
As number of samples $n \to \infty$, the corresponding population training loss can be expressed as
\begin{equation}\label{eq:training-loss-population}
L[f]= \frac{1}{2} \left(
 \int p\log\frac{2e^{f}}{1+e^{f}}+\int q\log\frac{2}{1+e^{f}} \right),
\end{equation}
and 
a direct verification shows that the maximizer is (see e.g.  \cite{goodfellow2014generative} where it is proved in terms of $D=\frac{e^{f}}{1+e^{f}}$)
\begin{equation}\label{eq:fstar-free}
f^{*}=\arg\max_{f}L[f]=\log\frac{p}{q},
\quad
L[f^{*}]=  \frac{1}{2}
\left(  \int p\log\frac{2p}{p+q}+\int q\log\frac{2q}{p+q} \right)
= \text{JSD}(p,q)
\end{equation}
which characterizes the solution of \eqref{eq:training-loss-population} when the
function class is arbitrarily large or large enough to contains $f^{*}$,
JSD referring to the Jensen-Shannon divergence. 

Note that the $f_\theta$ identified in practice, 
which we call $\hat{f}_{tr} \in {\cal F}_{\Theta}$ (``tr'' for ``trained''),
differs from $f^*$ for three reasons:

\begin{itemize}
\item
(Approximation error) 
The neural network function class ${\cal F}_{\Theta}$ is of finite complexity and may not contain $f^*$.

\item
(Estimation error)
Only finite training samples are used, which makes $L_{n, \text{tr}} \neq L$.

\item
(Optimization error)
The optimization of $L_{n,\text{tr}}[f]$ may attain a local 
rather than  global  optimum. We call the global optimum $\hat{f}_{\text{gl}}$ (``gl'' for ``global'').
\end{itemize}

The goal of analysis is thus to prove that the logit test  \eqref{eq:def-hatT-logratio} using $\hat{f}_{\text{tr}}$ 
obtained from training on the training set
can distinguish different $p$ and $q$ with efficiency.

\begin{figure}[t]
\centering{
\includegraphics[trim={40pt 260pt 50pt 190pt},clip,height=.28\linewidth]{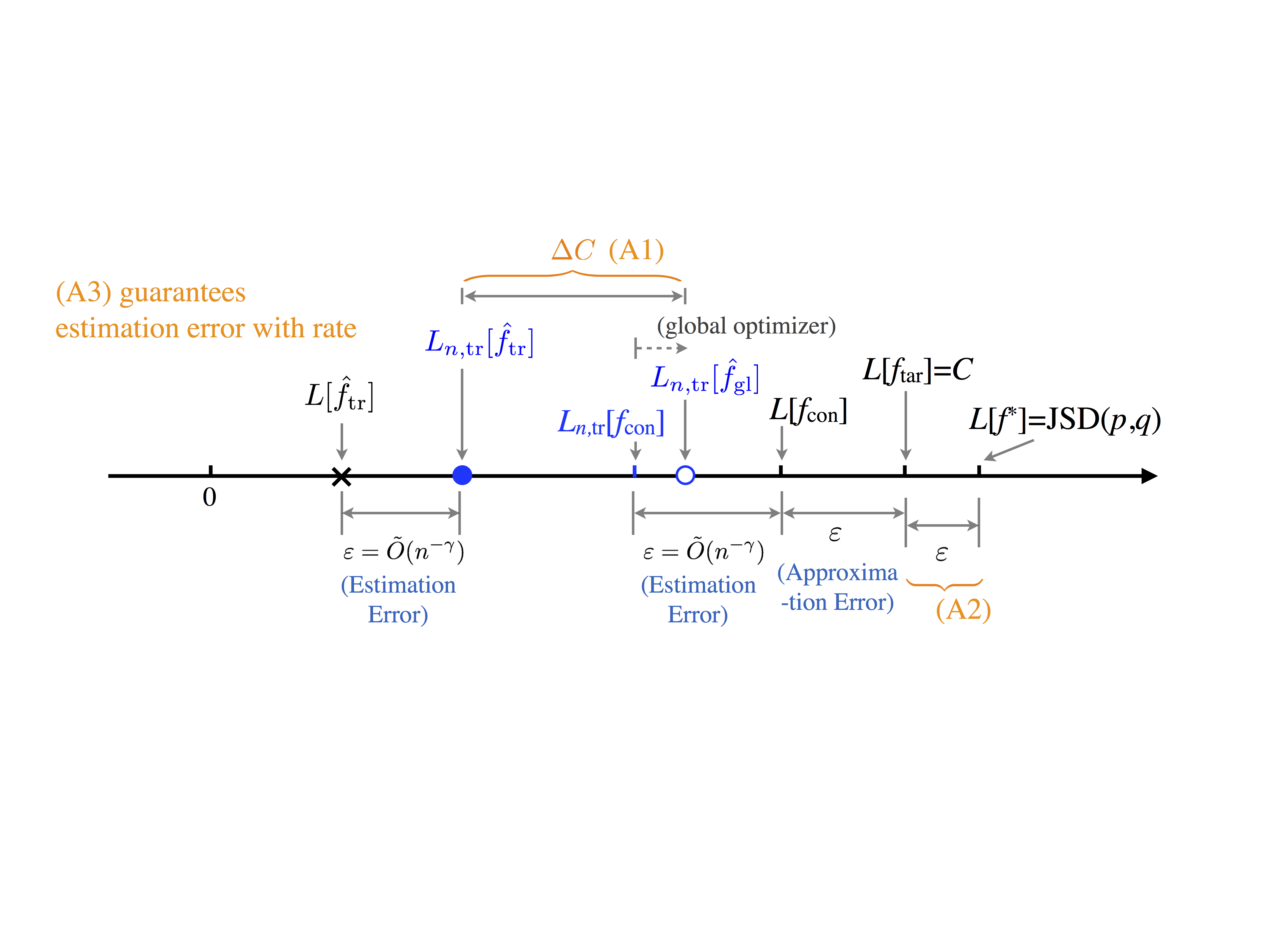}
}
\vspace{-10pt}
\caption{
Roadmap of analysis, illustrating the different functions $f^*$, $f_{\text{tar}}$,
$\hat{f}_{\text{gl}}$, $\hat{f}_{\text{tr}}$
and $f_{\text{con}}$.
The last three lie in the network function family ${\cal F}_\Theta$.
}
\label{fig:diag1}
\vspace{-10pt}
\end{figure}

\subsection{Roadmap of analysis}\label{subsec:roadmap}

We illustrate the differences among $\hat{f}_{tr}$,
$\hat{f}_{gl}$, and $f^*$ in the diagram in Fig. \ref{fig:diag1},
which also gives the roadmap of our analysis. 
Specifically,
we introduce $f_{\text{con}} \in {\cal F}_\Theta$
which approximates a target witness function $f_{\text{tar}}$ that carries a sufficiently large $L[f_{\text{tar}}]$.
The building blocks are as follows, where we use $\varepsilon$ to stand for a generic small number which may differ in different places: 

Starting from $L[f^*] = \text{JSD} >0$, JSD standing for JSD$(p,q)$,
\begin{itemize}
\item[1.]
A small $ |L[f_{\text{tar}} ] - L[f^*]|$ assuming that $f^*$ has regularity or has a regular surrogate $f_{\text{tar}}$. 
This relies on the assumption on $p$ and $q$, c.f. Assumption \ref{assump:ftar}.
Then $ L[f_{\text{tar}}] > \text{JSD} - \varepsilon$.

\item[2.]
A small $| L[f_{\text{con}}] -L[f_{\text{tar}}] |$ when network complexity is sufficiently large.
This relies on the network approximation analysis, 
and we bound the needed network complexity 
to scale with the intrinsic dimension $d$
when $p$ and $q$ lie on or near to low-dimensional manifolds. 
Then $ L[f_{\text{con}}] > \text{JSD} - 2\varepsilon$.

\item[3.]
A small $| L[f_{\text{con}}] -  L_{n,\text{tr}}[f_{\text{con}}] |$
by concentration bound. 
Note that we will need to deal with the deviation of $L_{n,\text{tr}}[\hat{f}_{\text{tr}} ]$
later, so we bound $ \sup_{ f \in {\cal F}_{\Theta}} | L[f] -  L_{n,\text{tr}}[f]|$
based on a bound of the covering number of the network function family ${\cal F}_{\Theta}$. 
In particular,
for on-or-near manifold densities
 the convergence rate of the estimation error
 is improved to only involving the intrinsic dimensionality $d$.
This step gives $L_{n,\text{tr}}[f_{\text{con}}]  > \text{JSD} - 3\varepsilon$.

\item[4.]
The optimization gives that $ L_{n,\text{tr}}[\hat{f}_{\text{gl}}]  \ge L_{n,\text{tr}}[f_{\text{con}}] $.
To handle $\hat{f}_{\text{tr}}$,
we assume the following on network training
\begin{assump}\label{assump:training}
The neural network training of optimizing $\max_{f \in {\cal F}_\Theta}L_{n,\text{tr}}[f]$
outputs $\hat{f}_{\text{tr}}$ which achieves a training loss that is $\Delta C$-close to the global optimum, $\Delta C$ small enough,
namely, \[
L_{n,\text{tr}}[\hat{f}_{\text{gl}}] - \Delta C 
\le L_{n,\text{tr}}[\hat{f}_{\text{tr}}] \le  L_{n,\text{tr}}[\hat{f}_{\text{gl}}].
\]
If the training algorithm is stochastic, then the above holds with high probability. 
\end{assump}
Some recent neural network optimization literature supports this assumption \cite{lu2017depth,kawaguchi2019elimination}.
In certain settings different than ours, it is proved that $\Delta C$ can be made small,
especially when training with large samples and using over-parametrized networks \cite{mei2018mean,arora2019exact,arora2019fine}.
Our experiments in Section \ref{sec:1dexample} show that $\Delta C$ can be relatively small in practice.  
In the current paper we do not further analyze the optimization error.
With the above assumption, $L_{n,\text{tr}}[\hat{f}_{\text{tr}}] > \text{JSD} - \Delta C - 3\varepsilon$.

\item[5.]
By the bound in Step 3., 
$|  L[\hat{f}_{\text{tr}} ] - L_{n,\text{tr}}[\hat{f}_{\text{tr}} ]|$ is small. 
Then 
$L[\hat{f}_{\text{tr}}] > \text{JSD} - \Delta C - 4\varepsilon$.
\end{itemize}

The above steps derive a non-zero lower bound of $L[\hat{f}_{\text{tr}}] $,
which then leads to a non-zero population test statistic 
$T[\hat{f}_{\text{tr}}]$. 
The testing consistency is then proved by the asymptotic normality of the empirical statistic $\hat{T}$  \eqref{eq:def-hatT-logratio}
(on testing set, rather than training set) by CLT. 
Throughout the steps, an important new analysis is for the near-manifold densities,
which we detail in Section \ref{sec:theory-approx} in a more general form.
The steps towards proving the test consistency are carried out in Section \ref{sec:theory-consist}.

\section{Example of Data on $\R$ and 1D Manifold}\label{sec:1dexample}

\begin{figure}[t]
\begin{center}
		\begin{tabular}{ c c}
		        \hspace{-1.75cm}
		        
		         \raisebox{0.75cm}{
			 \stackunder[5pt]{
			\includegraphics[height=.23\textwidth]{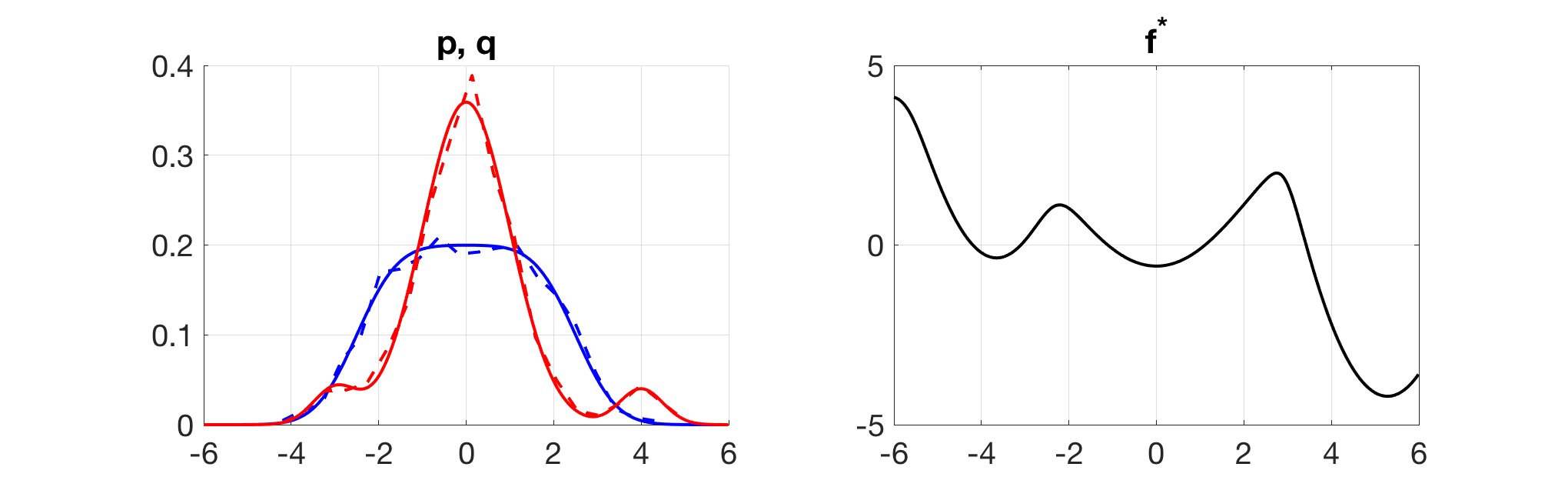}
			}{Example 1. Data in $\R$.}
			}
			 & 
			\hspace{-1.5cm}
			\includegraphics[height=.3\textwidth]{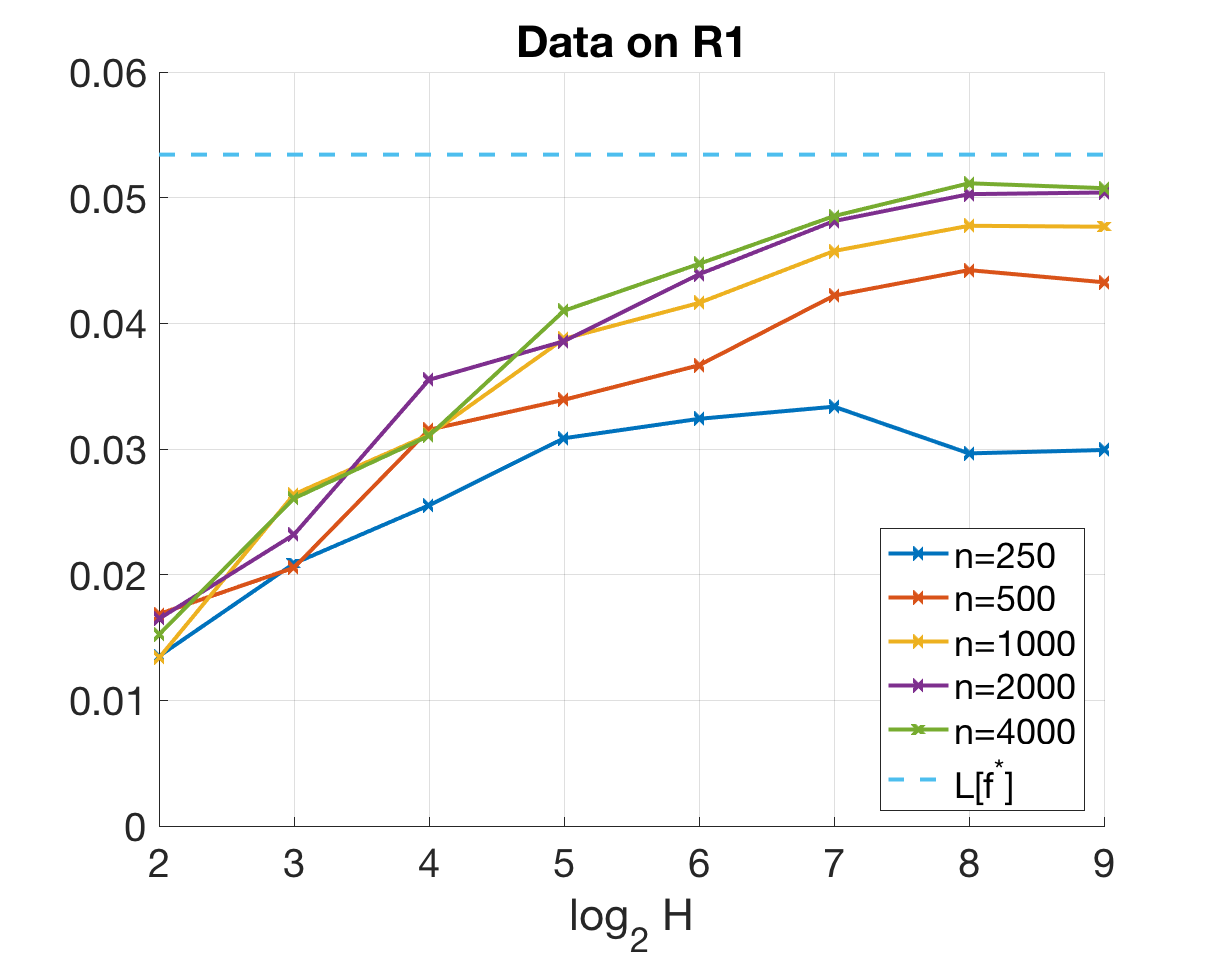} 
			 \\
			  \hspace{-1.75cm}
			 
			 \raisebox{0.75cm}{
			 \stackunder[5pt]{
			 \includegraphics[height=.23\textwidth]{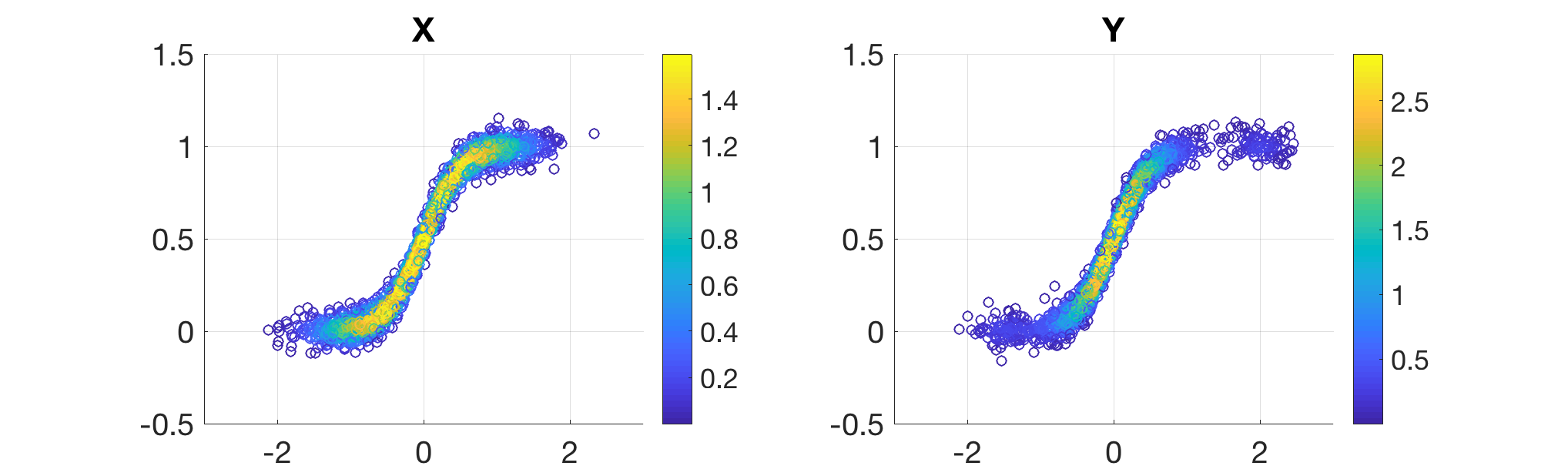}
			 }{Example 2. Data near a 1D manifold in $\R^2$}
			}			
			 & 
			\hspace{-1.5cm}
			\includegraphics[height=.3\textwidth]{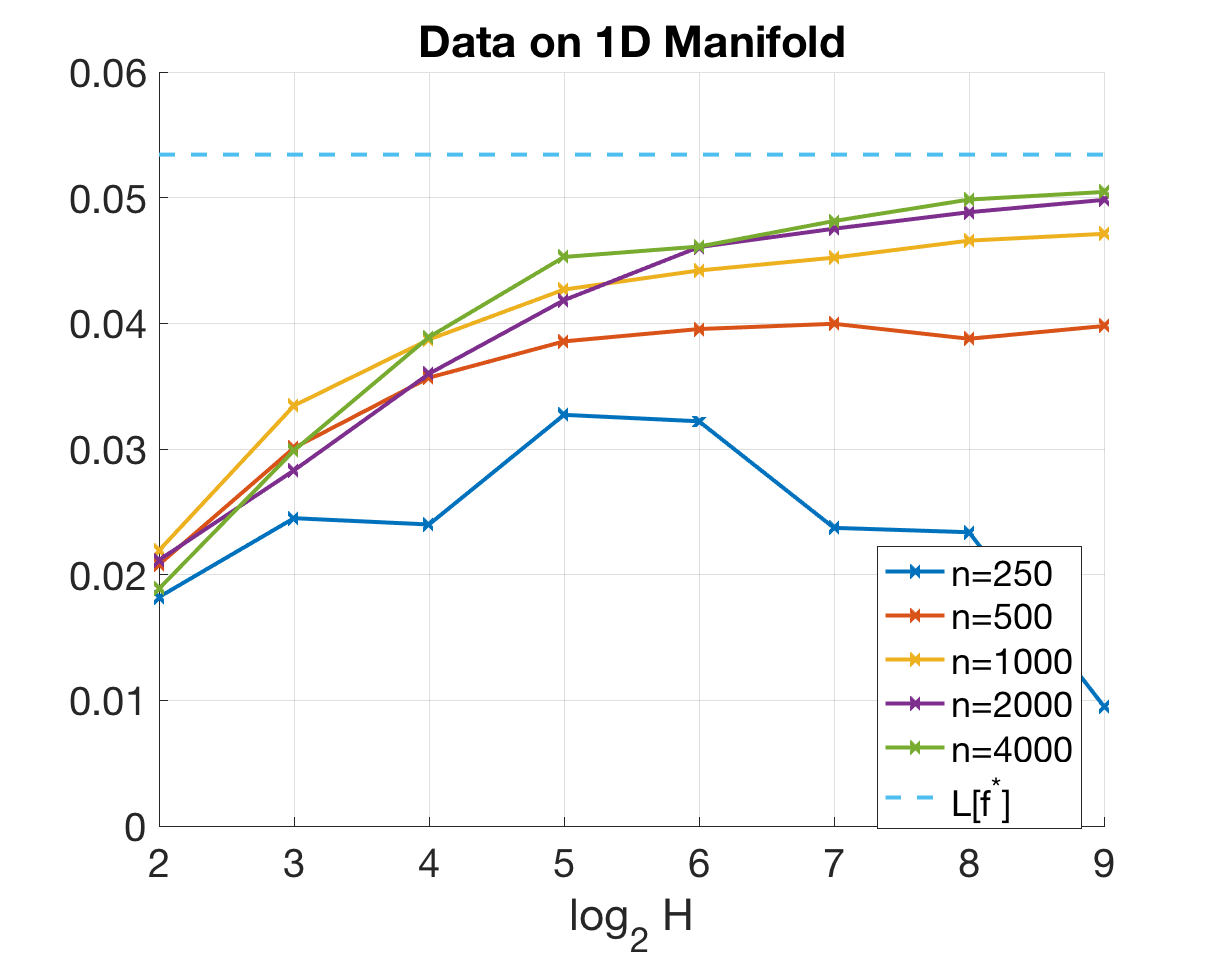} 
		\end{tabular}
\end{center}
\vspace{-20pt}
\caption{
(Left) Datasets $X \sim p$ and $Y \sim q$ of the two examples in Section \ref{sec:1dexample}.
In the plot of $p$, $q$ of Example 1, the dashed line shows the histogram of samples in $X$ and $Y$ respectively.
(Right) 
The values of $L[\hat{f}_\text{tr}]$ (averaged over 40 replicas of experiments) plotted against $log_2(H)$,
where $H$ is the width of the hidden layer in the classification neural network,
$H$ indicating the complexity of ${\cal F}_\Theta$.
The dashed line shows $L[f^*] = \text{JSD}(p,q)$, 
which is the unconstraint maximizer of $L$ as in \eqref{eq:training-loss-population}.
} 
\label{fig:eg-1d}
\vspace{-5pt}
\end{figure}

This section presents experiments on two sets of 1D data, 
in $\R$, and in $\R^2$ but lying near to a 1D curve, respectively. 
The experiments are set to verify the Assumption \ref{assump:training} 
and to observe the influence of the intrinsic dimensionality of data.
More experiments of two-sample testing are reported in Section \ref{sec:experiment}.

{\bf Datasets, training and evaluaton.}
The densities are Gaussian mixtures and have analytical formulas. Plots of the datasets are given in Fig. \ref{fig:eg-1d} (left). 
Specifically,
\begin{itemize}
\item
Example 1.

Distribution of $p$ is
$ \frac{1}{5}( {\cal N}(-2, \sigma_p^2) + {\cal N}(-1, \sigma_p^2) + {\cal N}(0, \sigma_p^2) + {\cal N}(1, \sigma_p^2) + {\cal N}(2, \sigma_p^2) )$,
where $\sigma_p= 0.8$;
Distribution of $q$ is
$ (1-\delta) {\cal N}(0,1) + \frac{\delta}{2} {\cal N}(-3,\sigma_q^2) + \frac{\delta}{2} {\cal N}(4,\sigma_q^2)$,
where $\sigma_q$ = 0.5, $\delta= 0.1$.

\item
Example 2. The data vector $(x,y)$ is defined by
\[
x = \frac{t}{2}, \quad y = \text{Sigmoid}( 2t) + s,
\quad \text{Sigmoid}(z) = \frac{1}{1 + exp(-z)},
\]
where the random variables $t \sim p$ or $q$ as in Example 1, and $s \sim {\cal N}(0, \sigma_s^2)$ independently from $t$, $\sigma_s = 0.05$.
The finite $\sigma_s$ adds a small Gaussian noise to the $y$-coordinate such that the 2D data lie near to a 1D manifold, given by the curve when $\sigma_s = 0$.
\end{itemize}

For both examples, the network has one hidden layer of $H$ neurons, $H= 4, 8, \cdots, 512$,
and ReLU activation function. 
The training is by stochastic optimization and 40 replicas are conducted.
For a given witness function $f$, either analytic or parametrized by a trained neural network,
we compute the value of $L[f]$ by a numerical integration on the 1D or 2D domain.

{\bf Results.}
The values of $L[\hat{f}_{\text{tr}}]$, averaged over replicas,
are plotted against $H$ for increasing number of training samples $n=|X_{\text{tr}}|+|Y_{\text{tr}}| = 250, 500, \cdots, 4000$,
as shown in Fig. \ref{fig:eg-1d} (right).
The mean and standard deviation of $L[f_{\text{tr}}]$ over replicas are shown in Table \ref{tab:eg-1d-stdL}.
As shown in the plot and the table, the network training achieves $L[\hat{f}_{\text{tr}}]$ more stably as $n$ increases,
and the curves indicate reliable pattern except for small values of $n$ ($n=$250)
and the first few small values of $H$ ($H \le 32$).
The experimental results show that 
\begin{itemize}
\item
As the network complexity increases, 
the curve of $L[\hat{f}_{\text{tr}}]$ goes up and approach  the unconstrained maximum $L[f^*]$.
For larger $n$ the difference $L[f^*] - L[\hat{f}_{\text{tr}}]$ is smaller,
indicating that increasing training size $n$ reduces the influence of finite-sample.
Since 
$L[\hat{f}_{\text{gl}}] $ lies between $L[f^*]$ and $L[\hat{f}_{\text{tr}}]$,
this means that the stochastic optimization achieves a loss
which is $\Delta C$-close to that of $\hat{f}_{\text{gl}}$, supporting Assumption \ref{assump:training},
and one may furtherly expect $\Delta C$ to be small at least when $H$ is large.

\item
The trend of the increase of  $L[\hat{f}_{\text{tr}}]$ as the network complexity increases 
behaves similarly for Example 1 and Example 2,
which indicates that it is the the intrinsic geometry of the datasets 
that affects the efficiency of the network to produce a witness function with differential power. 
\end{itemize}

The second observation that for on or near-manifold datasets, 
only the intrinsic geometric complexity
influences the performance of neural network methods has been reported in experimental literature \cite{ansuini2019intrinsic,rifai2011manifold} and approximation literature \cite{shaham2018provable,chen2019nonparametric,schmidt2019deep}.
This motivates our theoretical work to reduce the needed network complexity 
to only scale with the intrinsic complexity of the densities $p$ and $q$.
Our result provides an explanation from the viewpoint of approximation and estimation error analysis.

\section{Neural Network Approximation of Near-manifold Integrals}
\label{sec:theory-approx}

This section establishes a result that,
for a near-manifold density $p$
(described by exponential decay of $p$ away from the manifold ${\cal M}$),
the uniform approximation of a $\R^D$-Lipschitz function in $L^\infty({\cal M})$
implies approximation in $L^1(\R^D, p)$ up to an error proportional to the scale of the exponential tail.
This serves to resolve Step 2 in Section \ref{subsec:roadmap},
where, 
since only on-manifold approximation suffices,
the network complexity scales with the
 intrinsic manifold dimensionality 
 and the target function restricted to ${\cal M}$.
The approximation in $L^1(\R^D, p)$ implied by that in $L^\infty({\cal M})$
is also used in the estimation error analysis in Steps 3-5 in Section \ref{sec:theory-consist}.
We give the near-manifold-$p$-integral approximation result 
in a slightly more general form than that of $L[f]$ as in \eqref{eq:training-loss-population},
which can be of independent interest. 
All proofs in Section \ref{sec:main-proofs}.

\subsection{The integral approximation result}

Let ${\cal M} \subset \R^D$ be a compact smooth manifold of dimension $d$.
We define ${\cal P}_\sigma$ to be the class of densities in $\R^D$ which decay exponentially fast away from the manifold ${\cal M}$,
that is, for some small $ 0 < \sigma < 1$, and absolute constant $c_1 > 0 $,
\begin{equation}\label{eq:density-mannifold-exp-decay}
{\cal P}_\sigma = \{ p \text{ density on $\R^D$ s.t. } \text{Pr}_{X \sim p}[ d(X, {\cal M}) > t] \le c_1 e^{- \frac{t}{\sigma}}\},
\end{equation}
where $d(x, {\cal M}) := \inf_{ y \in {\cal M}} \|x-y\|_2$ for any $x \in \R^D$.
We treat $c_1$ as fixed throughout the analysis and $\sigma$ as the parameter indicating the size of the exponential tail. 

The aim of the analysis is to approximately compute the integral 
\begin{equation}\label{eq:def-If}
I[f] : = \int_{\R^D} p(x) T(f(x)) dx,
\end{equation}
by replacing $f$ with a neural network function $f_\theta$,
where $T$ is a 1D Lipschitz function so as to make the result more general,
and $p \in {\cal P}_\sigma$ which concentrates near the manifold. 
Due to the exponential decay of the density $p$, we expect the integral $I[f]$ to be dominated by the contribution of the integral on the manifold. 
Indeed, restricting $f$ to be on the manifold allows us to approximate $f|_{\cal M}$ by a neural network whose model complexity  
as approximation error $\to 0$ 
only depends on $f|_{\cal M}$ and the intrinsic manifold geometry. 
The bound of $|I[f] - I[f_\theta]|$, which involves integral in the ambient space,
is then obtained by considering the corresponding on-manifold integrals of $f$ and $f_\theta$ respectively, which gives the following Theorem.

The on-manifold function approximation by neural network is built upon the result in \cite{shaham2018provable}, which starts from an atlas on ${\cal M}$,
and will be set up in Section \ref{subsec:setup-manifold}. 
The subscript of Lip denotes the domain on which the Lipschitz continuity is considered.

\begin{theorem}\label{thm:ambient-intergral-approx}
Suppose that $f: \R^D \to \R$ is Lipschitz on $\R^D$, $f|_{\cal M}$ is $C^2$ on the manifold,
$T: \R \to \R$ is Lipschitz-1,
and $p \in {\cal P}_\sigma$ with $\sigma < \frac{1}{2}$.
Consider $I[f]$ as in \eqref{eq:def-If},
then 
for any $  \varepsilon < 1$,
there is a neural network architecture $\Theta$ with $ O_{f ,{\cal M}}(\varepsilon^{-d/2}) + N_0$ many trainable parameters, 
and a function $f_{con} \in {\cal F}_\Theta$ such that 
\begin{equation}\label{eq:bound-thm-ambient-integral}
|I[f_\text{con}] - I[f]| \le (1+ 2 C_1({\cal M}) c_1 \sigma)  \varepsilon + C_3( f, {\cal M}) c_1 \sigma,
\end{equation}
where  
\begin{equation}\label{eq:C3-thm-ambient-integral}
C_3( f, {\cal M}) := 
 2 C_1({\cal M})    \| T\circ f \|_{L^\infty({\cal M})} 
 +
C_2({\cal M}) 
 ( L_{{\cal M}, f}   + \text{Lip}_{\R^D}( f )  ),
\end{equation}
$C_1({\cal M})$, $C_2({\cal M})$  are as in \eqref{eq:def-C1C2} and only depending on the manifold-atlas,
the meaning of $O_{f,{\cal M}}(\cdot )$, $N_0$ and $L_{{\cal M},  f}$ 
is the same as in Theorem \ref{thm:mainthm-shaham2018provable},
noting that $f$ in Theorem \ref{thm:mainthm-shaham2018provable} is $f|_{\cal M}$ here. 
In particular, none of the three changes with $\varepsilon$, and $N_0$ is also independent of $f$ but involves manifold-atlas and $D$. 
\end{theorem}

The proof combines the integral comparison for Lipschitz functions on $\R^D$ in Proposition \ref{prop:integral-swap},
and the on-manifold function approximation in Theorem \ref{thm:mainthm-shaham2018provable}, 
where it is proved that $\text{Lip}_{\R^D}(f_{con}) \le L_{{\cal M},  f}$ 
using the wavelet construction of $f_{con}$. In Section \ref{sec:theory-consist}, 
we will introduce a Lipschitz regularization of the network family (Assumption \ref{assump:netLip}), 
and further bound the $\infty$-norm of $T\circ f$ on ${\cal M}$ in the two-sample testing context.

\begin{figure}[t]
\centering{
~~~~~~~~~~~~~~~~~
\includegraphics[trim={40pt 230pt 0pt 50pt},clip,height=.38\linewidth]{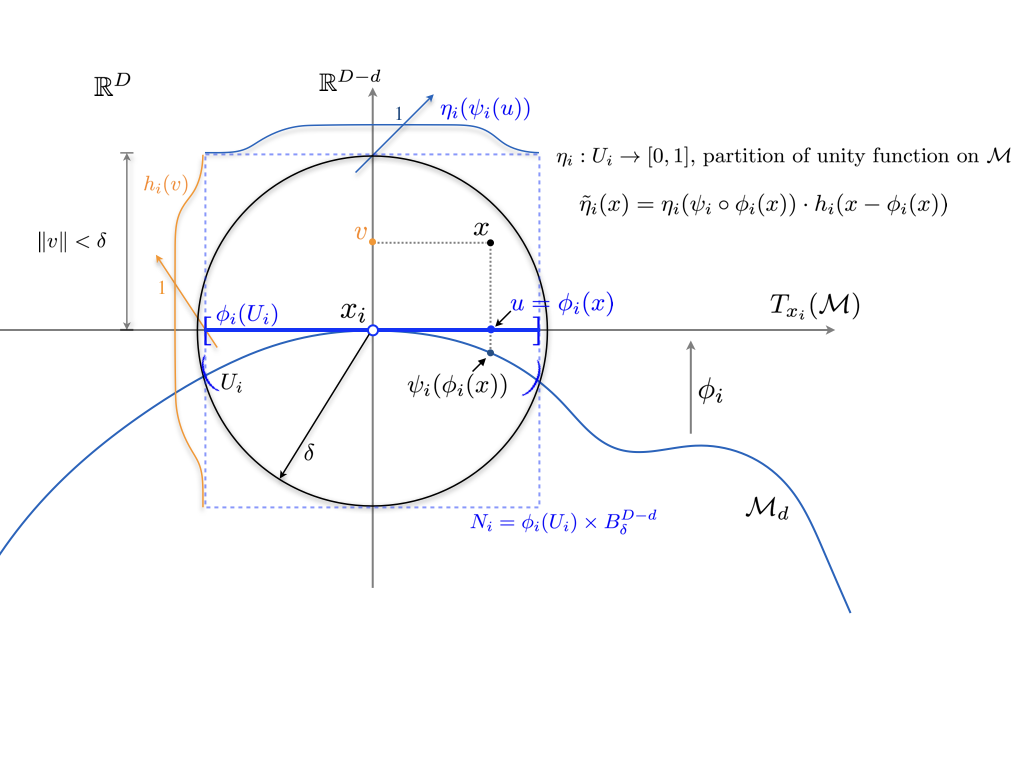}
}
\caption{
A $d$-dimensional manifold ${\cal M}$ embedded in the ambient space $\R^D$ (only part of ${\cal M}$ is shown), 
an atlas $\{(U_i, \phi_i)\}_{i=1}^K$ which gives a $\delta$-cover of ${\cal M}$,
and the extended partition of unity function $\tilde{\eta}_i$ defined on the neighborhood $N_i$ of $U_i$ in $\R^D$.
}
\label{fig:diag2}
\vspace{-10pt}
\end{figure}

\subsection{Manifold atlas and on-manifold function approximation}\label{subsec:setup-manifold}

We first establish some notations for the manifold and atlas cover.   
The manifold $\mathcal{M}$ can be covered with an atlas $\{(U_i, \phi_i)\}_{i=1}^{ K}$,
where $U_i = B(x_i,\delta)\cap \mathcal{M}$ is an open set on $\mathcal{M}$, 
$ 0 < \delta < 1$, the choice of which to be specified below.
The orthogonal projection
 $\phi_i:U_i\rightarrow \mathbb{R}^d$ is the map that takes $U_i$ to the tangent plane $T_{x_i}({\cal M})$.
We also define the map $\psi_i: \phi_i(U_i)\rightarrow U_i$, 
which is the inverse of $\phi_i$ due to the one-to-one correspondence between $U_i$ and $\phi_i(U_i)$.
Let $d_{\cal M}$ denote the manifold geodesic distance.
In the construction, the Euclidean ball radius $\delta$ is chosen to be small enough such that $B(x, 2\delta) \cap {\cal M}$ is isomorphic to a ball in $\R^d$ and 
there exist positive $\alpha_{i}$ and $\beta_{i}$ s.t.
\begin{equation}\label{eq:equiv-norm-Ui}
\alpha_{i} \| \phi_i(x) - \phi_i(x') \|_2 
\le 
d_{\cal M}(x, x')
\le
\beta_{i} \| \phi_i(x) - \phi_i(x') \|_2,
\quad 
\forall x, x' \in U_i,
\end{equation}
and for all $i$, 
\begin{equation}\label{eq:global-alpha-beta}
\alpha_{i} \ge \alpha_{\cal M} \ge \frac{1}{2}, 
\quad  1\le \beta_{i} \le \beta_{\cal M} \le 2.
\end{equation}
To see why this is possible:
for any point $x \in {\cal M}$, 
there is a $\delta_x >0$ and constants  $\frac{1}{2} < \alpha_x \le 1 \le  \beta_x < 2$
such that whenever $\delta' < \delta_x$,
the relation \eqref{eq:equiv-norm-Ui} with constants $\beta_x, \alpha_x$ 
holds on the neighborhood $U_x := B(x,\delta') \cap {\cal M}$,
and at the same time $B(x, 2\delta') \cap {\cal M}$ is isomorphic to a ball in $\R^d$.
This is because that the manifold is locally near Euclidean,
which means that $\beta_x, \alpha_x$ can be made close to 1 if $\delta' \to 0+$.
The existence of $\delta_x$ and $\beta_x, \alpha_x$ 
is due to manifold smoothness and the finiteness of the local curvature near $x$. 
The $\inf_{x \in {\cal M}} \delta_x$ exists and is positive due to compactness and smoothness of ${\cal M}$.
Setting that (and the minimum with 1)  as $\delta$ for all $x$ 
 leads to a finite cover of ${\cal M}$ which is $\{U_i\}_{i=1}^K$
with constants $\alpha_i, \beta_i$ on each $U_i$, and then the global $\alpha_{\cal M}, \beta_{\cal M}$ exist.
Note that while the atlas and particularly the radius $\delta$ depend on the embedding of ${\cal M}$ in the ambient space $\R^D$, the atlas remains valid if $D$ increases to $D'$, ${\cal M}$ is isometrically embedded into $\R^{D'}$, and the reach of the manifold is maintained \cite{aamari2019estimating},
e.g., 
when $\R^D$ is isometrically embedded in $\R^{D'}$. 
Thus we view any quantity which only depends on the $\delta$-atlas as intrinsic to the manifold geometry.

Given the covering atlas, there exists a partition of unity $\{\eta_i\}_{i=1}^{K}$  on ${\cal M}$ such that supp$(\eta_i) \subset U_i$, 
$\eta_i \in C^\infty({\cal M})$, and $\sum_{i=1}^{ K} \eta_i(x) = 1$ for all $x\in \mathcal{M}$. 
Under this setting, we have the following:

\begin{theorem}
\label{thm:mainthm-shaham2018provable}
Notations as above, 
suppose $f \in C^2(\mathcal{M})$, 
then for any $\varepsilon < 1$, 
there exists a four layer feed-forward network with rectified linear unit activations and 
$N = O_{f,\mathcal{M}}(\varepsilon^{-d/2}) + N_0$ 
parameters,
such that  the network function  $f_N: \R^D \to \R$ satisfies that
\begin{align*}
\|f - f_N\|_{L^\infty(\mathcal{M})} \le \varepsilon,
\end{align*}
where the constant in $O_{f,\mathcal{M}}(\cdot)$ only depends on $f$ and manifold-atlas,
specifically, it is $(d+K)\delta^d (K C_{f,\eta})^{d/2}$,
$K$ being the number of coverings in the $\delta$-atlas,
the constant $C_{f,\eta}$ depending on $d$ and up to 2nd manifold-derivatives of $f$ and the partition of unity functions of the $\delta$-atlas.
$N_0 = C( KdD + \frac{D}{\delta} \log\frac{D}{\delta}  )$,
 $C$ being an absolute constant, 
 and $N_0$ does not depend on $f$ nor change with $\varepsilon$.  

Furthermore, the constructed network function $f_N$ is globally Lipschitz and $Lip_{\mathbb{R}^D}(f_N) \le L_{\mathcal{M},f}$, which is a constant depending on $f$ and the manifold-atlas, but independent of $\varepsilon$ and $N$.
\end{theorem}

The proof is by constructing a sub-network function $\bar{f}_{N,i}$
on a neighborhood $N_i \subset \R^D$ around each $U_i \subset {\cal M}$ to approximate $f \eta_i$ respectively, and then taking a sum over $i$ using the partition of unity property. 
Specifically, for each $U_i$, the neighborhood $N_i$ is defined as
\begin{equation}\label{eq:def-Ni}
N_i : = \phi_i (U_i) \times B_\delta^{D-d},
\quad B_\delta^{D-d}:= \{ x \in \R^{D-d}, \, \|x\|_2 < \delta \},
\end{equation}
thus $\phi_i(N_i) = \phi_i(U_i)$,
where we also denote $\phi_i$ as the orthogonal projection from $\R^D$ to the tangent space $T_{x_i}({\cal M})$. 
The relation \eqref{eq:equiv-norm-Ui}\eqref{eq:global-alpha-beta} gives the following lemma,
\begin{lemma}\label{lemma:manifold-in-tube}
For any $x \in U_i$,
$\|x - \phi_i(x)\| < \delta \sqrt{1-\frac{1}{\beta_i^2}} \le \frac{\sqrt{3}}{2} \delta$.
\end{lemma}

The sub-network function $\bar{f}_{N,i}(x)$, $x=(u,v)$ in local coordinates, is constructed by first approximating $f \eta_i(\psi_i(u))$, viewed as a function on $\R^d$, by some $\hat{f}_{N,i}(u)$, $u \in \R^d$, and then extending constantly into $N_i$ for a distance  $< \delta$ by a tent-like function $g_i(v)$ on $\R^{D-d}$ which is 1 when $\| v \| < \frac{\sqrt{3}}{2}\delta$ and 0 when $\|v \| \ge 1$. Lemma \ref{lemma:manifold-in-tube} guarantees that $\bar{f}_{N,i}$ restricted to $x \in U_i$ equals $\hat{f}_{N,i}(u)$ because $g_i(v)=1$ on $U_i$. 
Thus $f_N$ by taking a sum $\sum_{i=1}^K \bar{f}_{N,i}$ is uniformly approximating $f$ on the manifold. 
The Lipschitz constant of $f_N$ is proved by considering that of $\hat{f}_{N,i}$ and $g_i$ respectively which bounds each $\text{Lip}_{\R^D}(\bar{f}_{N,i})$.
The statement regarding the global Lipschitz continuity of $f_N$ is a byproduct of the construction in \cite{shaham2018provable}, but was not explicitly stated therein. 
For completeness, we provide a proof of the Theorem in Section \ref{sec:main-proofs}.

\subsection{Comparison of on and near-manifold integrals}

As we would like to analyze the integral of near manifold density in the ambient space, we extend the partition of unity function $\eta_i$ to the neighborhood $N_i$ defined as in \eqref{eq:def-Ni} as
\begin{equation}\label{eq:def-tileta}
\tilde{\eta}_i (x) = \eta_i( \psi_i\circ \phi_i(x) ) \cdot h_i( x - \phi_i(x) ), 
\end{equation}
where $h_i( x)$ is continuous on $ \R^{D-d}$, vanishes outside $B_\delta^{D-d}$, $0 \le h_i \le 1$ and 
\begin{equation}\label{eq:hx=1}
h_i( x ) =1  \text{ if } \|x\| \le  \delta \sqrt{1-\frac{1}{\beta_i^2} },
\quad
\text{Lip}( h_i ) < \frac{2\beta_i^2}{\delta}.
\end{equation}
This can be constructed, e.g., by $h_i(x) = g(\|x\| / \delta)$  where $g(r)=1$ on $[0, \sqrt{1-\frac{1}{\beta_i^2} }]$, 0 outside $r > 1$,
and linearly interpolating in between. 
This construction guarantees the following properties of the extended function $\tilde{\eta}_i$'s:

\begin{lemma}\label{lemma:lip-etai}
For $i=1,\cdots,K$, 
$\tilde{\eta}_i$ as a function in $\R^D$
vanishes outside $N_i$,
equals $\eta_i$ on $U_i$,
is Lipschitz continuous on $\R^D$, 
and, for all $i$,
\[
\text{Lip}_{\R^D} (\tilde{\eta}_i) =  \text{Lip}_{N_i}(\tilde{\eta}_i) \le L_{\cal M},
\]
where $L_{\cal M}$ is an absolute constant depending on the manifold-atlas.
\end{lemma}

We then establish a key result used in the proof of Theorem \ref{thm:ambient-intergral-approx},
which compares the ambient space integral to a ``projected'' integral on the manifold.

\begin{proposition}[Integral comparison]
\label{prop:integral-swap}
Notations of manifold-atlas and partition of unity functions as above.
Suppose $g: \R^D \to \R$ is Lipschitz continuous on $\R^D$,
$p \in {\cal P}_\sigma$ with $\sigma < \frac{1}{2}$, 
and define \[
\tilde{p}(x) = \sum_{i=1}^K \eta_i(x) \tilde{p}_i(x),
\]
where $\tilde{p}_i$ is an atlas dependent ``projection'' of the density $p$ to $U_i$, 
the explicit formula to be given below,
then
\begin{equation}\label{eq:claim-integral-swap}
\left| \int_{\R^D} g(x)p(x) dx - \int_{\cal M} g(x) \tilde{p}(x) d_{\cal M}(x) \right|
\le
\left( \| g\|_{L^\infty({\cal M})} C_1({\cal M}) + \text{Lip}_{\R^D}(g) C_2({\cal M}) \right) c_1 \sigma,
\end{equation}
where $c_1$ as in the definition of ${\cal P}_\sigma$ \eqref{eq:density-mannifold-exp-decay},
\begin{equation}\label{eq:def-C1C2}
C_1( {\cal M}) = 3 K L_{{\cal M}},
\quad
C_2( {\cal M}) = K (2L_{{\cal M}} + 1+\beta_{\cal M}),
\end{equation}
$ L_{{\cal M}}$ as in Lemma \ref{lemma:lip-etai}, and $\beta_{\cal M}$ as in \eqref{eq:global-alpha-beta}.
$C_1({\cal M})$ and $C_2({\cal M})$ are absolute constants determined by the manifold and atlas.
\end{proposition}
In particular, taking $g = 1$, the proposition gives that
\begin{equation}\label{eq:tildep-almost-density}
\left| \int_{\cal M} \tilde{p}(x) d_{\cal M}(x)  - 1 \right|
\le 
 C_1({\cal M})  c_1 \sigma,
\end{equation}
which means that the constructed $\tilde{p}$ is close to being integral 1 on ${\cal M}$ up to an error proportional to $\sigma$.

The proof of the main Theorem \ref{thm:ambient-intergral-approx}
then follows by combining Proposition \ref{prop:integral-swap}
and  Theorem \ref{thm:mainthm-shaham2018provable}.

\section{Consistency of Network Logit Test}\label{sec:theory-consist}

In this section, we fulfill the steps listed in Section \ref{subsec:roadmap} to prove the theoretical guarantee of the network logit two-sample testing.
The manifold setting of densities reduces the dimensionality from $D$ to $d$ in both the approximation error (Step 2)
and the estimation error analysis (Steps 3-5). 
All proofs in Section \ref{sec:main-proofs}.

\subsection{Settings of data densities and Step 1. $f_{\text{tar}}$}

We consider the following three settings of densities $p$ and $q$,
\begin{itemize}
\item
(General setting)
Densities in $\R^D$ with sub-exponential tail.
This includes compactly supported densities,
which are mostly encountered in practice.

\item
(On-manifold setting)
Densities constrained on a smooth compact manifold ${\cal M} \subset \R^D$.

\item
(Near-manifold setting)
Densities which are exponentially decay away from the manifold ${\cal M}$, as defined in \eqref{eq:density-mannifold-exp-decay}
with some positive $c_1$ and small $\sigma$.
\end{itemize}

The analysis is under the same framework,
where bifurcations take place in 
the integral approximation
and bounding of the estimation error.
We consider the class of sub-exponential densities in $\R^D$ as
\begin{equation}\label{eq:def-calP-subexp-RD}
{\cal P}_{\text{exp}} = \{
p \text{ density on $\R^D$ s.t. } \text{Pr}_{X \sim p}[ \|X \| > t] \le C e^{-t/c}
\}
\end{equation}
for some $C , c >0$.
One can always rescale the space to make $p$ and $q$ supported on a diameter-$O(1)$ domain, when compactly supported, or $c=1$, when exponentially decay,
even when the dimension $D$ is large.  
This corresponds to normalizing the data vectors to be of $O(1)$ norm in practice.
By this normalizing argument, we assume supp$(p+q)$ has diameter-$O(1)$ when compactly supported
and generally exponentially decay with $c=1$ in below.

We start from Step 1. in  Section \ref{subsec:roadmap}. A sub-exponential density $p$ may vanish at certain points in $\R^D$ or have discontinuity,
e.g., when supp($p$) is compact.
This makes $f^* = \log \frac{p}{q}$ possibly diverge at a point,
e.g., when supp($p$) and supp($q$) partially do not overlap. 
Observe that unless $p = q$,
one can always restrict to the  interior of  supp($\frac{p+q}{2}$)
and consider a bounded version of $f^*$, e.g., $f=\min\{ \max\{ f^*, M\}, -M\}$ for $M > 0$,
such that $L[f]$ as in  \eqref{eq:training-loss-population} is close to JSD($p, q$) and $>0$. 
Furthermore, $L[f]$ can be written as
\begin{equation}\label{eq:def-Lf-2}
L[f] = \frac{1}{2} \left( 
\int_{\R^D} p T_p \circ f 
+
\int_{\R^D} q T_q \circ f 
\right),
\quad
T_p( \xi ) = \log \frac{2 e^\xi}{1 + e^{\xi}},
\quad
T_q( \xi ) = \log \frac{2}{1 + e^{\xi}}, 
\end{equation}
where  $T=T_p$ and $T_q$ all satisfy that
\begin{equation}\label{eq:LipT}
T: \R \to \R \text{ is Lipschitz and Lip$(T) < $1,  and $T(0)=0$}.
\end{equation}
As a result, approximating the bounded function $f$ under $L^1(p)$ and $L^1(q)$
will approximate the integral $L[f]$.
Thus we can choose the approximator of $f$ to be smooth,
and, by the sub-exponential decay of $p$ and $q$, 
to be compactly supported.
According to these procedures,
we take the following assumption which is mild under our setting.

\begin{assump}\label{assump:ftar}
Suppose $p \neq q$ in ${\cal P}_{\text{exp}}$ are given,
then there exists a compactly-supported smooth function $f_{\text{tar}}$ on $\R^D$ such that 
$L[f_{\text{tar}}] := C > 0$.
\end{assump}
In particular, if $(p+q)$ are compactly supported, then supp$(f_{\text{tar}})$ can be made within that domain. 
For general sub-exponential densities, the diameter of supp$(f_{\text{tar}})$ can be made proportional to the scale parameter $c$ in ${\cal P}_{\text{exp}}$. 
By the normalizing argument above, the diameter of  supp$(f_{\text{tar}})$ is $O(1)$. 
Note that there are more than one choice of $f_{\text{tar}}$ to fulfill Assumption \ref{assump:ftar},
and one may sacrifice the regularity of $f_{\text{tar}}$ to make $C$ arbitrarily close to $L[f^*] $ = JSD($p,q$)$> 0$.
We further discuss this after Proposition \ref{prop:step2-general-pq}.

\subsection{Step 2. neural network approximation}

The neural network approximation theory provides  $f_{\text{con}}$
which uniformly approximates $f_{\text{tar}}$ on a compact domain $\Omega$.
We consider the three settings of densities respectively.

\subsubsection{General setting} 
When $p$ and $q$ are general sub-exponential densities on $\R^D$, we make use of the compact-supportedness of $f_{\text{tar}}$ as in Assumption \ref{assump:ftar}, and set it as the target function to approximate.

\begin{proposition}\label{prop:step2-general-pq}
Suppose that $p \neq q$, JSD$(p,q) > 0$, 
and $p, q$ are in ${\cal P}_{\text{exp}}$  with $c=1$. 
Under Assumption \ref{assump:ftar}, 
there exists a smooth and compactly supported $f_{\text{tar}} $ with $L[f_{\text{tar}}] = C  > 0$,
and diameter of  supp($f_{\text{tar}}$)  $< C'$.
Then for any $ \varepsilon < 1 $,
there is a neural network architecture $\Theta$
with $ O_{f_{tar}, C'}(\varepsilon^{-D/r} \log \frac{1}{\varepsilon}) $ 
many trainable parameters, 
and $f_{con} \in {\cal F}_\Theta$ such that
\[
L[f_\text{con}] \ge C - \varepsilon > 0,
\]
where the constant in $O_{f_{tar}, C'}(\cdot)$ depends on up to the $r$-th derivative of $f_{\text{tar}}$, the diameter $C'$, $r$ and $D$, $r \ge 1$.
When $p$ and $q$ are compactly supported, 
 $C'$ is the diameter of supp($p+q$).
\end{proposition}

As previously commented beneath Assumption \ref{assump:ftar}, 
$C'$ is an $O(1)$ constant even when dimensionality $D$ can be large,
thus the network complexity remains $ O(\varepsilon^{-D/r})$.
The proof of Proposition \ref{prop:step2-general-pq} is by a direct application of the standard network approximation theory, e.g. that in \cite{yarotsky2017error},
where the sub-exponential tail of the distribution does not affect because the constructed network function can be made supported on the same domain of supp($f_{\text{tar}}$). More recent approximation result which improves the approximation rate, such as \cite{yarotsky2018optimal},
will improve the complexity needed accordingly. 
We remark in Appendix \ref{app:remark} about the choice of $f_{\text{tar}}$ in Assumption \ref{assump:ftar}
and the $r$ used in Proposition \ref{prop:step2-general-pq},
including the case when $p$ and $q$ nearly do not overlap thus the density ratio is singular. 
There is a trade-off between how close is $(C-\varepsilon)$ to $L[f^*]$
and how regular $f_\text{tar}$ is and then how large the network complexity is needed.
The trade-offs can be made precise in specific settings of the densities, and we leave the choices abstract here.

\subsubsection{ On-manifold setting}
When $p$ and $q$ are degenerate and constrained to a compact smooth manifold ${\cal M}$ in $\R^D$,
the integral  $L[f]$ is carried out on the manifold only. 
Replacing the Euclidean metric with the Riemannian geometry on ${\cal M}$,
and the integral in $\R^D$ with integration on ${\cal M}$ with Riemannian volume element,
the choice of a smooth $f_\text{tar}$ with $L[f_\text{tar}] = C > 0$ as in Assumption \ref{assump:ftar} extends.
This gives the on-manifold-version of Proposition \ref{prop:step2-general-pq},
where the classical network approximation is replaced with Theorem \ref{thm:mainthm-shaham2018provable},
which guarantees the existence of $f_\text{con} $ such that 
\[
\| f_\text{con} - f^* \|_{L^\infty(\cal M)} \le \varepsilon,
\]
where the needed network complexity of $O(\varepsilon^{-d/2})$,
namely reducing the exponent factor from $D$ to the intrinsic dimensionality $d$. 
The proof of Proposition \ref{prop:step2-general-pq} directly extends by replacing integration on $\Omega$ with that on ${\cal M}$
and the rest is the same.

\subsubsection{ Near-manifold setting}

When the densities decays sub-exponentially away from the manifold, 
since ${\cal M}$ is compact, the densities belong to be sub-exponential family as in the general setting. 
While Proposition \ref{prop:step2-general-pq} still applies,
the needed network complexity is not intrinsic.
We apply the analysis in Section \ref{sec:theory-approx} to improve this.

\begin{proposition}\label{prop:step2-nearmanifold-pq}
Suppose that $p \neq q$, JSD$(p,q) > 0$, 
and $p, q \in {\cal P}_\sigma$ as defined in \eqref{eq:density-mannifold-exp-decay} with  $\sigma < \frac{1}{2}$.
Under Assumption \ref{assump:ftar}, 
there exists a smooth and compactly supported $f_{\text{tar}} $ with $L[f_{\text{tar}}] = C  > 0$,
and there is one point $x_0 \in {\cal M}$ such that $f_{\text{tar}}(x_0) = 0$.
Then for any $ \varepsilon <1 $,
there is a neural network architecture $\Theta$ with $ O_{f_{tar},{\cal M}}(\varepsilon^{-d/2}) + N_0$ many trainable parameters, 
and $f_{con} \in {\cal F}_\Theta$ such that
\begin{equation}\label{eq:bound-prop-step2-nearmanifold}
 L[f_\text{con}] \ge  C - \Bigg( \varepsilon \cdot (1+ 2 c_1 \sigma C_1({\cal M}) )  
	+ \sigma \cdot c_1 \left( 
		C_4({\cal M}) \text{Lip}_{\R^D}(f_{\text{tar}}) + C_2({\cal M}) L_{{\cal M}, f_{\text{tar}}}   \right)  \Bigg), 
\end{equation}
where 
\[
 C_4({\cal M}):=2 C_1 ({\cal M})\text{diam}({\cal M}) + C_2({\cal M}),
 \]
and $C_1$, $C_2$ as in \eqref{eq:def-C1C2},
the constants $C_1$, $C_2$, $C_4$ only depend on the manifold-atlas.
The meaning of $O_{f_{tar},{\cal M}}(\cdot)$, $N_0$, $L_{{\cal M}, f_{\text{tar}}} $
is the same as in Theorem \ref{thm:ambient-intergral-approx} taking $f=f_{\text{tar}}$.
The r.h.s. of \eqref{eq:bound-prop-step2-nearmanifold}
$> 0 $ when $\varepsilon$ and $\sigma$ are sufficiently small.
\end{proposition}
The condition that $f_{\text{tar}}$ vanishes at one point on ${\cal M}$:
Since $p \neq q$, and by construction $f_{\text{tar}}$ is the smooth surrogate of $f^*=\log \frac{p}{q}$,
then $f_{\text{tar}}$ vanishes at least at one point in ${\R^D}$.
When both $p$ and $q$ are concentrating near the manifold ${\cal M}$, it generically holds that $f_{\text{tar}}$ vanishes at least at one point on the manifold.
Thus the condition is mild and does not pose a constraint. 
The bound in \eqref{eq:bound-prop-step2-nearmanifold} shows that 
when $\sigma$ is sufficiently small,
then the integral of $L[f_{\text{tar}}]$ can be approximated by constructing on-manifold approximation of the target function only,
which reduces the network complexity to be intrinsic. 
Particularly, if $c_1 \sigma  < \alpha \varepsilon$ for some positive $\alpha$,
then the difference from $C$ in the r.h.s. is less than
\[
\varepsilon ( 1 + 2 C_1({\cal M})  \alpha  \varepsilon ) 
+  \alpha \varepsilon 
\{ C_4({\cal M}) \text{Lip}_{\R^D}(f_{\text{tar}}) + C_2( {\cal M}) L_{{\cal M}, f_{\text{tar}}}  \}
\sim  \varepsilon ( 1 + \alpha \{\text{(constant inside)}\} )
\]
when $\varepsilon$ is small, which is an $O(\varepsilon)$ bound as long as $\sigma$ is at the order of $O(\varepsilon)$.

\subsection{Concentration of $L_n$ and Steps 3-5}

Steps 3 and 5 are based on the concentration of $L_n[ f_\theta]$ close to $L[f_\theta]$ when $f_\theta = f_{\text{con}}$ or 
trained on the training set. 
We omit subscript ``tr'' in $L_n$, which emphasizes that $L_n$ is the empirical loss on the training samples. 
We will upper bound, under proper Lipschitz regularization of network function class ${\cal F}_\Theta$, that
\begin{equation}\label{eq:what-bound-to-prove}
 \sup_{f \in {\cal F}_\Theta} | L_n[f] - L[f] | \le \, ?
 \quad \text{w.h.p. for sufficiently large $n$,}
\end{equation}
where w.h.p. stands for ``with high probability'' and will be made precise.
For the three settings of densities,
we prove that the bound in \eqref{eq:what-bound-to-prove} is 
$\tilde{O}(n^{- \frac{1}{2+D}}) $ for ${\cal P}_{\text{exp}}$ densities in $\R^D$,
$\tilde{O}(n^{- \frac{1}{2+d}}) $ for on-manifold densities,
and $\tilde{O}(n^{- \frac{1}{2+d}}) + O(\sigma)$ for near-manifold densities in ${\cal P}_\sigma$,
$\tilde{O}$ meaning that the constant may involve $\log n$ (Proposition \ref{prop:conc-supF-Ln}).
We first introduce the needed Lipschitz regularization of network functions,
which leads to a bound of the covering number of ${\cal F}_\Theta$
that is used in the concentration analysis.

\subsubsection{Lipschitz regularization of network functions}

When neural network is using ReLU or other continuous nonlinear activations, 
the network function is typically differentiable or piece-wise differentiable on $\R^D$,
and is globally Lipschitz because all the weights are finite.
However, $\text{Lip}_{\R^D}(f_\theta)$ for a member $f_\theta \in {\cal F}_{\Theta}$ may potentially be large. 
In the network approximation analysis,
Theorem \ref{thm:mainthm-shaham2018provable}  shows that as the approximation error $\varepsilon $ gets small
the constructed network function $f_{\text{con}}$ to approximate $f$ is globally Lipschitz with $\text{Lip}_{\R^D}( f_{\text{con}} ) \le L_{{\cal M}, f}$,
a constant only depending on manifold-atlas and $f|_{\cal M}$.
This means that when the target function is smooth,
constraining on bounded Lipschitz constant of $f_\theta$
does not prevent the network approximation of $f$ to high accuracy.
This approximation result however does not mean that the trained network function $\hat{f}_{n,tr}$ has certain bounded Lipschitz constant automatically.
$\hat{f}_{n,tr}$ with poor Lipschitz regularity may incure larger variance of the statistic with finite samples.
Applying regularization to the network function balances between bias and variance, and is commonly  used in practice.
We thus consider network function families with Lipschitz regularization,  that is,
\begin{assump}\label{assump:netLip}
To achieve decreasing approximation threshold $\varepsilon$, 
the network function family $\Theta = \Theta (\varepsilon)$ being used has increasing complexity,
and regularization of ${\cal F}_{\Theta}$ can be applied such that for all $\varepsilon$,
\[
\sup_{f \in {\cal F}_\Theta} \text{Lip}_{\R^D} (f )  \le L_\Theta.
\]
The universal Lipschitz constant bound $L_\Theta$ can be viewed as a network hyperparameter,
and we call the restricted family ${\cal F}_{\Theta, \text{reg}}$.
\end{assump}

\subsubsection{Covering number of regularized ${\cal F}_\Theta$}

The Lipschitz regularization 
enables bounding of the covering number of function space by the covering number of the domain when compact. For a compact set $K \subset \R^D$, define its covering number as
\[
{\cal N}(K, r):= \inf \{ \text{Card}(S), \,  K \subset \underset{ s \in S}{\cup} \bar{B}_r( x)  \},
\quad 
r > 0.
\]
where ``Card" stands for cardinal number, $\bar{B}_r(x)$ is the Euclidean closed ball $\{y, \|y-x\| \le r \}$,
and the requirement of $S$ is equivalent to that $S$ is an $r$-net of $K$.
The general covering number is denoted as ${\cal N}( X, r, \| \cdot \| )$, where $X$ is the set to be covered, and $r$ is the radius of the closed $\| \cdot \|$-ball. 

\begin{lemma}\label{lemma:cover}
Let $K$ be a compact set in $\R^D$ with covering number ${\cal N}(K, r)$, $r >0$.
Consider the function class
\[
{\cal F}: = \{ f:\R^D \to \R, \, \text{Lip}_{\R^D}(f) \le L, \, \|f\|_{L^\infty(K)} \le B \},
\]
where $L > 0$, $B > 0$.
Then for any subset ${\cal F}'$ of ${\cal F}$ (including ${\cal F}' = {\cal F}$),
and any $0 < r < \frac{B}{L}$, 
there is a finite set $F \subset {\cal F}'$ which forms an ($4Lr$)-net of ${\cal F}'$, i.e., 
${\cal F}' \subset \underset{f \in F}{\cup} \bar{B}(f, 4Lr, \|f\|_{L^\infty(K)} )$
and 
\begin{equation}\label{eq:boundNF}
\text{Card}(F)
\le \left( \frac{2B}{L r} \right)^{{\cal N}(K, r)}.
\end{equation}
\end{lemma}

The lemma proves an upper bound for the covering number of the whole class ${\cal F}$,
which will contain regularized network function class as a subset. 
When applying to analyzing the general and on-or-near manifold densities,
the $K$ will be a ball in $\R^D$ and the compact manifold respectively,
the covering number ${\cal N}(K)$ then
differs in a factor of $r^{-D}$ v.s. that of $r^{-d}$.

\subsubsection{Concentration of $L_n$ sup over network function family}

The concentration of $L_n[f]$ for a single Lipschitz $f$ is a direct result of Bernstein's inequality for sub-exponential random variables.
Specifically, the following lemma, proved in Appendix \ref{app:proofs}, 
verifies that the random variables $T \circ f(x_i)$ are sub-exponential due to that the density of $x_i$ is in ${\cal P}_\text{exp}$.
\begin{lemma}\label{lemma:sub-exp-xi}
Suppose that $T: \R \to \R$, Lip($T$) $ \le1$, 
$f: \R^D \to \R$, $\text{Lip}_{\R^D}(f) \le L$, $x_i \sim p$, $i=1,\cdots,n$, i.i.d.,  and $p$ is sub-exponential on $\R^D$, i.e.  $p \in {\cal P}_{\text{exp}}$  with $c=1$,
then $\xi_i : = T(f(x_i))$ are i.i.d 1D sub-exponential random variables, and specifically,
\begin{equation}\label{eq:sub-exp-xi}
\Pr [ | \xi_i - \E \xi_i  | > t]
\le C' e^{-c' \frac{t}{L}},
\quad \forall t > 0,
\end{equation}
where $C'$ and $c'$ are absolute positive constants.
\end{lemma}
Bernstein's inequality for sub-exponential random variables (Corollary 2.8.3 in \cite{vershynin2018high}) then gives that 
\begin{equation}\label{eq:Bern-1}
\Pr \left[
 \left| \frac{1}{n }  \sum_{i=1}^n T(f(x_i)) - \E_{x \sim p} T(f(x)) \right| \ge t
 \right]
 \le 2 \exp \left \{
 - c_0' n \frac{t^2}{L^2} \right\},
 \quad \forall  0< t  < c_1' L,
\end{equation}
where $c_0'$, $c_1'$ are absolute positive constant.
We will control the sup over network function family by a covering argument. 
Define the regularized neural network function class for architecture $\Theta$ as
\begin{equation}\label{eq:def-Freg}
{\cal F}_{\Theta, \text{reg}}(\Omega) = \{
f \in {\cal F}_\Theta,
\, \text{Lip}_{\R^D}( f ) \le L,
\, \exists x_0 \in \Omega, f(x_0) = 0
\},
\end{equation}
where the dependence on $\Omega$ is only via the assumption on the location of a vanishing point of $f$.
Under Assumption \ref{assump:netLip}, $L$ equals the universal constant $L_\Theta$,
the subscript $\Theta$ is omitted here for simplicity. 

\begin{proposition}\label{prop:conc-supF-Ln}
Let $B_R$ denote the ball in $\R^D$ centering at origin, $R \ge 1$.
For the three density settings,  suppose that
\begin{itemize}
\item[(1)] General. $p, q \in {\cal P}_{\text{exp}}$ with $c=1$. 
When $p$ and $q$ are compactly supported, supp($p+q$) $\subset B_{R}$.

\item[(2)] On-manifold.
Case (1) plus that supp($p+q$) $\subset {\cal M} \subset B_R$ in $\R^D$.

\item[(3)] Near-manifold. ${\cal M}$  as in (2),
case (1) plus that
$p, q \in {\cal P}_\sigma$ as defined in \eqref{eq:density-mannifold-exp-decay} with  $\sigma < \frac{1}{2}$.
\end{itemize}
Suppose that the network function family, 
denoted as ${\cal F}_{\Theta, \text{reg}}$ for all cases,
 is ${\cal F}_{\Theta, \text{reg}}( B_{R}) $ for case (1),
and ${\cal F}_{\Theta, \text{reg}}( {\cal M} ) $ for cases (2) (3).
Then, 
when $n$ is sufficiently large, with probability $\to 1$ as $n \to \infty$,
\[
\underset{ f \in {\cal F}_{\Theta, \text{reg}}}{\sup} | L_n[f] - L[f]|
\le
\begin{cases}
\tilde{C} \cdot L \log n (\log n /n)^{1/(2+D)},  & \text{case (1)}\\
\tilde{C}({\cal M}) \cdot L (\log n /n)^{1/(2+d)}, &   \text{case (2)}\\
\tilde{C}({\cal M}) \cdot L (\sigma +  (\log n /n)^{1/(2+d)}), &   \text{case (3)}
\end{cases}
\] 
where $\tilde{C}( \cdot)$ refers to a positive constant that may depend on $\cdot$,
without $(\cdot)$ an absolute constant, 
and the notation stands for different constants in different cases.
In case (1), if $p$, $q$ are compactly supported, 
the bound can improve to $\tilde{C}(R) \cdot L (\log n /n)^{1/(2+D)}$ removing a $\log n$ factor. 
The $\to 1$ probability is exponentially high except for the non-compactly supported case in (1).
\end{proposition}

In case (1), though the densities may have sub-exponential tail,
 it is generic to assume that exists $x_0 \in B_{R}$ such that $p(x_0)=q(x_0)$. 
Since $\log \frac{p}{q}$ vanishes at least at one point inside $B_{R}$, we consider network functions that also have this property.
The condition that $R \ge 1$ is only for convenience, and does not pose any constraint to our setting. 
Note that the estimation error in cases (2) improves from the generic case (1) by reducing the $D$ to the intrinsic dimensionality $d$,
and in the near-manifold case, the bound adds another term of $O(\sigma)$,
similar to the result of the approximation error. 
The analysis of case (3) again uses the integral comparison technique in Section \ref{sec:theory-approx}.

As a remark, if $p$, $q$ have compact support $\Omega$ which has a smaller volume than $B_R^D$, 
then the covering number ${\cal N}(\Omega,r)$ can be made $< {\cal N}(B_R^D,r)$,
the latter leading to the $-\frac{1}{2+D}$ rate in case (1). 
This is the fundamental reason that the rate can be improved to $-\frac{1}{2+d}$ for on or near manifold densities in cases (2) and (3).
Our proof technique bounds the estimation error by the covering complexity of the 
``essential'' support of the densities, allowing certain sub-exponential tail, 
which can be viewed as capturing the essential complexity of the densities. 
The manifold setting is a special case that is natural in applications.

\subsection{Testing power and consistency of network-logit test}
\label{subsec:test-consist}

We have now finished the 5 steps in Section \ref{subsec:roadmap} to prove that,
under Assumptions \ref{assump:training}, \ref{assump:ftar}, \ref{assump:netLip},
 the $\hat{f}_{\text{tr}}$
obtained by training a classifier on the training set has a $L[\hat{f}_{\text{tr}}]$
which is 
\begin{equation}\label{eq:summary-gap-C}
L_{\text{gap}} := \Delta C + O( \varepsilon) + O(\sigma) + \tilde{O}(n^{-1/(2+d')})
\end{equation}
close to $L[f_{\text{tar}}] = C > 0$, 
$f_{\text{tar}}$ being a smooth surrogate of $f^*=\log \frac{p}{q}$,
and $C$ can approach the unconstrained upper bound $L[f^*] = $JSD$(p,q)$ if $f^*$ itself is regular. 
In \eqref{eq:summary-gap-C}, 
$\varepsilon$ is the network approximation error of $L[f_{\text{tar}}]$, 
$\sigma$ is the sub-exponential decay scale of near-manifold densities (and does not appear for general or on-manifold density settings),
$\Delta C$ is the optimization error, and the last term decaying with $n$ is the estimation error,
where
\begin{itemize}
\item
For general sub-exponential densities in $\R^D$,
$d' = D$, the network complexity is $O(\varepsilon^{-D/r})$ for $r$-regular $f_{\text{tar}}$.

\item 
When $p$, $q$ are on or near a $d$-dimensional compact smooth manifold ${\cal M}$,
$d'=d$, the network complexity is $O(\varepsilon^{-d/2})$ for $C^2$ $f_{\text{tar}}|_{\cal M}$.
\end{itemize}
The constants in all the big $O$'s may depend on the network Lipchitz upper bound $L_\Theta$,
the function $f_{\text{tar}}$,
the diameter of the domain where the densities lie on (in case of exponential decay, the scale $c$ in ${\cal P}_{\text{exp}}$),
and the manifold-atlas if the densities are on-or-near manifold,
but not on $\varepsilon$.

When $L_\text{gap}$ is small enough,
the above result guarantees that $L[\hat{f}_{\text{tr}}] > 0$.
With an elementary lemma showing that $T[f] \ge 4 L[f]$,
which is Lemma \ref{lemma1}
(the tightness of the relaxation is explained in the remark below and in Lemma \ref{lemma2}),
it gives that $T[\hat{f}_{\text{tr}}] > 0$.
This sets a strictly positive expectation of the test statistic $\hat{T}$.
The testing consistency then follows by standard CLT, 
after proving a bounded variance of the test statistic. 
This gives the main theorem of this section:

\begin{theorem}\label{thm1}
Notations as before,
under Assumptions \ref{assump:training}, \ref{assump:ftar}, \ref{assump:netLip},
and the generic settings in Proposition \ref{prop:conc-supF-Ln},
let $\hat{f}_{\text{tr}}$ be the trained network function from $n_{tr}$ many samples of $X$ and $Y$,
and $T_n$ be the test statistic evaluated on the testing set where $|X_{\text{te}}| = |Y_{\text{te}}| = n$.
If $\Delta C$, $\varepsilon$, $\sigma$ are sufficiently small and $n_{tr}$ sufficiently large
to make $L_{\text{gap}} < C$,
where the needed network complexity is bounded as above,
then

(1) When $p \neq q$,
 $\E T_n  = T[ \hat{f}_{\text{tr}} ] > 0$,
and specifically $\frac{1}{4} \E T_n > C -  L_{\text{gap}}$.

(2) 
When $p = q$, 
$\sqrt{n} T_n  \to {\cal N}(0,  \sigma_{{\cal H}_0}^2)$ in distribution.
When $p \neq$ q,
$\sqrt{n} (T_n  - \E T_n) \to {\cal N}(0,  \sigma_{{\cal H}_1}^2)$ in distribution.
$\sigma_{{\cal H}_0}$ and $\sigma_{{\cal H}_1}$ are all bounded by 
$L_\Theta$ multiplied by an absolute constant.
\end{theorem}

The above Theorem directly gives the asymptotic test consistency, the proof of which is omitted.
\begin{corollary}\label{cor1-asymp-power}
Suppose Theorem \ref{thm1} holds and notations as therein,
and when  $p \neq q$, $T:= \E T_n > 4 (L[f_\text{tar}] - L_{\text{gap}}) > 0$.
Let $ 0 < \alpha < 1$ be the two-sample test level, typically $\alpha = 0.05$, 
and the test threshold be 
$\tau_n = \frac{\sigma_{{\cal H}_0}}{\sqrt{n}} \Psi^{-1}( \alpha)$, 
where $\Psi(x) := \int_x^\infty \frac{1}{\sqrt{2\pi}} e^{-y^2/2} dy$.
Then, as $n = n_{te} \to \infty$,
 $\Pr [T_n > \tau_n] \to \alpha$ if $p=q$,
 and when $p \neq q$,
\[
\Pr [T_n > \tau_n | {\cal H}_1]
\to 
1 - \Psi\left( \sqrt{n} \cdot \frac{T - \tau_n}{\sigma_{{\cal H}_1}} \right)
\ge
1 - \Psi\left( \sqrt{n} \cdot \frac{ 4 (L[f_\text{tar}] - L_{\text{gap}})  - \tau_n}{\sigma_{{\cal H}_1}} \right)
\]
which $\ge 1- c_0 e^{-c_1 n}$ for positive constants $c_0$ and $c_1$ (different from the other $c_0$, $c_1$ in the paper).
\end{corollary}

Going beyond the asymptotic consistency as $n\to \infty$,
one may obtain a lower bound of testing power based on the bound of $\E T_n $
and $\text{Var}(T_n)$ in Theorem \ref{thm1} by Chebyshev inequality. 
This will provide a finite-sample testing power 
which is positive and approaching 1 
as long as $(L[f_\text{tar}] - L_{\text{gap}}) $ can exceed
a constant multiple of $n_{te}^{-1/2}$.
Details omitted here.

\section{Experiments}\label{sec:experiment}

This section conducts numerical experiments of the proposed two-sample test and compares with alternatives,
on synthetic 1D and manifold densities and evaluating hand-written digits generating models. 
Codes to produce experimental results will be publicly online.

\begin{figure}[t]
\begin{center}
		\begin{tabular}{ c c c c}
		        \hspace{-1.25cm}
			\includegraphics[width=.25\textwidth]{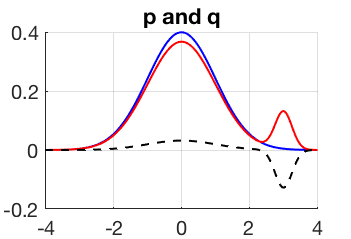}
			 & 
			\hspace{-0.55cm}
			\includegraphics[width=.25\textwidth]{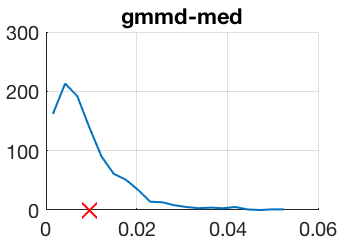} 
			&
			\hspace{-0.55cm}
			\includegraphics[width=.25\textwidth]{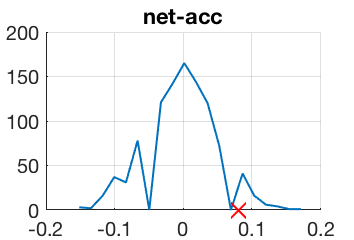}
			 &
			\hspace{-0.55cm}
			\includegraphics[width=.25\textwidth]{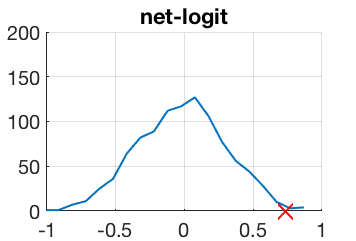} 
			 \\
			
			 \hspace{-1cm}
			 
			 \raisebox{1.25cm}{
			 \scriptsize
			 \stackunder[5pt]{
			 \begin{tabular}{ c|c c c  }
			\hline
                           				& {\it gmmd} 	&  {\it net-acc}  	& {\it net-logit}  \\
			 \hline
			mean 	         	& 19.14		& 19.98  			& 	78.09 	 \\
			std		         	&  1.95 		& 10.43			&	20.56		\\
			median  	 		& 19.63		& 17.63			& 	84.13	 \\
			 \hline
			\end{tabular}
			}{Test Power}
			}			
			 & 
			\hspace{-0.55cm}
			\includegraphics[width=.24\textwidth]{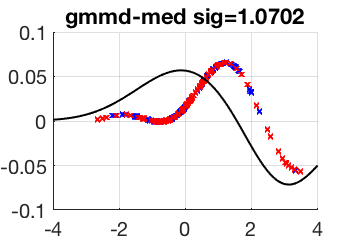} 
			&
			\hspace{-0.55cm}
			\includegraphics[width=.24\textwidth]{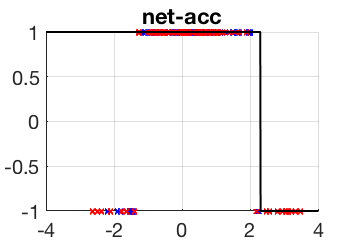} 
			&
			\hspace{-0.55cm}
			\includegraphics[width=.24\textwidth]{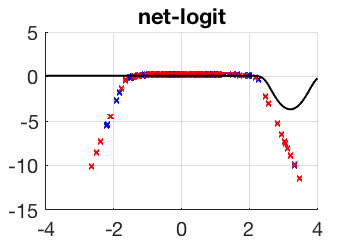}
		\end{tabular}
\end{center}
\vspace{-15pt}
\caption{
\small
{\bf Plots}: 
Top-left: Two densities  $p$ and $q$ in Eg.3, $\delta = 0.08$.
Right three columns:
The test statistic $\hat{T}$ on $|X_\text{te}|= |Y_\text{te}|=100$ samples
i.e. under ${\cal H}_1$ (red cross)
and the histogram of $\hat{T}$  under 1000 permutation tests i.e. ${\cal H}_0$ (blue curve).
The population witness function (black curve) and the empirical one evaluated on test data (red cross for $X_\text{te}$,
blue crosses for $Y_\text{te}$) of the three methods are shown in the bottom row.
{\bf Table}: 
The mean, standard deviation (``std") and median of the test power 
computed  over $n_{\text{rep}}=20$ replicas.
The randomness of test power estimation is explained in Appendix \ref{app:exp-1d-section}.}
\label{fig:1dexample}
\end{figure}

\begin{figure}[t]
\begin{center}
\hspace{-0.65cm}
 \raisebox{0.3cm}{
\includegraphics[height=0.125\textwidth,trim = {0cm 0cm 0cm 0cm}, clip]{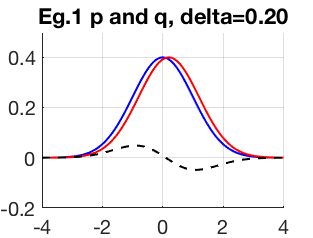}
}
\hspace{-0.4cm}
\includegraphics[height=0.15\textwidth,trim = {6cm 0cm 6cm 0cm}, clip]{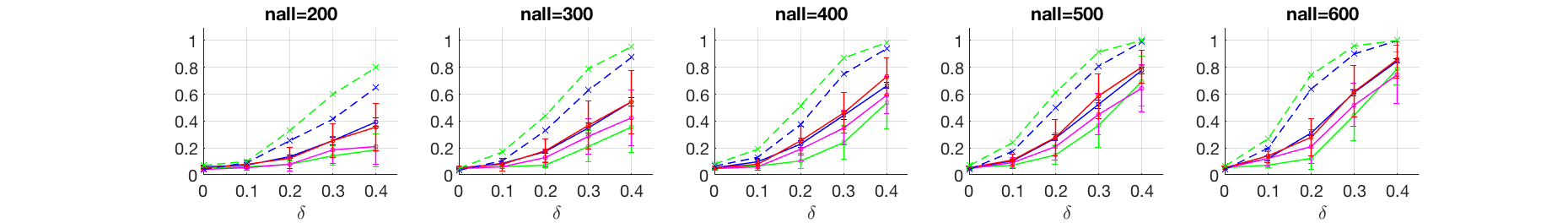}
\\
\hspace{-0.65cm}
 \raisebox{0.3cm}{
\includegraphics[height=0.125\textwidth,trim = {0cm 0cm 0cm 0cm}, clip]{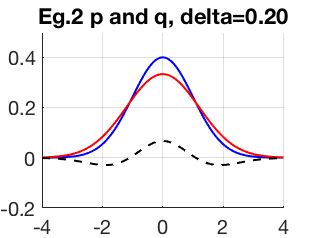}
}
\hspace{-0.4cm}
\includegraphics[height=0.15\textwidth,trim = {6cm 0cm 6cm 0cm}, clip]{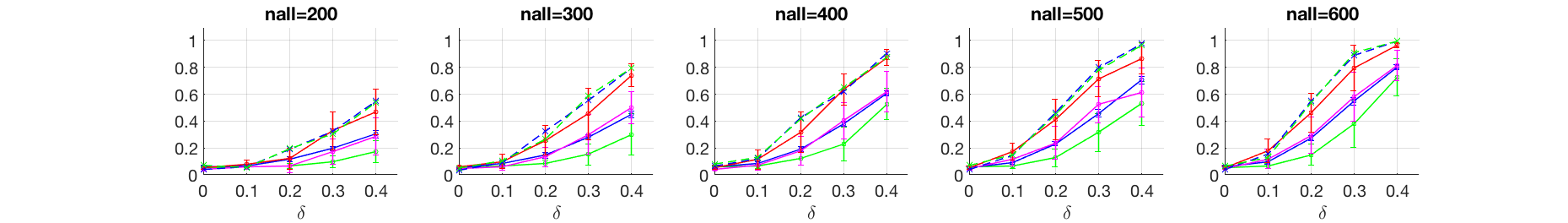}
\\
\hspace{-0.65cm}
 \raisebox{0.3cm}{
\includegraphics[height=0.125\textwidth,trim = {0cm 0cm 0cm 0cm}, clip]{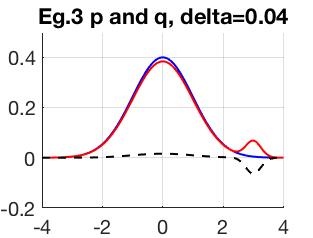}
}
\hspace{-0.4cm}
\includegraphics[height=0.15\textwidth,trim = {6cm 0cm 6cm 0cm}, clip]{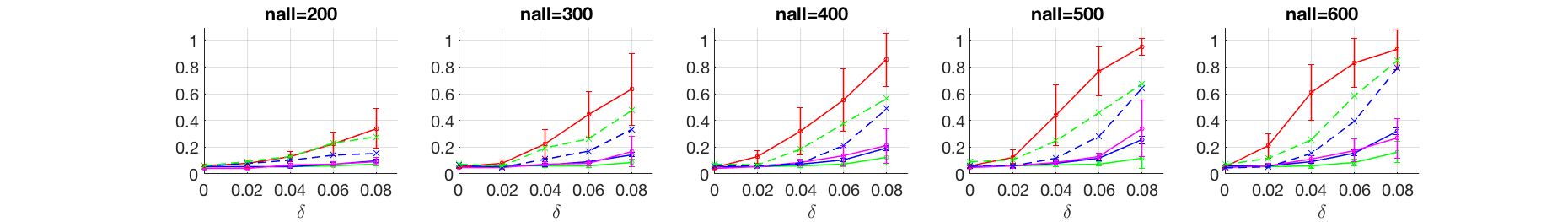}
\end{center}
\vspace{-20pt}
\caption{
Three examples of 1D data  in Section \ref{subsec:exp-1d}. 
Test power of:
{\it gmmd} (blue),
{\it gmmd-ad} (green),
{\it net-acc} (pink),
{\it net-logit} (red),
error bar standing for the standard deviation of the estimated power over 20 replicas,
and
{\it gmmd+} (blue dash),
{\it gmmd++} (green dash).
$n_{all} =|X| + |Y|$ including half-half training-testing split. 
}\label{fig:test5-eg3}
\end{figure}

\subsection{Synthetic 1D normal density departure}\label{subsec:exp-1d}

The following three examples all have
$x_i \sim {\cal N} (0,1)$
and

\begin{itemize}
\item
Eg.1. Mean shift, $y_i \sim {\cal N}(\delta , 1)$;

\item
Eg.2. Dilation of variance, $y_i \sim {\cal N}(0, (1+\delta)^2)$;

\item
Eg.3. Mixture with bump at tail, 
$y_j\sim (1-\delta){\cal N}(0,1) + \delta {\cal N}(3,\frac{1}{16})$.

\end{itemize}

We examine the tests:
(1) {\it net-acc},
(2) {\it net-logit}, 
(3) {\it gmmd} (setting kernel bandwidth $\sigma$ to be median distance),
(4) {\it gmmd-ad} (selecting $\sigma$ adaptively using the training set),
and
(5) {\it gmmd}+ (using all training and testing samples, median $\sigma$),
(6) {\it gmmd}++ (using all data and post-selecting $\sigma$ over a range).
Tests (1)-(4) only use the test set when measuring the power, while (5)-(6) access both the training and test sets.
More details about experimental set-ups are in Appendix \ref{app:exp-1d-section}.
The test powers of all the methods  are plotted in Fig. \ref{fig:test5-eg3} for the three examples. 
For Eg.1 and Eg.2, 
{\it gmmd}+, {\it gmmd}++ are performing consistently better than the other four which only access the test data set,
particularly in Eg.1.
In Eg.3,
 {\it net-logit} gives stronger power than  {\it gmmd}+, {\it gmmd}++
when $n_{\text{all}} > 200$,
where  {\it net-acc} remains inferior to the two.
Among the four methods (1)-(4),
the performances on Eg.1 are comparable,
and {\it net-logit} gives better power on Eg.2 and Eg. 3.
This is especially the case of Eg. 3, 
where {\it net-logit} shows the most significant advantage.
The test powers, empirical and population witness functions of tests (1)(2)(3) on E.g. 3 are shown in Figure \ref{fig:1dexample},
and more details in Appendix \ref{app:exp-1d-eg3-detail}.

\subsection{Comparison of test-power by witness function}

Generally, there is no best test methods to use for all data sets, and the test power depends on both method and data.
Here we study Eg. 3. in more detail, comparing the witness functions of the network-based and kernel-based methods,
so as to explain the empirical advantage of {\it net-logit} in this case. 

The analysis and experiments in Section \ref{sec:1dexample} show that,
with large samples and sufficiently large neural network, the trained witness function $\hat{f}_{tr}$
approaches the population witness function $f^*$ in terms of the population divergence $L[f]$. 
Similar theories hold for kernel MMD tests. Thus comparing test power based on population witness functions 
sheds light on the behavior of the tests.

{\bf Witness function.}
The population witness functions for {\it gmmd}
is $w_\sigma(x) := \int g_\sigma(x-y) (p(y)-q(y))dy$,  $g_\sigma(z):=e^{-|z|^2/(2\sigma^2)}$,
and its empirical counterpart is by replacing $p$ and $q$ with the empirical densities of $X_{\text{te}}$ and $Y_{\text{te}}$ respectively.
Recall that the population and empirical witness function for {\it net-logit} test are
$f^*(x) = \log \frac{p(x)}{q(x)}$ and $f_\theta$ respectively.
For  {\it net-acc}, 
when $|X_\text{te}|= |Y_\text{te}|$,
 it is equivalent to using the sign (taking value of $\pm 1$) of $f_\theta$ instead of $f_\theta$ 
in computing the test statistic in \eqref{eq:def-hatT-logratio}, as shown in \eqref{eq:def-hatT-netacc}.
Thus we call $\text{Sign}(f_\theta(x))$ the empirical witness function for the  {\it net-acc} test,
and $\text{Sign}(f^*(x))$ the population one.
The population and empirical witness functions (in one test run) are plotted in Fig. \ref{fig:1dexample}.
Comparing to {\it gmmd}, 
the witness function of {\it net-logit}, 
i.e., the log density ratio, 
weighs larger at the differential region which is at the tail of the density $p$.
Taking the sign of $f_\theta$ as done in {\it net-acc} test
introduces discontinuity of at the decision boundary neat $x=2$,
which leads to comparatively larger variance of $\hat{T}$.
This intuitively explains why the {\it net-logit} test performs better.

\begin{figure}[t]
\centering
\begin{tabular}{ c c }
\raisebox{2cm}{
\scriptsize
\begin{tabular}[t]{ c | c c c   }
\hline
                           	&  $\textbf{Mean}$ 	& 	$\textbf{Std}$	&  $\textbf{Mean}$/$\textbf{Std}$	 \\
 \hline
{\it gmmd}  	         & 0.0087  		& 	 0.0421  		&  0.2075  		\\
{\it net-acc}		&  0.1579  		&   	0.6087    		&  0.2594			   \\
 {\it net-logit}  	 	&  0.2445     		& 	 0.9011   		&  0.2714				\\
 \hline
\end{tabular}}
&
\hspace{0.5cm}
\includegraphics[height=0.175\textwidth]{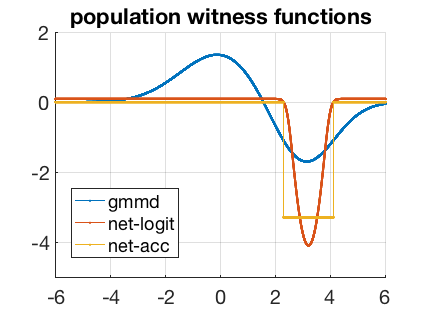}
\end{tabular}
\vspace{-10pt}
\caption{
\small
(Table) 
The values of $\textbf{Mean}$, $\textbf{Std}$, and their ratio 
of the three tests in Fig. \ref{fig:1dexample}.
(Plot)
The population witness functions of the three tests,
normalized to have $\textbf{Std}=1$.
}\label{fig:test4b}
\vspace{-20pt}
\end{figure}

{\bf  Quantitative comparison.}
To further explain the testing power difference,
we give a quantitative comparison.
Let $w$ be the population witness function of the three methods respectively, and define
\begin{align*}
\textbf{Mean}  := \E_{x\sim p, Y \sim q} (w(X)-w(Y)),  
\quad
\textbf{Std}  := \sqrt{ \text{Var}_{x\sim p}( w(X)) +  \text{Var}_{Y \sim q}(w(Y)) }.
\end{align*}
The larger the $\textbf{Mean}$, and the smaller the $\textbf{Std}$, the more powerful the test is going to be.
To remove the scaling equivalence (a test statistic multiplied by a positive constant gives the same test power),
 we will use the ratio of $\textbf{Mean}$ and $\textbf{Std}$ as an indicator of test power. 
 A similar approach has been taken to study kernel MMD \cite{serfling1981approximation, cheng2017two}.
 Using the explicit formula of $p$ and $q$ in Eg. 3,
 the values of $\textbf{Mean} $ and $\textbf{Std}$ can be analytically computed, shown in the table in Fig. \ref{fig:test4b},
 where {\it net-logit} gives the largest ratio.
 The normalized witness functions are plotted in Fig. \ref{fig:test4b},
where {\it net-logit} witness function gives the largest weights to the differential region of $p$ and $q$ in this example.

\begin{figure}[t]
\hspace{-.7cm}
 \raisebox{0.1cm}{
\includegraphics[height=0.16\textwidth,trim = {0cm 0.5cm 0cm 0.5cm}, clip]{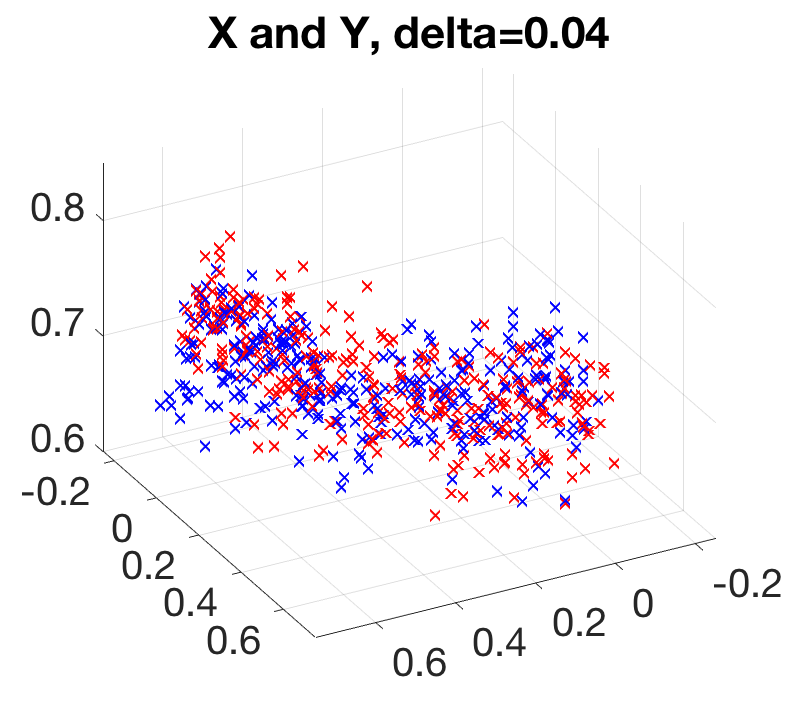}
}
\hspace{-0.7cm}
\includegraphics[height=0.155\textwidth,trim = {6cm 0cm 6cm 0cm}, clip]{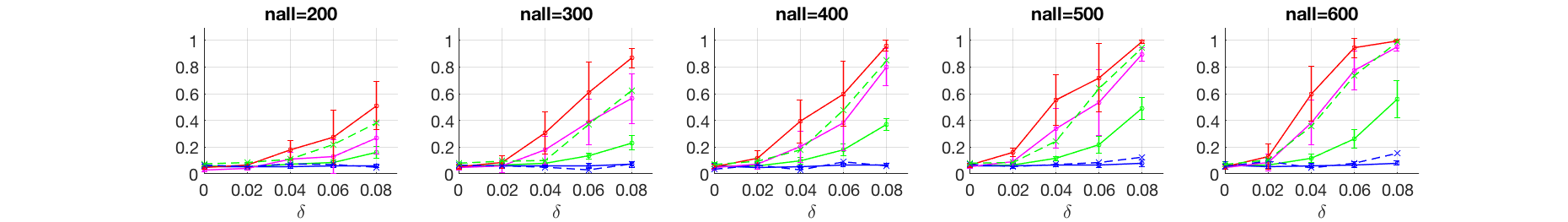}
\vspace{-20pt}
\caption{
\small
Test power of the different tests
on data on sphere in $\R^3$ in Section \ref{subsec:exp-2d}.
Markers same as in Fig. \ref{fig:test5-eg3}.
}\label{fig:test6-sphere}
\end{figure}

\subsection{Synthetic 2D manifold density}\label{subsec:exp-2d}

The example consists of $p$ and $q$ which lie on the sphere $S^2$, 
a 2-dimensional manifold embedded in $\R^3$.
A realization of samples $X$ and $Y$ is shown in the left of Fig. \ref{fig:test6-sphere},
and the formula of densities are given in Appendix \ref{app:detail-exp-2d}.
Fig. \ref{fig:test6-sphere} plots the test power of the 5 methods over increasing density departure $\delta$ and sample size.
It can been seen that {\it net-logit} gives the fastest growth of power as $\delta$ increases and the 
strongest average power for all $n_{\text{all}}$,
but the variation can be large if the power is not close to 1.
The {\it gmmd-ad} improves the power of  {\it gmmd},
but does not do as well as {\it net-acc}, 
which again performs inferior to {\it net-logit}. 
{\it net-logit} performs better than {\it gmmd++} (post-selecting $\sigma$, green dash)
and the advantage is more evident when $n_{\text{all}} > 200$.  
This indicates that larger sample size can be particularly in favor of network-based tests, 
which rely on the search in the network parameter space optimized on a separated training set.

\subsection{Generated vs authentic MNIST data}\label{subsec:exp-mnist}

\begin{figure}[t]
\begin{center}
\begin{tabular}{ c c c  }
\hspace{-0.25cm}
	\multirow{2}*{
	\hspace{-0.5cm} 
	 \stackunder[2pt]{
	\includegraphics[width=.35\textwidth, trim = {5.5cm 14.5cm 15cm 1cm}, clip]{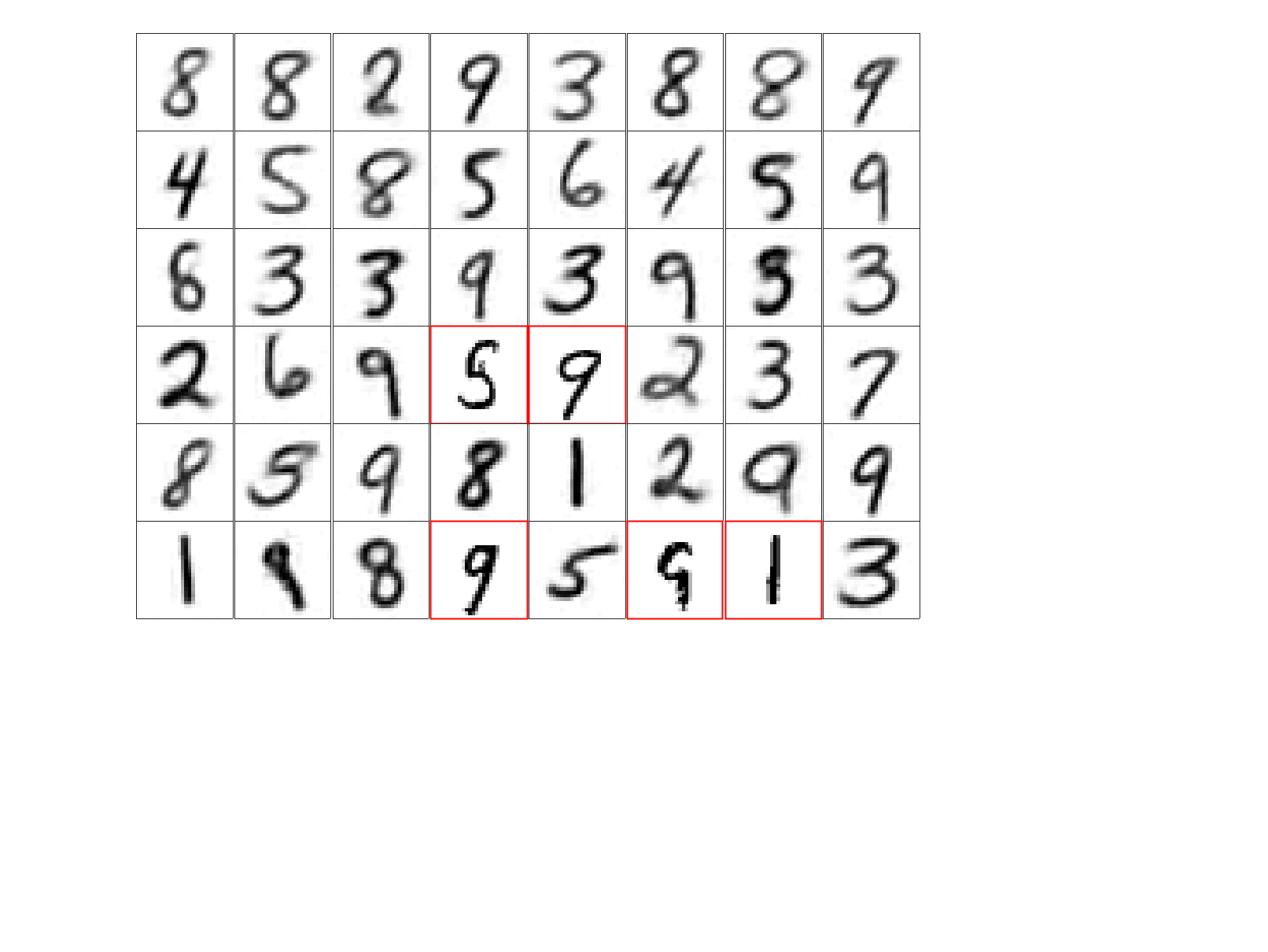}
	}{Detected by {\it gmmd-ad}}
	}
	&
	\multirow{2}*{
	\hspace{-0.75cm}
	 \stackunder[2pt]{
	\includegraphics[width=.35\textwidth, trim = {5.5cm 14.5cm 15cm 1cm}, clip]{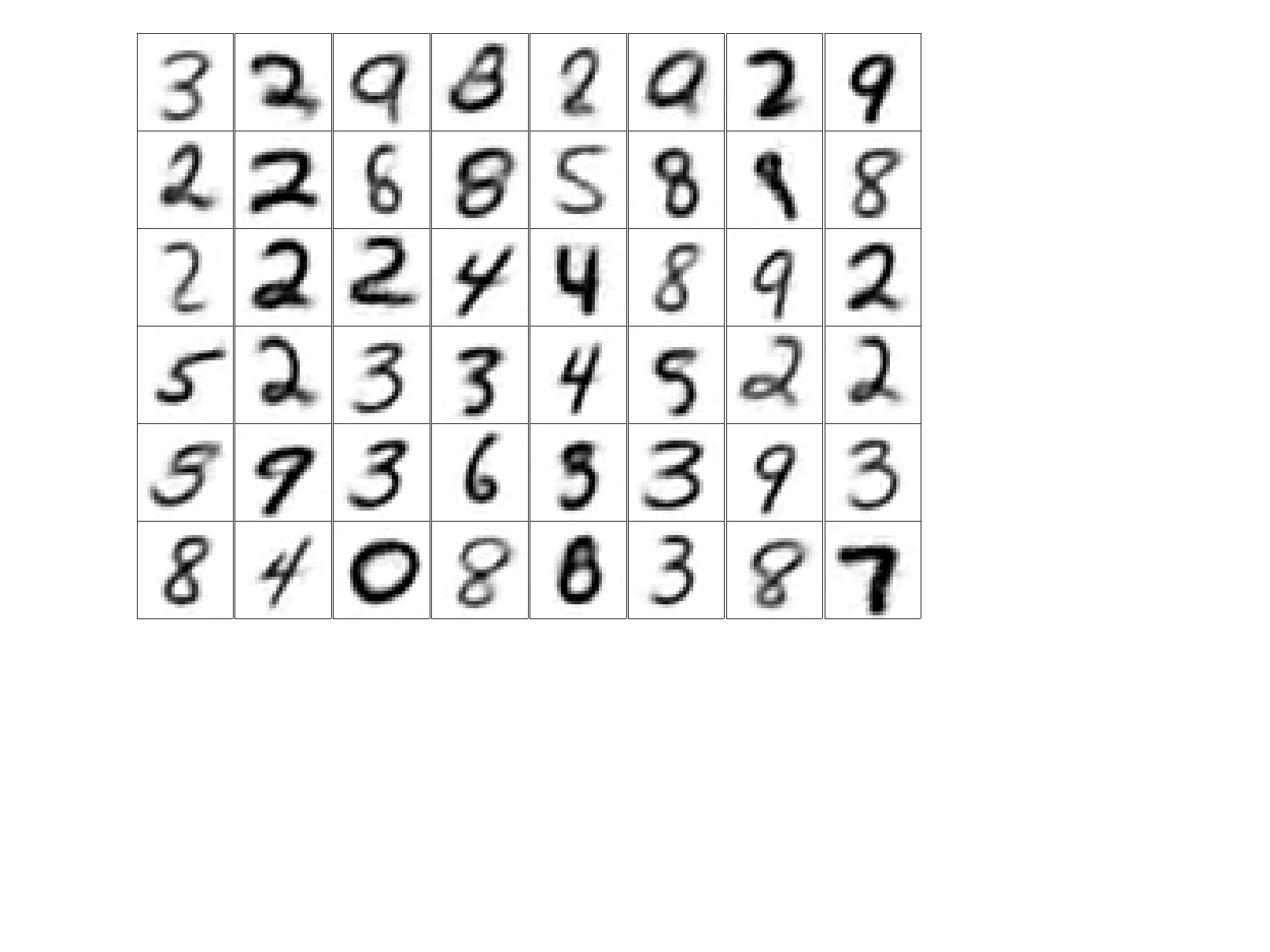}
	}{Detected by {\it net-logit}}
	}
	&
	\hspace{-0.75cm}
	 \raisebox{-2cm}{
		\includegraphics[width=0.3\textwidth]{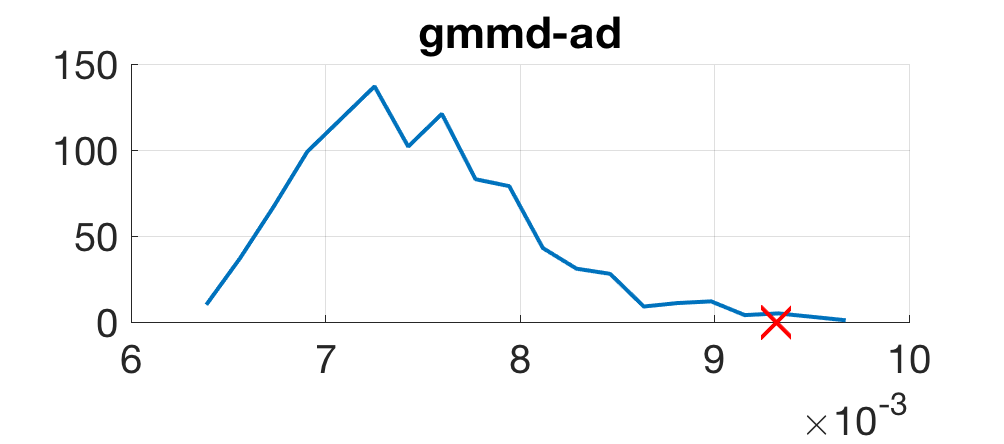} 
		}  
	\\
     & & 	
     \hspace{-0.5cm}
	 \raisebox{-0.5cm}{
     	\includegraphics[width=0.3\textwidth]{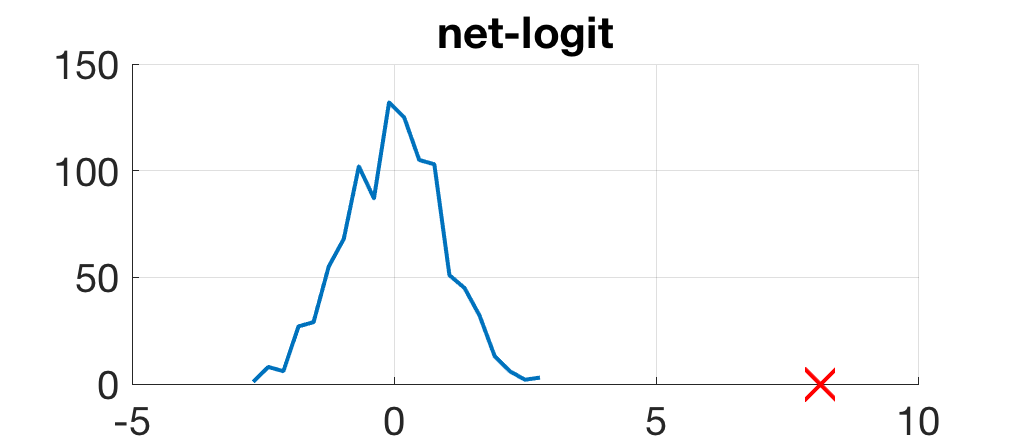}
	}
\end{tabular}
\end{center}
\vspace{-5pt}
\caption{
\small
Two-sample problem of differentiating 
$p$, the density of authentic MNIST digits, 
and $q$ which contains a $\delta=0.4$ fraction of 
digits ``faked''  by a generative model.
$|X|=|Y|=500$. 
The {\it gmmd-ad} and {\it net-logit} tests use half as training set,
and test on the other $|{\cal D}_{\text{te}}|=500$ samples.
Left and middle: 
the most likely fake digits identified by the empirical witness functions of the two tests,
red box indicates authentic digits incorrectly identified. 
Right: 
The test statistic $\hat{T}$ (${\cal H}_1$) and the histogram of its value under 1000 permutation tests (${\cal H}_0$).
}\label{fig:mnist-catch-fake}
\vspace{-20pt}
\end{figure}

\begin{figure}[t]
\centering
\includegraphics[height=0.16\textwidth]{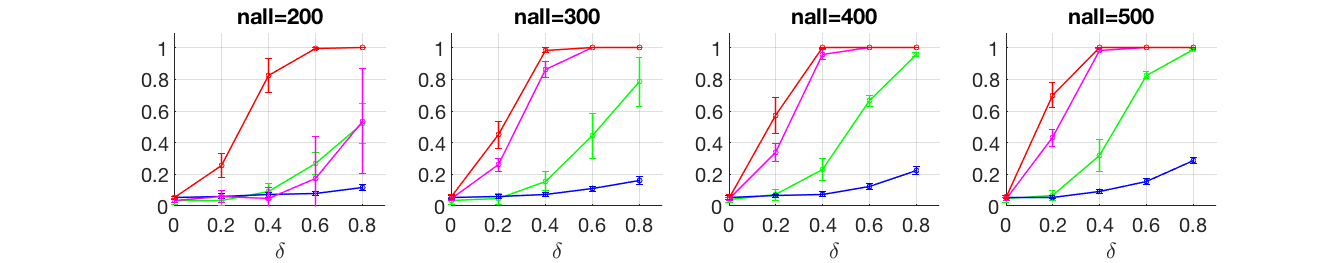}
\vspace{-10pt}
\caption{
\small
Test power of 
{\it gmmd} (blue)
 {\it gmmd-ad} (green)
{\it net-acc} (pink)
{\it net-logit} (red)
 on differentiating authentic vs synthesized MNIST digits produced by a generative model,
 where sample $X$ has all authentic ones,
and  $\delta$ stands for the fraction of synthesized ones in $Y$,
$n_{all} =|X| + |Y|$ including half-half training-testing split. 
}\label{fig:mnist-power}
\vspace{-10pt}
\end{figure}

As a real-world data example,
we study the task of distinguishing ``faked" MNIST samples produced by a pre-trained generative network
from authentic ones. 
The MNIST dataset consists of gray-scale hand-written digits of size $28 \times 28$ falling into 10 classes,
which is relatively simple and thus is viewed to lie near to low-dimensional manifolds in the ambient space of $\R^{784}$. 
More details about the generative and classification networks in Appendix \ref{app:mnist}. 
We compare 
(1) {\it net-acc}
(2) {\it net-logit}
(3) {\it gmmd} 
(4) {\it gmmd-ad}
on two samples $X$ and $Y$,  
half of ${\cal D}=X\cup Y$ used for training.
$X$ consists of authentic MNIST samples, and $Y$ of a mixture of authentic and faked ones,
i.e.  $p = p_{\text{data}}$ and $q = (1-\delta) p_{\text{data}} + \delta p_{\text{model}} $, $\delta \in [0,1]$.
The test power for increasing $\delta$ and sample size $n_{all} = |{\cal D}|$ up to 500 is shown in  Fig. \ref{fig:mnist-power},
where {\it net-logit} gives the strongest power throughout all cases,
and the two network-based tests significantly outperforms the other two
when $n_{all} \ge 300$.
The adaptive choice of kernel bandwidth also improves the power over the median-distance choice.
The standard deviation of the {\it net-acc} and {\it net-logit} power is less than that of {\it gmmd-ad} power when $n_{\text{all}}=300$ and $\delta \ge 0.4$, 
when the former two give near to 1 power. 
We also observe that the training of the CNN classifier in this experiment 
is  more stable than that of the previous fully-connected network on low-dimensional synthetic data,
as revealed in the training error evolution plots, c.f. Fig. \ref{fig:training-detail-synthetic} Fig. \ref{fig:training-detail-mnist}. 
With another pre-trained model which generates faked images that are closer to authentic ones,
 {\it net-logit} again shows the best discriminative power,
 {\it net-acc} gives comparable performance starting $n_{\text{all}} = 300$,
while {\it gmmd} and {\it gmmd-ad} gives trivial power up to $n_{all} = 500$,
c.f. Fig. \ref{fig:mnist-power-dim128}.

Setting $n_{\text{all}} = 1000$, $\delta = 0.4$,
the results of {\it gmmd-ad} and {\it net-logit}
in one test run is shown in Fig. \ref{fig:mnist-catch-fake}.
Based on the $n_{\text{all}} = 500$ plot in Fig. \ref{fig:mnist-power},
both tests shall have non-trivial power, and that of  {\it net-logit} shall be close to 1.
In this test,
both methods correctly rejects ${\cal H}_0$,
yet the {\it net-logit} statistic deviates from the distribution of $\hat{T}| {\cal H}_0$ more significantly,
indicating stronger power (shown in the histogram plots).
To compare the detecting ability of the empirical witness function $\hat{w}$ of  {\it gmmd-ad} and {\it net-logit},
for each method,
we sort the 250 samples in $Y_{\text{te}}$ (among which 100 are faked ones)
in ascending order of the value of $\hat{w}$ and select the first 100 samples.
These are samples which the model views as most likely to be faked ones.
The success rate of identifying faked samples is $\sim 60$ by  {\it gmmd-ad} $\hat{w}$,
and $\sim 90$  by  {\it net-logit} $\hat{w}$.
The first 48 most likely faked digits identified with both  witness functions are plotted in Fig. \ref{fig:mnist-catch-fake},
where {\it gmmd-ad} $\hat{w}$ incorrectly includes 5 authentic samples,
and none by {\it net-logit} $\hat{w}$.

\section{Proofs}\label{sec:main-proofs}

\subsection{Proofs in Section \ref{sec:theory-approx}}

\subsubsection{Proofs in Section \ref{subsec:setup-manifold}}

\begin{proof}[Proof of Lemma \ref{lemma:manifold-in-tube}]
For any $x \in U_i = B(x_i, \delta)\cap {\cal M}$,
we have that (note that $\phi_i(x_i) = x_i$)
\[
\| x - x_i\|
\le d_{\cal M}(x, x_i)
\le \beta_i \| \phi_i(x) - x_i\|,
\]
where the first inequality is by that geodesic distance is always larger than the Euclidean distance,
and the second inequality is by \eqref{eq:equiv-norm-Ui}.
This gives that 
\begin{equation}\label{eq:bound1-lemmalip-etai}
\| \phi_i(x) - x_i\| \ge \frac{1}{\beta_i} \| x - x_i\|.
\end{equation}
Meanwhile, as $\phi_i$ is an orthogonal projection, we have that
$
\| x - x_i\|^2  =
\| \phi_i(x) - x_i \|^2 + 
\| x- \phi_i(x)\|^2$,
and then with \eqref{eq:bound1-lemmalip-etai} it gives that
\[
\| x- \phi_i(x)\| \le  \| x - x_i\| \sqrt{1-1/{\beta_i^2} } 
<\delta \sqrt{1-1/{\beta_i^2} },
\]
which is further bounded by $\frac{\sqrt{3}}{2} \delta$ by $\beta_i \le 2 $ as in \eqref{eq:global-alpha-beta}.
\end{proof}

\begin{proof}[Proof of Theorem \ref{thm:mainthm-shaham2018provable}]
Under the setup of manifold atlas $\{ (U_i, \phi_i) \}_{i=1}^K$, 
the network function is
\begin{equation}\label{eq:def-fN-1}
f_N(x) = \sum_{i=1}^K \bar{f}_{N, i}(x), \quad x \in \R^D,
\end{equation}
where each $\bar{f}_{N, i}$ is constructed near $U_i$
in the following way.
For each $i$, we work with the local coordinates 
$x=(u,v)$, $u:=\phi_i(x) \in \R^d$,
and $v \in \R^{D-d}$, and we use the bar to denote that the function is defined in $\R^D$.
We utilize a modified construction of 
 Eqn. (41) (45) (46) in \cite{shaham2018provable}, which are the following 
 
\begin{align}
\hat{f}_{ i}(u)
& := \sum_{(k,b)} c_{k,b} {\psi}_{k,b}(u), 
\quad k = 0, 1, 2, \cdots, \quad  b \in 2^{-\frac{k}{d}} \Z_d,
\label{eq:fNi-expansion}
\\
{\psi}_{k,b}(u) 
& := 2^{\frac{k}{2}}\left( {\varphi}_{k,b}(u) - 2^{-1}{\varphi}_{k-1,b}(u) \right), 
\label{eq:def-barpsikb}
\\
{\varphi}_{k,b}(u) 
& := c_d \text{Relu}\left( \sum_{j=1}^d t( 2^{\frac{k}{d}}(u_j -b_j)  )  -2(d-1) 
\right),
\label{eq:def-barphikb}
\end{align}
where
$c_d$ is a constant depending on $d$, c.f. Eqn. (19) (20) in \cite{shaham2018provable};
$t: \R \to \R$ is a trapezoid-shaped continous function, $t(x) =2$ for $-1< x < 1$, 0 for $|x| > 3$, and linearly interpolating in between,
and $t$ can be written as the sum of four Relu activations in 1D,
c.f. Eqn. (18) in  \cite{shaham2018provable}.
The finite-term summation $\hat{f}_{N,i}$
will be a truncation $k\le k_{max}$ of $\hat{f}_i$,
and the network function
\begin{align}
\bar{f}_{N,i}(x) 
& := \text{Relu}( \text{Relu}(\hat{f}_{N,i}(u)) + F_0 g_\delta (v) - F_0) 
\nonumber \\
& ~~~ - \text{Relu}( \text{Relu}(-\hat{f}_{N,i}(u)) + F_0 g_\delta (v) - F_0), 
\quad F_0 := \| f\|_{L^\infty({\cal M})}  + 1,
\label{eq:def-fNi-new}
\end{align}
where $g_\delta: \R^{D} \to \R$ will be constructed such that, 
$C_1$ and $C_2$ being absolute constants,
\begin{itemize}
\item[(1)] $g_\delta$ is continuous on $\R^{D}$,
supp($g_\delta$) $\subset B_\delta^{D}$, $0 \le g_\delta(v) \le 1$ and $g_\delta(v) = 1$ if $\|v\| \le \frac{\sqrt{3}}{2}\delta$.

\item[(2)] $g_\delta$ is piece-wise differentiable on $\R^{D}$, and $\| \nabla g_\delta \| \le  \frac{C_1}{\delta}$ when differentiable. 

\item[(3)] $g_\delta(v)$ can be represented by a network with $\le C_2 D \frac{\log D + \log 1/\delta}{\delta}$ many parameters.
\end{itemize}
Note that while $v = x - \phi_i(x)$ lies in $(T_{x_i}({\cal M}))^{\perp}$ which is $(D-d)$-dimensional space, 
in practice our network $g_\delta$ takes the coordinates of $v$ in $\R^D$ (to avoid $D$-by-$D$ dense connection in the change of coordinate layer), thus $g_\delta$ is constructed to be a mapping from $\R^D$ to $\R$ with the properties (1)-(3).

\begin{lemma}\label{lemma:construct-g-delta}
Given $ 0 < \delta \le 1$, 
a function $g_\delta: \R^{D} \to \R$ that satisfies (1)(2)(3) can be constructed.
\end{lemma}

The approximation construction
proceeds by setting the coefficients $c_{k,b} = \langle f\eta_i, \tilde{\psi}_{k,b} \rangle$ in \eqref{eq:fNi-expansion},
where $\eta_i$ is the partition of unity function on $U_i$,
the function $ f\eta_i$ is viewed as a function on $\R^d$ (due to one-to-one correspondence between $U_i$ and $\phi_i(U_i)$)
which is supported on $\phi_i(U_i)$,
and the inner product is taking on $\R^d$.
The function $\tilde{\psi}_{k,b} $ is the dual of basis $\psi_{k,b}$,
which is compactly supported on $\R^d$.
Concerning the convergence of the infinite summation,
first observe that
while $b$ is on an infinite grid in $\R^d$ as in \eqref{eq:fNi-expansion}, 
since the functions $\varphi_{k,b}(u)$ are compactly supported, 
only finitely many $b$ are involved in the summation for each $k$.
Specifically, the $2^{-\frac{k}{d}}$ spacing of $b_j$ in each of the $d$ direction 
matches the wavelet basis spacial rescaling $2^{\frac{k}{d}}$ in \eqref{eq:def-barphikb},
then one can verify that (c.f. Remark C.5 in \cite{shaham2018provable})

\begin{itemize}
\item[{ \bf (P1)}]
In \eqref{eq:fNi-expansion}, at each $u \in \R^d$, for each $k$, at most $12^d$ many $b$'s are involved in the summation. 
\end{itemize}

Another important property is the decaying of the wavelet coefficients $c_{k,b}$ as $k$ increases due to the $C^2$ regularity of the function 
$f\eta_i$ on $\R^d$, which is Proposition C.4 in \cite{shaham2018provable}:

\begin{itemize}
\item[{ \bf (P2)}]
For any $b$ and $k = 0,1,2,\cdots$, $|c_{k,b}| \le C_{d,f, {\eta_i}} 2^{- ( \frac{2k}{d} + \frac{k}{2}) }$,
the constant depending on the 2nd derivative of $f\eta_i$ as a function on $\R^d$ and the dimension $d$. 
\end{itemize}

For any $\epsilon < 1$, let $\hat{f}_{N,i}(u)$ be the $k \le k_{max}$ truncation of \eqref{eq:fNi-expansion},
and $k_{max}$ large enough such that
\begin{equation}\label{eq:truncate-by-epsilon}
\sum_{k > k_{max}} 12^d \cdot C_{d,f, {\eta_i}} 2^{- ( \frac{2k}{d} + \frac{k}{2}) }  4 c_d 2^{k/2} < \epsilon,
\end{equation}
which is possible due to the summability of $2^{-2k/d}$.
By (P1), (P2), this proves that
\begin{equation}\label{eq:uniform-approx-hatfNi}
| \hat{f}_{N,i}(u) - f\eta_i(u) | < \epsilon, \quad \forall u \in \R^d \text{  in the local coordinate on $T_{x_i}({\cal M})$}.
\end{equation}
We now claim that, with the definition \eqref{eq:def-fNi-new},
\begin{equation}\label{eq:uniform-approx-barfNi}
|\bar{f}_{N,i}(x) - f\eta_i(x)| < \epsilon, \quad \forall x \in {\cal M} \text{ even not in $U_i$},
\end{equation}
and to verify this, consider three cases respectively,

(i) $x=(u,v) \in U_i$. By Property (1) of $g_\delta$ and Lemma \ref{lemma:manifold-in-tube}, $g_\delta(v)=1$.
Then \eqref{eq:def-fNi-new} becomes 
\[
\bar{f}_{N,i}(x)  
=\text{Relu}( \text{Relu}(\hat{f}_{N,i}(u)) ) 
- \text{Relu}( \text{Relu}(-\hat{f}_{N,i}(u)) ) = \hat{f}_{N,i}(u),
\]
and then \eqref{eq:uniform-approx-barfNi} follows by \eqref{eq:uniform-approx-hatfNi}.

(ii)  $x \notin U_i$, but $\phi_i(x) \in \phi_i(U_i)$. 
This only happens if $\| x - \phi_i(x)\| > \delta$. 
(Otherwise, since $\| \phi_i(x) - x_i \| < \delta$, 
then $\| x - x_i\| < 2\delta$. 
This means that $x \in B_{2\delta}^D(x_i) $,
while $B_{2\delta}^D(x_i) \cap {\cal M}$ is isomorphic to ball in $\R^d$
by construction, c.f. beginning of Section \ref{subsec:setup-manifold},
thus $\phi_i(x) \in \phi_i(U_i)$ only when $x \in U_i$, drawing a contradiction.)
Thus in the local coordinates $x=(u,v)$, $u = \phi_i(x)$, $\|v\| > \delta$,
and then by Property (1) of $g_\delta$, $g_\delta(v) = 0$.
Note that while $f\eta_i(x) =0$,
$f\eta_i(u) $ may not be zero due to that $u = \phi_i(x) \in \phi_i(U_i)$. However, $|f\eta_i(u)| \le |f(u)| \le \|f\|_{L^\infty({\cal M})}$,
and then \eqref{eq:uniform-approx-hatfNi} gives that
$|\hat{f}_{N,i}(u)| < \epsilon + |f\eta_i(u)| < 1+ \|f\|_{L^\infty({\cal M})} = F_0$.
Inserting into \eqref{eq:def-fNi-new} gives that \[
\bar{f}_{N,i}(x)  = \text{Relu}( \text{Relu}(\hat{f}_{N,i}(u))  - F_0) 
 - \text{Relu}( \text{Relu}(-\hat{f}_{N,i}(u))  - F_0) = 0, 
\] 
thus  \eqref{eq:uniform-approx-barfNi} holds.

(iii) $x \notin U_i$, and $\phi_i(x) \notin \phi_i(U_i)$. 
Again $f\eta_i(x) = 0$.
Since $f\eta_i(u)$ vanishes outside $\phi_i(U_i)$, $f\eta_i(u) =0$.
By \eqref{eq:uniform-approx-hatfNi}, $| \hat{f}_{N,i}(u)| < \epsilon$. 
Note that $g_\delta(v)$ may not be zero, but remains between $0$ and $1$.
Note that $\text{Relu}(\hat{f}_{N,i}(u)) = (\hat{f}_{N,i}(u))^+ \in [0, \varepsilon)$,
and then 
\begin{equation}\label{eq:barfNi+}
\bar{f}_{N,i}(x)^+ :=\text{Relu}( (\hat{f}_{N,i}(u))^+   +F_0 g_\delta(v) - F_0) \le  (\hat{f}_{N,i}(u))^+ < \epsilon .
\end{equation}
Similarly, 
\begin{equation}\label{eq:barfNi-}
\bar{f}_{N,i}(x)^- :=
\text{Relu}( (\hat{f}_{N,i}(u))^-   +F_0 g_\delta(v) - F_0) \le (\hat{f}_{N,i}(u))^- < \epsilon.
\end{equation}
For each $u$, only one of $(\hat{f}_{N,i}(u))^{\pm}$ is non zero, 
and when $(\hat{f}_{N,i}(u))^{+}=0$ then so is $\bar{f}_{N,i}(x)^+ $, same for minus. 
Thus $\bar{f}_{N,i}(x) = \bar{f}_{N,i}(x)^+ - \bar{f}_{N,i}(x)^- $
satisfies that  $|\bar{f}_{N,i}(x)| < \epsilon$, which proves  \eqref{eq:uniform-approx-barfNi}.

Given  \eqref{eq:uniform-approx-barfNi}, back to \eqref{eq:def-fN-1}, and by that $\sum_{i=1}^K \eta_i = 1$ on ${\cal M}$, 
for any $x \in {\cal M}$,
\[
| f_N(x) - f(x)| \le \sum_{i=1}^K | \bar{f}_{N,i}(x) -  f\eta_i(x) | < \epsilon K.
\]
Setting $\epsilon := \frac{\varepsilon}{K}$ to begin with, 
which determines $k_{max}$ in \eqref{eq:truncate-by-epsilon},
finishes the approximation part of the Theorem, 
and to verify the claimed number of neural network parameters,
we use big $O$ to denote multiplying an absolute constant here:

\begin{itemize}
\item
The first layer which conducts change of coordinate of $x \in \R^D$ to local coordinates around  $U_i$, for each $i$,
 takes $O(dD)$ parameters. 
 Because $ u \in \R^d$ is determined by the projection $\phi_i$, which is a $D$-to-$d$ linear transform, of $x_c:=(x-x_i)$, and $v = x_c - \phi_i (x_c)$ can be computed in another layer which has $O(dD)$ weights.

 \item 
 On the branch sub-network for each $i$,
 the layer which produces $\bar{f}_{N,i}$ takes the input of $\hat{f}_{N,i}(u)$ and $g_\delta(v)$ and uses $O(1)$ weights.
In $\hat{f}_{N,i}(u)$,
the number of basis $\# \{(k,b)\} = \sum_{k=0}^{k_{max}} (2\delta)^d 2^{k} = O(\delta^d 2^{k_{max}})$,
because the diameter of $\phi_i(U_i) \le 2\delta$ due to $U_i \subset  B_\delta(x_i)^D$,
 thus a grid of $b$ on $(-\delta,\delta)^d$ suffices.
 $k_{max}$ is set in \eqref{eq:truncate-by-epsilon} which satisfies $2^{k_{max}} \le (\frac{ \varepsilon }{K C_{d,f,\eta_i}'} )^{-d/2}$,
 where $C_{d,f,\eta_i}'$ is a constant depending on $f\eta_i$, $d$ and atlas. 
 Let $C':= \max_{i=1}^K C_{d,f,\eta_i}'$, 
 and then $2^{k_{max}} \le \varepsilon^{-d/2} (K C')^{d/2}$.
 We use this $k_{max}$ for all atlas $i$ subnet, 
thus  $\# \{(k,b)\} \le O( \delta^d \varepsilon^{-d/2} (K C')^{d/2}) $.
The bottom layer which constructs $\varphi_{k,b}(u)$  takes $O(d \# \{(k,b)\})$ many weights,
and can be shared across $i$. 
The upper layer which linearly combines $\psi_{k,b}$ from $\varphi_{k,b}$'s to form $\hat{f}_{N,i}$ 
involves $i$-atlas specific coefficients $c_{k,b}$, and uses $O( \# \{(k,b)\})$ many weights.
In $g_\delta(v)$, the 2 layer network uses $O( \frac{D}{\delta} \log \frac{D}{\delta} )$ parameters
 by Property (3) of $g_\delta$, and are shared across all $i$.

\end{itemize}

Summing over the $K$ atlas sub-networks, the total number of parameters is, omitting absolute constant,
\begin{align}
& K  dD  +  (d + K)\delta^d \varepsilon^{-d/2} (K C')^{d/2} + \frac{D}{\delta} \log \frac{D}{\delta}
 = C_{f, {\cal M}}   \varepsilon^{-d/2} + N_0, 
 \\
&  C_{f, {\cal M}}  : = (d + K)\delta^d  (K C')^{d/2} ,
\quad 
N_0 := K  dD  + \frac{D}{\delta} \log \frac{D}{\delta}.
\end{align}

We now prove the Lipschitz constant part of the Theorem.
For fixed $i$,
we first bound $\text{Lip}_{\R^D}( {\bar{f}_{N,i}} )$. 
Working with local coordinates $x =(u,v)$, for any $x, x' \in \R^D$, 
\begin{align}
& |{\bar{f}_{N,i}} (x) - {\bar{f}_{N,i}} (x')|
\le | \bar{f}_{N,i}(x)^+  - \bar{f}_{N,i}(x')^+ | + | \bar{f}_{N,i}(x)^-  - \bar{f}_{N,i}(x')^- |
\nonumber \\
& \le 2( |\hat{f}_{N,i}(u) - \hat{f}_{N,i}(u')| + F_0 |g_\delta(v) -  g_\delta(v')|  ).
\label{eq:bound-Lip-fNi-1}
\end{align}
To bound  $\text{Lip}_{\R^d}( \hat{f}_{N,i} )$,
recall that
\[
\hat{f}_{N,i} (u) = \sum_{k=0}^{k_{max}} \sum_b c_{k,b} \psi_{k,b}(u),
\]
and in \eqref{eq:def-barpsikb}, \eqref{eq:def-barphikb},
by that $t$ is piece-wise differentiable on $\R$, the function ${\varphi}_{k,b}$ is piece-wise differentiable on $\R^d$,
and then so is ${\psi}_{k,b}$. We have that 
\begin{equation}\label{eq:bound-grad-phikb}
\| \nabla {\varphi}_{k,b}(u)  \|
\le 
|c_d| \|2^{k/d} \mathbbm{1}_d \|
=
|c_d| (  d\cdot 2^{2k/d} )^{1/2}
=: c_{d}' \cdot 2^{k/d},
\quad \forall k=0,1,2,\cdots,
\end{equation}
where $c_{d}'$ is a constant depending on $d$.
Then, 
\begin{equation}\label{eq:bound-nabla-barpsikb}
\| \nabla {\psi}_{k,b}(u) \| 
\le 
2^{\frac{k}{2}} \left( 
\| \nabla {\varphi}_{k,b}(u) \|  + 2^{-1} \| \nabla {\varphi}_{k-1,b}(u) \|  \right)
\le 
2^{\frac{k}{2}}  c_{d}'  ( 2^{\frac{k}{d}} +  2^{-1}  2^{\frac{k-1}{d}})
\le  \frac{3}{2} c_{d}' 2^{\frac{k}{2}}2^{\frac{k}{d}},
\end{equation}
and then, for any $u, u' \in \R^d$,
\begin{align}
& |\hat{f}_{N,i} (u) - {\hat{f}_{N,i}} (u')|
 = \left|
\sum_{k=0}^{k_{max}} \left(  \sum_{b \in {\cal B}_u} c_{k,b} {\psi}_{k,b}(u) -  \sum_{b \in {\cal B}_{u'}} c_{k,b} {\psi}_{k,b}(u')  \right) 
\right| \\
& ~~~
\text{ ($B_u$ denoting the set of $b$'s involved in the summation)}
\nonumber
\\
& \le
\sum_{k=0}^{k_{max}} \sum_{b \in {\cal B}_u \cup  {\cal B}_{u'}} |c_{k,b}| |{\psi}_{k,b}(u) - {\psi}_{k,b}(u')|
\\
& \le 
\sum_{k=0}^{k_{max}}  (2\cdot 12^d) C_{d,f, {\eta_i}} 2^{- ( \frac{2k}{d} + \frac{k}{2})} 
\cdot \frac{3}{2} c_{d}' 2^{\frac{k}{2}}2^{\frac{k}{d}} \| u - u'\|
\\
& ~~~
\text{
(by (P1), (P2), \eqref{eq:bound-nabla-barpsikb} and  piece-wise continuous differentiability of ${\psi}_{k,b}$ in $\R^d$)
}
\nonumber
\\
& =:  \|u - u'\| C_{f, {\cal M}} \sum_{k=0}^{k_{max}} 2^{-  \frac{k}{d} }  \\
& \le \|u - u'\| C_{f, {\cal M}} (1- 2^{-1/d})^{-1} 
=: \|u - u'\| C_{f, {\cal M}}',
\end{align}
where $C_{f, {\cal M}} $ is a constant depending on function $f$ and manifold-atlas, including $d$ and $\eta_i$, and so is $C_{f, {\cal M}} '$.

Meanwhile, Property (2) of $g_\delta$ gives that $\text{Lip}_{\R^{D-d}}(g_\delta) \le \frac{C_1}{\delta}$, $C_1$ being absolute constant.
Then \eqref{eq:bound-Lip-fNi-1} continues as below, $x=(u,v)$ being orthogonal decomposition,
\[
\le 2(  C_{f, {\cal M}}'  \|u - u'\| + F_0 \frac{C_1}{\delta} \|v - v'\|)
\le 2 ( C_{f, {\cal M}}' +  F_0 \frac{C_1}{\delta}) \sqrt{2} \| x - x'\|=: C_{f,{\cal M}}'' \| x - x'\|,
\]
and this proves that 
$\text{Lip}_{\R^D}( {\bar{f}_{N,i}} ) \le C_{f, {\cal M}}''$.
By \eqref{eq:def-fN-1},
\[
\text{Lip}_{\R^D}( f_{N} ) \le  K C_{f, {\cal M}}'' =: L_{{\cal M}, f}
\]
which is a constant depending on $f$ and manifold-atlas only.
\end{proof}

\begin{figure}
\centering{
 \raisebox{3pt}{
\includegraphics[trim={50pt 300pt 650pt 180pt},clip,height=.265\linewidth]{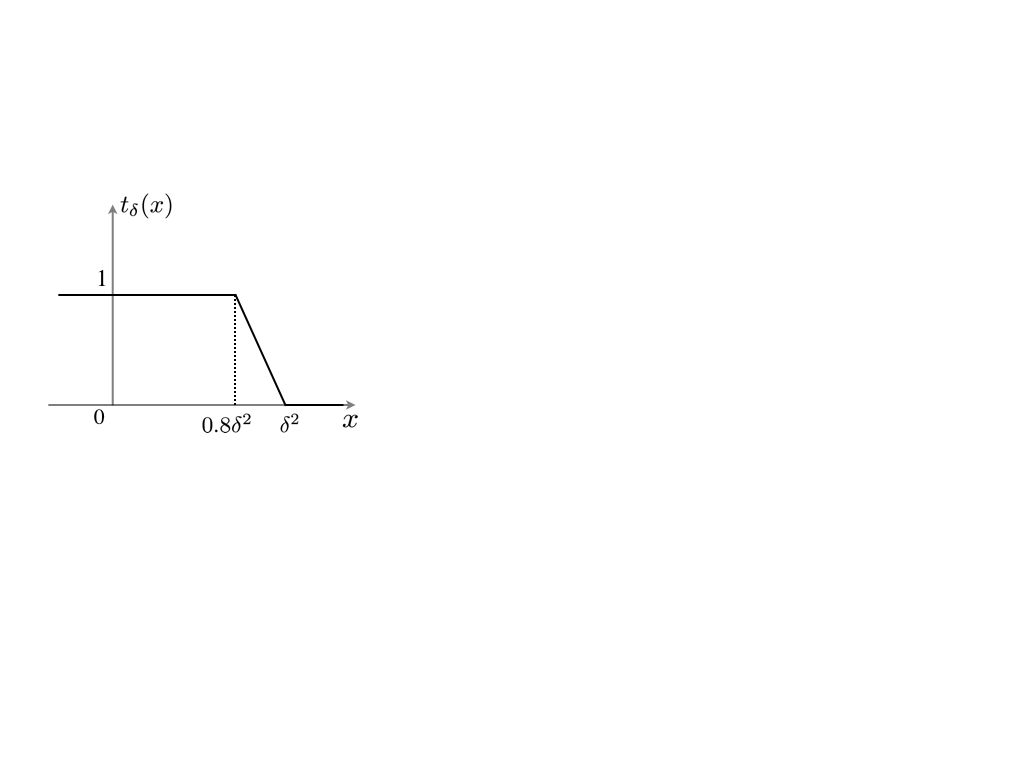}
}
\hspace{-10pt}
\includegraphics[trim={120pt 300pt 340pt 180pt},clip,height=.27\linewidth]{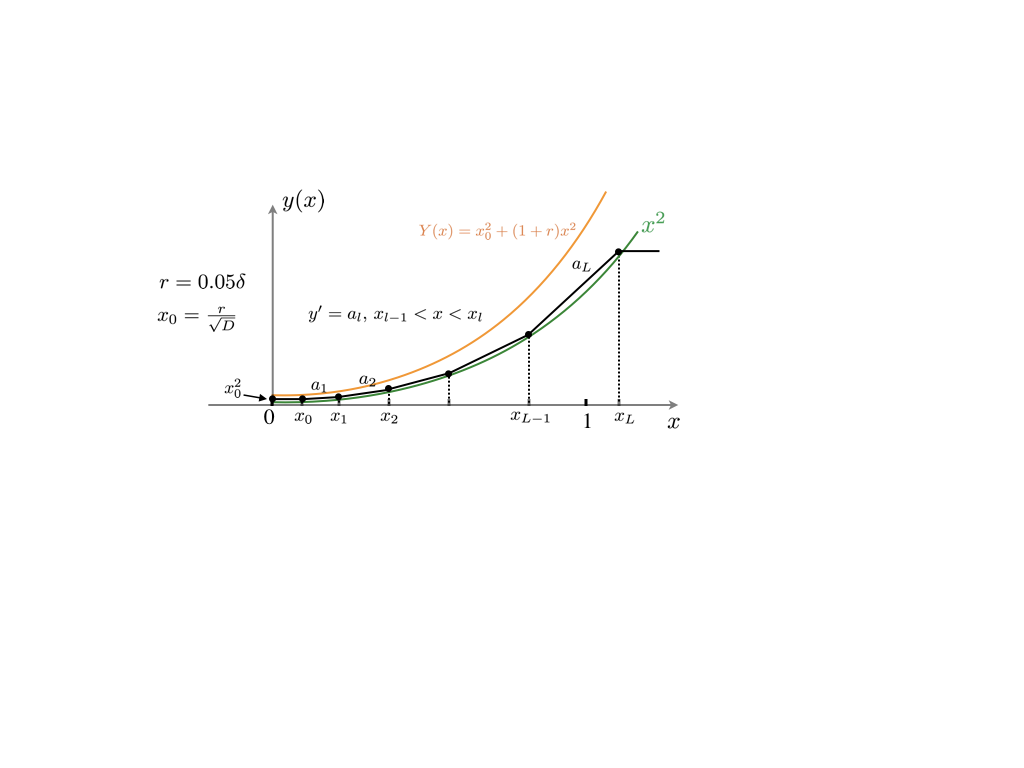}
}
\vspace{-25pt}
\caption{
The construction of $g_\delta: \R^{D} \to \R$.
(Left) $t_\delta$ as in \eqref{eq:def-g}.
(Right) $y(x)$ by $2(L+1)$ Relu's as in \eqref{eq:def-y-1} \eqref{eq:def-y-2},
only the $x > 0$ part is shown and $y(x)=y(-x)$.
}
\label{fig:diag-g}
\vspace{-10pt}
\end{figure}

\begin{proof}[Proof of Lemma \ref{lemma:construct-g-delta}]
Given $ 0 < \delta \le 1 $ fixed, we prove the construction
of $g=g_\delta :\R^m \to \R $ for any dimension $m$, and take $m=D$.
Let $g$ be in the form of
\begin{equation}\label{eq:def-g}
g(v) := t_\delta \left ( \sum_{j=1}^m y(v_j) \right),
\quad
t_\delta( x) := 1 -\text{Relu} \left( \frac{ x- 0.8 \delta^2}{0.2 \delta^2} \right) + 
			\text{Relu}\left( \frac{x-\delta^2}{0.2\delta^2} \right), \quad x \in \R,  
\end{equation}
and $y: \R \to \R$ will be an approximation of $y(x) \approx x^2$ for $|x| \le 1$.
$t_\delta( x )=1 $ when $ x < 0.8 \delta^2$, 0 when $ x > \delta^2$, and linearly interpolating in between, see Fig. \ref{fig:diag-g}.
To construct $y$, define $r: =  \frac{1}{20} \delta $, $ 0< r \le \frac{1}{20} $, 
define a sequence of points
 \begin{align}
& ~~~ x_0 := \frac{r}{\sqrt{m}}, \quad x_l := \rho x_{l-1}, \quad l=1,\cdots, L,  \\
& \rho=1+2r  > 1,
\quad \text{ $L$ is the smallest integer s.t. $x_L = x_0 \rho^L > 1$.}
\end{align}
and let $y(x)$ be a piece-linear function, $y(-x) = y(x)$, and $y(x) = x_0^2$ when $|x| \le x_0$,
$y(x_l) = x_l^2$, $l=1,\cdots, L$, and $y(x) = x_L^2$ for $x > x_L$, see Fig. \ref{fig:diag-g}.
Such $y(x)$ can be represented by $2(L+1)$ many Relu functions, specifically, 
\begin{align}
y(x) & := y_+(x) + y_+(-x) - x_0^2, 
\quad a_l := x_{l-1}+x_{l}, \, l=1, \cdots, L, \, a_{l+1 } > a_l,
\label{eq:def-y-1}
\\
 y_+(x) & = x_0^2 + a_1 \text{Relu}(x-x_0) + (a_2-a_1) \text{Relu}(x-x_1) + \cdots 
 \nonumber \\
 	& ~~~ + (a_L - a_{L-1}) \text{Relu}(x-x_{L-1}) - a_L \text{Relu}(x-x_{L}).
\label{eq:def-y-2}
\end{align}
Since $y(x_l) = x_l^2$ for all $0 \le l \le L$, and by convexity of $x^2$, we have that 

(p1) $y(x) \ge x^2$ whenever $|x| \le x_L$, and $y(x) = x_L^2 > 1$ when $|x| > x_L$.
   
We also claim that

(p2) $0 \le y(x) \le x_0^2 + (1+r)x^2 := Y(x)$, for all $x$.

(p3)  $y$ is piece-wise linear on $\R$, 
when $x \neq x_l$, $y'(x) $ exists, and $y'=0$ if $|x| < x_0 $ or $|x | > x_L$,
$|y'(x)| <  2.1 |x|$ if $ x_0 < |x| < x_L$.

To verify (p2) and (p3): By symmetry of $y$, only consider when $x \ge 0$. When $x \le x_0$, $y(x) = x_0^2 \le Y(x)$.
For $x \in [x_{l-1}, x_l]$, on the left end point $y(x_{l-1}) = x_{l-1}^2 \le (1+r)x_{l-1}^2 < Y(x_{l-1})$,
and $y'(x) = a_l$ on the interval, also $x_l = (1+2r)x_{l-1}$, then
\[
(Y-y)'(x) = 2(1+r)x - a_l = 2(1+r)x - (x_{l-1} + x_l) = 2(1+r)(x - x_{l-1}) \ge 0, \quad x_{l-1} \le x \le x_l,
\]
Thus $Y \ge y$ on $[x_{l-1}, x_l]$. When $|x| > x_L$, $Y(x) \ge (1+r)x^2 \ge x_L^2 = y(x)$. Thus (p2) holds.

The differentiability of $y$ is by construction, and when $ x_{l-1}< x < x_l $, $l=1,\cdots,L$,
$y'(x) = a_l = x_{l-1}+x_l = x_{l-1}2(1+r) < x 2(1+r) \le 2.1 x$, by that $r \le 0.05$. Thus (p3) holds.

We are ready to prove the properties (1)-(3) of $g$ defined as in \eqref{eq:def-g}:

(1) $y:\R \to \R$ is continuous on $\R$, and so is $t_\delta$, then $g$ is continuous on $\R^m$.
$t_\delta$ takes value between 0 and 1, and so is $g$.
For any $\| v \|_{\R^m} \ge \delta$, $\sum_{j=1}^m v_j^2 \ge \delta^2$,
by (p1) we have that $\sum_{j=1}^m y(v_j) \ge \delta^2$.
(If all $|v_j|\le x_L$, $y(v_j) \ge v_j^2$, then the sum $\ge \sum_{j} v_j^2 \ge \delta^2$.
If any one $|v_j| > x_L$,  $y(v_j) > 1 \ge \delta^2$, then so is the sum.)
Then by that $t_\delta(x) = 0$ if $x \ge \delta^2$, $g(v) = 0$.
This proves that $g$ vanishes outside $B_\delta^{m}$.

For any $\| v \|_{\R^m} \le \frac{\sqrt{3}}{2} \delta$, $\sum_{j=1}^m v_j^2 \le 0.75 \delta^2$,
then (p2) gives that 
\begin{align}
& \sum_{j=1}^m y(v_j) \le \sum_{j} ( x_0^2 +(1+r) v_j^2 ) = x_0^2 m + (1+r) \|v\|^2 
= r^2 + (1+r) \|v\|^2 \nonumber \\
& ~~~
\le  (0.05 \delta)^2 + 1.05 \cdot 0.75 \delta^2  < 0.8 \delta^2,
\quad \text{(by that $r = 0.05 \delta \le 0.05$)}
\end{align}
then by that $t_\delta(x)= 1$ when $x < 0.8 \delta^2$, $g(v) = 1$.

(2) $y(v_j)$ applies to each coordinate $v_j$ and $y$ is piece-wise linear on $\R$,
and $t_\delta$ is also piece-wise linear  on $\R$, thus $g$ is piece-wise linear on $\R^m$ and thus piece-wise differentiable. 
By chain rule,
$\partial_{v_j} g(v) =  t_\delta'( \sum_{j} y(v_j)) y'(v_j)$
and $t_\delta'$ is non-zero only if $\sum_{j} y(v_j) \in (0.8 \delta^2, \delta^2)$,
in which case $|t_\delta'( \sum_{j} y(v_j))| = \frac{1}{0.2\delta^2}$ and 
$|y'(v_j)| \le 2.1 |v_j|$ by (p3). 
(Because $y(v_j) \ge 0$ and $\sum_j y(v_j) \le \delta^2$, each $y(v_j) \le \delta^2 \le 1$, 
thus $|v_j| < x_L$, and $|y'(x)| \le 2.1 |x|$ as long as $|x| < x_L$).
Also note that as proved in (1),
one needs  $\| v\| < \delta $ to make $\sum_{j} y(v_j) < \delta^2$.
This means that $\nabla g(v) $, when existing,
 is non-zero only when $\|v\| < \delta$ and $\sum_{j} y(v_j) \in (0.8 \delta^2, \delta^2)$, and then
\[
|\partial_{v_j} g(v) | \le \frac{1}{0.2\delta^2} 2.1 |v_j|, \quad j=1, \cdots, m,
\]
which gives that 
\[
\| \nabla g(v) \|^2 \le (\frac{10.5}{\delta^2} )^2 \|v\|^2 \le (\frac{10.5}{\delta^2} )^2 \delta^2= \frac{(10.5)^2}{\delta^2}.
\]
This proves (2) with $C_1 = 10.5$.

(3) Representing $g$ as a two-layer neural network, 
the top layer $t_\delta$ has 2 neurons and each takes $m$ inputs, thus it has $O( m)$ weights, 
here we use big $O$ to denote multiplying an absolute constant and same below.
The bottom layer branches for the $m$ coordinates,
each branch is a sub-network $y(v_j)$ with one hidden layer of width $2(L+1)$ and a scalar input $v_j$,
and thus it has $O(L)$ weights, and all the $m$ branches has $O( mL)$ weights.
Thus the total number of parameters is $O ( m(1+L))$.
By definition,  $x_0 (1 + 2r)^{L-1} \le 1$, thus, by that $\log(1+2r) > r$ ($ 0 < r \le 0.05$),
\[
L \le 1 + \frac{ \log \sqrt{m} + \log \frac{1} {r}}{r} 
= 1 + \frac{ \frac{1}{2}\log m + \log \frac{20}{\delta}}{0.05 \delta},
\]
this proves that total number of network parameters $\le C m (2 + \frac{ \frac{1}{2}\log m + \log \frac{20}{\delta}}{0.05 \delta})$ for an absolute constant $C$, which is (3) with $C_2$ being an absolute constant.
\end{proof}

\subsection{Other proofs in Section \ref{sec:theory-approx}}

\begin{proof}[Proof of Theorem \ref{thm:ambient-intergral-approx}]

Since $f$ is Lipschitz on $\R^D$, $T$ is Lip-1 on $\R$,
the composed function $ T\circ f$ 
is Lipschitz on $\R^D$, 
and 
\[
\text{Lip}_{\R^D}( T\circ f ) \le  \text{Lip}_{\R^D}(f). 
\]
By that $I[f] = \int_{\R^D} p \cdot T\circ f$,
applying Proposition \ref{prop:integral-swap} gives that
\[
I[f] = \int_{\cal M} T\circ f (x) \tilde{p}(x) d_{\cal M}(x)  +r_1, 
\]
and
\begin{equation}\label{eq:bound-r1}
|r_1|
 \le   \left( \| T\circ f\|_{L^\infty({\cal M})} 
C_1
+ \text{Lip}_{\R^D}(f) 
C_2
 \right) 
c_1 \sigma,
\end{equation}
where $C_1:=3KL_{\cal M} $,
$C_2: = K(2 L_{\cal M} + 1 + \beta_{\cal M} )$
as in Proposition \ref{prop:integral-swap}.

Let $f_{\text{con}}$ be given by 
Theorem \ref{thm:mainthm-shaham2018provable}
to uniformly approximate $f$ on $\cal M$ up to $\varepsilon$,
$\text{Lip}_{\R^D}(f_{\text{con}}) \le L_{ {\cal M}, f} $.
Repeat the above argument on $f_{\text{con}}$ in place of $f$, we have that 
\[
I[f_{\text{con}}] = \int_{\cal M} T \circ f_{\text{con}} (x) \tilde{p}(x) d_{\cal M}(x)  +r_2, 
\]
where since $\|f_{con} -f \|_{L^\infty({\cal M})} \le \varepsilon $ and Lip($T$)$\le 1$,
$\| T\circ f_{\text{con}} \|_{L^\infty({\cal M})}  \le \| T\circ f \|_{L^\infty({\cal M})} + \varepsilon$,
then
\begin{equation}\label{eq:bound-r2}
|r_2|
 \le   \left(  (\| T\circ f \|_{L^\infty({\cal M})} + \varepsilon)
C_1
+ L_{ {\cal M}, f} 
C_2
 \right) 
c_1 \sigma.
\end{equation}

Comparing the integrals on the manifold,
\begin{align}
& \left|
\int_{\cal M} T\circ f (x) \tilde{p}(x) d_{\cal M}(x)  - \int_{\cal M} T\circ f_{\text{con}}  (x) \tilde{p}(x) d_{\cal M}(x) 
\right| 
\nonumber \\
& \le
\int_{\cal M} |T\circ f (x) - T\circ f_{\text{con}} (x)| \tilde{p}(x) d_{\cal M}(x)
\nonumber \\
& \le \int_{\cal M} | f (x) -  f_{\text{con}} (x)| \tilde{p}(x) d_{\cal M}(x)
\quad \text{(by that $T$ is Lip-1)} 
\nonumber \\
& \le \varepsilon \int_{\cal M}  \tilde{p}(x) d_{\cal M}(x)
\quad \text{(by uniform approximation on ${\cal M}$)} 
\nonumber \\
& \le \varepsilon (1+ 3KL_{\cal M}  c_1 \sigma).
\quad \text{(by \eqref{eq:tildep-almost-density})}
\label{eq:bound-r3}
\end{align}

Collecting \eqref{eq:bound-r1}, \eqref{eq:bound-r2}, \eqref{eq:bound-r3},
$|I[f] - I[f_{\text{con}}]|$ is bounded by the sum of the three,
which is
\[
 (1+ C_1 c_1 \sigma)   \varepsilon
 + \left\{
 C_1 ( \varepsilon + 2 \| T\circ f \|_{L^\infty({\cal M})} )
 +
 C_2 (  L_{ {\cal M}, f}  + \text{Lip}_{\R^D}( f )  )
 \right\}
 c_1\sigma,
\]
as stated in the Theorem.
\end{proof}

\begin{proof}[Proof of Lemma \ref{lemma:lip-etai}]
For a fixed $i$, 
by the definition \eqref{eq:def-tileta},
and that 
(i) $\eta_i$ vanishes outside $U_i$, and $\phi(N_i) = \phi(U_i)$,
and 
(ii) $h_i$ vanishes outside $B_\delta^{D-d}$,
we have that supp$(\tilde{\eta}_i) \subset N_i$. 
To see that $\tilde{\eta}_i|_{U_i} = {\eta}_i$,
it suffices to show that $h_i( x-\phi_i(x))=1$ on $U_i$,
which follows by the definition of $h_i$
and Lemma \ref{lemma:manifold-in-tube}.

Finally, we prove the Lipschitz continuity of $\tilde{\eta}_i$ on $\R^D$.
First, $\tilde{\eta}_i$ is continuous on $\R^D$.
This is because 
$\tilde{\eta}_i$ has the factorized definition as in \eqref{eq:def-tileta},
and that $\eta_i( \psi_i (u))$ as a function on $\R^d$ is smooth,
plus that $h_i(v)$ as a function on $\R^{D-d}$ is continuous,
thus the product function $\tilde{\eta}_i$  is continuous on $\R^D$. 
Next we prove the Lipschitz constant.
By the global continuity and that supp$(\tilde{\eta}_i) \subset N_i$, 
$
\text{Lip}_{\R^D} (\tilde{\eta}_i) =  \text{Lip}_{N_i}(\tilde{\eta}_i)$.
For the latter,
consider $x, x' \in N_i$, and  let $y := \psi_i \circ \phi_i(x) \in U_i$,
\[
\tilde{\eta}_i(x) = \eta_i(y) h_i( x - \phi_i(x) ), 
\]
similarly for $x'$, $y'$. 
Since $\eta_i$ is smooth on ${\cal M}$ and compactly supported on $U_i$, 
we assume that 
\[
| \eta_i(y_1)-\eta_i(y_2) |\le c_i d_{\cal M}(y_1, y_2), \quad \forall y_1, y_2 \in U_i,
\]
i.e., $c_i=\text{Lip}(\eta_i)$ w.r.t. manifold geometry, and $c_i$ is determined once the manifold-atlas is fixed.
Then
\begin{align*}
& |\tilde{\eta}_i(x) - \tilde{\eta}_i(x') | 
 \le  |{\eta}_i(y) -\eta_i( y')| |h_i(x - \phi_i(x))| + |\eta_i(y')| | h_i( x-\phi_i(x))- h_i(x'-\phi_i(x')) |  \\
& \le c_i d_{\cal M}(y, y')  +  | h_i( x-\phi_i(x))- h_i(x'-\phi_i(x')) |  
\quad \text{(by  that $\eta_i$ and $h_i$ are bounded by 1 )}   
\\
& 
\le  c_i \beta_{i} \| \phi_i(y) - \phi_i(y') \|_2  +  \frac{2\beta_i^2}{\delta} \|  (x-\phi_i(x)) - (x'-\phi_i(x')) \|
\quad \text{(by \eqref{eq:equiv-norm-Ui} and \eqref{eq:hx=1})}   
\\
&
\le c_i \beta_{i} \| x - x' \|  +  \frac{2\beta_i^2}{\delta}  \| x-x'\|,
\end{align*}
where the last row is by that $\phi_i(y) = \phi_i(x)$, similarly for $y'$, and $\phi_i$ is orthogonal projection.
This proves that 
$\text{Lip}( \tilde{\eta}_i ) \le c_i \beta_{i} +  \frac{2\beta_i^2}{\delta}$
which can be upper bounded by $2c_i + \frac{8}{\delta}$ by \eqref{eq:global-alpha-beta}.
Taking maximum over $i$ gives that $\sup_i \text{Lip}( \tilde{\eta}_i ) \le L_{ {\cal M}}$, which is an absolute content determined by the manifold and atlas. 
\end{proof}

\begin{proof}[Proof of Proposition \ref{prop:integral-swap}]
We need two technical lemmas, the proofs are elementary and in Appendix \ref{app:proofs}.

\begin{lemma}\label{lemma-c3}
For any $p \in {\cal P}_\sigma$ defined as in \eqref{eq:density-mannifold-exp-decay}, if $ \sigma < \frac{1}{2}$, then
\[
\int_{\R^D} d(x, {\cal M}) p(x) dx , 
\int_{\R^D} d(x, {\cal M})^2 p(x) dx 
< c_1 \sigma.
\]
The requirement of $\sigma < 1/2$ is only used to simply the bound to be $c_1\sigma$, and the integral of $d(x,{\cal M})^2$ actually gives a abound of $2c_1 \sigma^2$.
\end{lemma}

\begin{lemma}\label{lemma:boundgx}
If $g :\R^D \to \R$ is globally Lipschitz continuous,
 then \[
| g(x) | 
\le 
\| g\|_{L^\infty({\cal M})} +  \text{Lip}(g) \cdot d(x, {\cal M}),
\quad \forall x \in \R^D,
 \]
 where $\| g\|_{L^\infty({\cal M})}$ is finite due to that $g$ is continuous and ${\cal M}$ is compact.
\end{lemma}

For each $i=1,\cdots, K$,
let $H_i : = \phi_i( U_i) = \phi_i( N_i)$, $H_i \subset T_{x_i}({\cal M})$.
We will show that
\begin{align}
& \int_{\R^D}p(x)g(x)dx 
 \approx \int_{\R^D }p(x)g(x)\sum_{i=1}^{K}\tilde{\eta}_{i}(x)
\quad\text{(error 1)} 
\nonumber \\
 & =\sum_{i=1}^K \int_{N_{i}} g(x)\tilde{\eta}_{i}(x)p(x)dx
 \quad \text{(supp($\tilde{\eta}_i$) $\subset N_i$, Lemma \ref{lemma:lip-etai})} 
 \nonumber\\
 & \approx \sum_{i}\int_{N_{i}} g(\psi_{i}\circ\phi_{i}(x)) \tilde{\eta}_{i}(x)p(x)dx
 \quad\text{(error 2)}, \label{eq:int-pq-1}
 \end{align}
 and then,
 on each $N_i = H_i \times B_\delta^{D-d}$,
  we change to local coordinate $x=(u,v)$, $u = \phi_i(x) \in H_i$, $v = x-\phi_i(x) \in B_\delta^{D-d}$,
 and use \eqref{eq:def-tileta} namely
 \[
 \tilde{\eta}_{i}(x) = \eta_i( \psi_i(u)) h_i(v)
 \] 
 to integrate w.r.t $v$ first,
 which continues \eqref{eq:int-pq-1} as
 \begin{align*}
 & =\sum_{i}\int_{H_{i}}(g\cdot\eta_{i})(\psi_{i}(u)) \left( \int_{B_\delta^{D-d}} h_i(v) p(u,v)dv \right) du
 \\
 & =: \sum_{i}\int_{U_{i}}g(z)\eta_{i}(z)\tilde{p}_{i}(z)d_{\cal M}(z)  
 \quad \text{(definition of $\tilde{p}_i$)}\\
 & =\int_{\cal M}g(z)\left(\sum_{i}\eta_{i}(z)\tilde{p}_{i}(z)\right)d_{\cal M}(z)   
 =\int_{\cal M} g(z)\tilde{p}(z) d_{\cal M}(z),
\end{align*}
where $d_{\cal M}(z)$ stands for the Reimannian volume measure on ${\cal M}$.
To prove the proposition,
it suffices to bound (error 1) and (error 2) and show that the sum $\le$ the right hand side of \eqref{eq:claim-integral-swap}.

Bound of (error 1):
\begin{align}
& \left| \int_{\R^D }p(x)g(x)dx  -  \int_{\R^D }p(x)g(x)\sum_{i=1}^{K}\tilde{\eta}_{i}(x)  \right| 
\nonumber \\
 &   \le 
  \int_{\R^D } p(x) \left| g(x) ( 1- \sum_{i=1}^{K}\tilde{\eta}_{i}(x) )  \right| dx
  =:   \int_{\R^D } p(x) |g(x) \xi(x)| dx, 
  \label{eq:bound-erro1-pf}
\end{align}
where  $\xi:=(  1- \sum_{i=1}^{K}\tilde{\eta}_{i} )$,
$\xi|_{\cal M} = 0$, 
and $\xi$ is Lipschitz on $\R^D$ with
\[
\text{Lip} (\xi ) \le  \sum_{i=1}^K \text{Lip}( \tilde{\eta}_{i}) \le K L_{{\cal M}}
\]
by Lemma \ref{lemma:lip-etai}.
By Lemma \ref{lemma:boundgx}, $\forall x \in \R^D$,
\begin{align*}
|g(x) \xi(x)| 
& \le 
( \| g\|_{L^\infty({\cal M})} +  \text{Lip}(g) \cdot d(x, {\cal M}) ) 
(0 + \text{Lip} (\xi )  \cdot d(x, {\cal M})) \\
& =  \text{Lip} (\xi ) ( \| g\|_{L^\infty({\cal M})} \cdot d(x, {\cal M}) + \text{Lip}(g) \cdot d(x, {\cal M})^2),
\end{align*}
and then \eqref{eq:bound-erro1-pf} continues as
\begin{align*}
&\le    \int_{\R^D } p(x) 
 \text{Lip} (\xi ) \left( \| g\|_{L^\infty({\cal M})} \cdot d(x, {\cal M}) + \text{Lip}(g) \cdot d(x, {\cal M})^2 \right) dx \\
& =  \text{Lip} (\xi ) \left( 
\| g\|_{L^\infty({\cal M})}   \int_{\R^D } p(x)  d(x, {\cal M}) 
+ \text{Lip}(g)  \int_{\R^D } p(x)  d(x, {\cal M})^2  \right) \\
& <  \text{Lip} (\xi )
\left( \| g\|_{L^\infty({\cal M})}   c_1 \sigma + \text{Lip}(g)   c_1 \sigma \right).
\quad \text{(by Lemma \ref{lemma-c3})}
\end{align*}
This proves that 
\begin{equation}\label{eq:bound-error1}
\text{(error 1)} <  K L_{{\cal M}} ( \| g\|_{L^\infty({\cal M})}   + \text{Lip}(g)  ) c_1 \sigma.
\end{equation}

Bound of (error 2):
\begin{equation}\label{eq:error2-eq1}
\left|
\sum_{i=1}^K \int_{N_{i}} ( g(x) - g(\psi_{i}\circ\phi_{i}(x) )) 
\tilde{\eta}_{i}(x) p(x)dx 
\right|
\le
\sum_{i=1}^K \int_{N_i} |   g(x) - g(\psi_{i}\circ\phi_{i}(x) ) | \tilde{\eta}_{i}(x) p(x) dx.
\end{equation}
Define $\xi_i(x) :=  g(x) - g(\psi_{i}\circ\phi_{i}(x) )$ for $x \in N_i$,
and we derive bound for $ |\xi_i(x)|  \tilde{\eta}_{i}(x)$ on $N_i$.
For any $x \in N_i$, 
exists $x^*_{\cal M} \in {\cal M}$ such that $\| x - x^*_{\cal M} \| = d(x, {\cal M})$,
but the point $x^*_{\cal M}$ may not lie in $U_i$.
Consider the two cases respectively:

(i) If $x^*_{\cal M} \in U_i$,  then $\xi_i (x^*_{\cal M}) = 0$, 
and then
\begin{equation}
| \xi_i (x)|
= | \xi_i (x)- \xi_i (x^*_{\cal M})|
\le \text{Lip}_{N_i}(\xi_i) \| x - x^*_{\cal M}\|
= \text{Lip}_{N_i}(\xi_i) d(x, {\cal M}).
\label{eq:bound-etaix-case1}
\end{equation}
For each $i$,
one can verify that $g(\psi_{i}\circ\phi_{i}(x) )$ has Lipschitz constant $\text{Lip}(g) \beta_{i}$ on $N_i$:
For any $x, x' \in N_i$, let $y = \psi_{i}\circ\phi_{i}(x)$, $y'= \psi_{i}\circ\phi_{i}(x')$, $y ,y' \in U_i$,
then
\begin{align*}
& | g(\psi_{i}\circ\phi_{i}(x) ) - g(\psi_{i}\circ\phi_{i}(x') )|
= |g(y) - g(y')| \le \text{Lip}(g) \|  y - y'\|   \\
& \le \text{Lip}(g) d_{\cal M}(y,y') 
\quad \text{(geodesic larger than Euclidean)}
\\
&\le \text{Lip}(g)  \beta_i \| \phi_i(y) - \phi_i(y')\|
\quad \text{(by \eqref{eq:equiv-norm-Ui})} \\
& = \text{Lip}(g)  \beta_i \| \phi_i(x) - \phi_i(x')\|
\le \text{Lip}(g)  \beta_i \| x - x'\|.
\end{align*}
As a result,
we have that
\begin{equation}\label{eq:bound-Lip-xii}
\text{Lip}_{N_i}(\xi_i) \le \text{Lip}(g) (1+\beta_{i}).
\end{equation}
Back to \eqref{eq:bound-etaix-case1}, by that $|\tilde{\eta}_i(x)| \le 1$, and \eqref{eq:global-alpha-beta},
we then have that
 \begin{equation}\label{eq:bound-case1}
 |\xi_i \cdot \tilde{\eta}_i (x) |
  \le 
  \text{Lip}(g) (1+\beta_{\cal M}) \cdot d(x, {\cal M}).
 \end{equation}

(ii) If $x^*_{\cal M} \notin U_i$, 
then $x^*_{\cal M}$ must be outside $N_i$. (Otherwise, suppose $x^*: = x^*_{\cal M}$ is in $N_i$,
$ \| x^* - \phi_i(x^* )\| \le \delta$ and $\phi_i(x^* ) \in \phi_i(U_i)$.
By construction in the beginning of Section \ref{subsec:setup-manifold},
$B_{2\delta}^D (x_i) \cap {\cal M}$ is isomorphic to Euclidean ball,
thus there is one-to-one correspondence between points in $B_{2\delta}^D (x_i) \cap {\cal M}$ and their projected image under $\phi_i$.
Since $\psi_i( \phi_i(x^* ) ) \in U_i $, $x^*$ cannot be $\psi_i( \phi_i(x^* ) )$, 
thus $x_i$ cannot $ \in B_{2\delta}^D (x_i)$.
This draws a contradiction, because $\| \phi_i(x^* ) - x_i \| < \delta$, then $ \| x^* - x_i\| < 2\delta$ by triangle inequality.)
Then the line from $x$ to $x^*_{\cal M}$ intersects with the boundary of $N_i$ at a point $x'$,
and
\[
\| x - x' \| \le \| x - x^*_{\cal M}\| = d(x, {\cal M}).
\]
Then by that $\tilde{\eta}_i(x') = 0$,
\[
|\tilde{\eta}_i(x)| = |\tilde{\eta}_i(x) - \tilde{\eta}_i(x')|
\le \text{Lip}( \tilde{\eta}_i) \| x- x'\|
\le \text{Lip}( \tilde{\eta}_i) d(x, {\cal M}).
\]
Meanwhile, Lemma \ref{lemma:boundgx} gives that 
\[
| \xi_i(x)| 
\le |g(x)|  + | g(\psi_{i}\circ\phi_{i}(x)) |
\le
2 \|g\|_{L^\infty({\cal M})} + \text{Lip}(g) d(x, {\cal M}).
\]
Together, and by the bound of $ \text{Lip}( \tilde{\eta}_i)$ in Lemma \ref{lemma:lip-etai},
\begin{equation}\label{eq:bound-case2}
| \xi_i(x)  \tilde{\eta}_i(x)| 
\le 
\left( 
2 \|g\|_{L^\infty({\cal M})} + \text{Lip}(g) d(x, {\cal M})
\right)
 L_{\cal M} d(x, {\cal M}).
\end{equation}

Combining the two cases,
we simply bound $ |\xi_i(x)|  \tilde{\eta}_{i}(x)$ on $N_i$ by the sum of 
\eqref{eq:bound-case1} and \eqref{eq:bound-case2},
and then \eqref{eq:error2-eq1} continues as
\begin{align}
 \text{(error 2)}  
& \le 
\sum_{i=1}^K
\int_{N_i} 
\left\{
\text{Lip}(g) (1+\beta_{\cal M}) \cdot d(x, {\cal M}) + 
\left( 
2 \|g\|_{L^\infty({\cal M})} + \text{Lip}(g) d(x, {\cal M})
\right)
 L_{\cal M} d(x, {\cal M}) 
 \right\}
  p(x) dx 
  \nonumber \\
&  < K c_1 \sigma  \left\{
 \text{Lip}(g) (1+\beta_{\cal M}) 
 +   L_{\cal M} (2 \|g\|_{L^\infty({\cal M})} + \text{Lip}(g) )
 \right\}
 \label{eq:bound-error2}
\end{align}
 where Lemma \ref{lemma-c3} is used.
 
 Finally, 
combining the two bounds of (error 1) and (error 2) 
\eqref{eq:bound-error1} and \eqref{eq:bound-error2}
proves the claim. 
\end{proof}

\subsection{Proofs in Section \ref{sec:theory-consist}}

\subsubsection{Approximation error analysis of $L[f]$}

\begin{proof}[Proof of Proposition \ref{prop:step2-general-pq}]
We first consider when $\Omega:=$supp$(p+q)$ is compact. 
Then supp($f_{tar}$) is inside $\Omega$ and let $C'$ be the diameter of $\Omega$.
Result in \cite{yarotsky2017error} guarantees the existence of $f_{\text{con}}$ such  that
\begin{equation}\label{eq:yarotsky2017}
\| f_{\text{tar}} - f_{\text{con}}\|_{L^\infty(\Omega)} \le \varepsilon
\end{equation}
with the needed neural network complexity stated in the proposition. 
Then, by   \eqref{eq:def-Lf-2},
 \begin{align*}
 |L[ f_{\text{tar}}] -  L[f_{\text{con}}]|
& \le 
\frac{1}{2}
\left( 
 \int p |T_p \circ f_{\text{tar}} - T_p \circ f_{\text{con}}|
+
 \int q |T_q \circ f_{\text{tar}} - T_q \circ f_{\text{con}}|
 \right) 
 \nonumber \\
& \le
\frac{1}{2}
\left( 
 \int p |f_{\text{tar}} - f_{\text{con}}|
 + 
  \int q |f_{\text{tar}} - f_{\text{con}}|
 \right)
 \quad \text{(by \eqref{eq:LipT})}
 \nonumber \\
 & 
 \le \varepsilon \frac{1}{2} ( \int_\Omega p + \int_\Omega q)
  = \varepsilon,
 \quad \text{(by supp$(p+q)$ $\subset \Omega$ and \eqref{eq:yarotsky2017})}
\end{align*}
which proves the claim. 

When the two densities are merely sub-exponential,
by a re-centering of the origin we assume that supp($f_{\text{tar}}$) lies inside $\Omega:=(-\frac{C'}{2}, \frac{C'}{2})^D$,
and all derivatives of $f = f_{\text{tar}}$ vanishes at $\partial \Omega$.
The constructive proof in \cite{yarotsky2017error} utilizes a partition of unity of the box $\Omega$ by evenly divided sub-boxes, 
and approximate $f$ by a Taylor expansion on each sub-box.
As a result,  
the $ f_{\text{con}}$ which fulfills \eqref{eq:yarotsky2017}
 also vanishes outside $\Omega$.
 This is because the only sub-boxes whose support are not in $\Omega$ are those that are centered on the boundary of $\Omega$, and then the coefficients in Taylor expansion vanish, thus those sub-boxes can be removed from the formula of  $f_{\text{con}}$.
Note that since $T(0)=0$ \eqref{eq:LipT}, whenever $f$ vanishes, so does $T\circ f$,
thus supp($T_p \circ f_{\text{tar}}$),  supp($T_p \circ f_{\text{con}}$) $\subset \Omega$.
Thus we have
\begin{align*}
& |\int p T_p \circ f_{\text{tar}} -  \int p T_p \circ f_{\text{con}} |
\le \int_{{\Omega}} p | T_p \circ f_{\text{tar}} -  T_p \circ f_{\text{con}} | 
 \le \int_{ \Omega} p | f_{\text{tar}} -  f_{\text{con}} | 
 \le \varepsilon \int_{ \Omega} p  \le \varepsilon,
\end{align*}
and similarly for the integral w.r.t. $q$. 
Putting together, it gives that  $| L[f_{\text{tar}}] - L[f_{\theta}]  |  \le  \varepsilon$.
\end{proof}

\begin{proof}[Proof of Proposition \ref{prop:step2-nearmanifold-pq}]
Under Assumption \ref{assump:ftar}, $f_{\text{tar}}$ is smooth and compactly supported in $\R^D$ and then globally Lipschitz. 
Since ${\cal M}$ is compact smooth manifold,
$f_{\text{tar}}|_{\cal M}$ is smooth on the manifold. 
Since $T=T_p$ and $T_q$ are Lipschitz-1,
Theorem \ref{thm:ambient-intergral-approx} applies to $f=f_{\text{tar}}$
and guarantees the existence of a $f_{\text{con}}$ in the network function family with the claimed complexity
to bound 
$| L[f_{\text{tar}}] - L[ f_{\text{con}} ]|$.
(Strictly speaking, Theorem \ref{thm:ambient-intergral-approx}  only applies to bound $| I[f_{\text{tar}}] - I[f_{\text{con}}]|$ 
where $I[f]= \int p T_p \circ f$, or $\int q T_q \circ f$, 
and $L[f]$ equals the average of the two. 
Nevertheless, the proof of Theorem \ref{thm:ambient-intergral-approx} uses the integral comparison lemma Proposition \ref{prop:integral-swap}
and the uniform approximation of $f_{\text{tar}}$ on the manifold,
which directly extends to prove the same result for $I[f] = L[f]$.)

To prove the proposition, it suffices to bound the quantities 
\[
\| T\circ f_{\text{tar}} \|_{L^\infty({\cal M})}, 
\quad
T=T_p,\,T_q.
\]
Let $f= f_{\text{tar}}$, since $f(x_0) = 0$ for $x_0 \in {\cal M}$, and then $T\circ f(x_0) = 0$, and
\[
\| T(f(x)) \| \le \text{Lip}_{\R^D}(f) \|x - x_0\| \le \text{Lip}_{\R^D}(f) \text{diam}({\cal M}),
\quad \forall x\in {\cal M}.
\]
Thus \[
\| T\circ f_{\text{tar}} \|_{L^\infty({\cal M})}
\le \text{Lip}_{\R^D}(f_{\text{tar}}) \text{diam}({\cal M}),
\quad 
T=T_p,\,T_q.
\]
Combining with  \eqref{eq:bound-thm-ambient-integral} and \eqref{eq:C3-thm-ambient-integral}
leads to \eqref{eq:bound-prop-step2-nearmanifold}.
\end{proof}

\subsubsection{Estimation error analysis of $L_n[f]$}
\begin{proof}[Proof of Lemma \ref{lemma:cover}]
For fixed $ 0 < r < \frac{B}{L} $,
let $X:=\{x_i\}_{i =1}^{N}$ be a $r$-net of $K$ such that $N = {\cal N}(K ,r)$.
Let $t := 4Lr > 0$,
and $F:= \{f_j \}_{j=1}^M$ be a maximal $t$-separated set in ${\cal F}'$, 
meaning that for any $j \neq j'$, $\| f_j - f_{j'}\|_{L^\infty(K)} > t$, and no more member in ${\cal F}'$ can be added to preserve this property.
Such $F$ always exists because it can be generated by adding points from an arbitrary point while preserving the $t$-separation property.
By construction, $F$ is a $t$-net of ${\cal F}'$.
We will show that $M \le$ \eqref{eq:boundNF}. 

Partition the 1D interval $[-B, B]$ into $N_1$ many disjoint sub-intervals, each of length $\le 2Lr$, 
and $N_1$ can be made $< \frac{B}{Lr}+1 \le \frac{2B}{ Lr}$.
For any $j\neq j'$, there must be one $x_i \in X$ such that $f_j(x_i)$ and $f_{j'}(x_i)$ lie in distinct subintervals. 
(Otherwise, $|f_j(x_i) - f_{j'}(x_i)| \le 2Lr$,
and by that both $f_j$ and $f_{j'}$ are $L$-Lipschitz, $| f_j(x) - f_{j'}(x)|< 4Lr$ for any $x \in B_r(x_i)$.
Since $\cup_{i=1}^N B_r(x_i)$ cover $K$, this means that $\| f_j - f_{j'} \|_{L^\infty(K)} \le 4Lr =t$,
contradicting with that $f_j$ and $f_{j'}$ are $t$-separated.)
Then each function $f_j$ corresponds to a string of interval indices 
$(I_1, \cdots, I_N)$, where $I_i \in \{ 1, \cdots, N_1\}$,
and these strings are distinct for the $M$ many $f_j$'s.
There are at most $N_1^N$ distinct strings, which means that $M \le N_1^N$.
\end{proof}

\begin{proof}[Proof of Proposition \ref{prop:conc-supF-Ln}]
Recall that  ( $|X_{tr}| = |Y_{tr}| = n$)
\[
L_n [f] = \frac{1}{2} \left(
\frac{1}{n} \sum_{i=1}^n T_p \circ f(x_i) + \frac{1}{n} \sum_{i=1}^n T_q \circ f(y_i) 
\right),
\quad
L[f] = \E L_n [f],
\]
where $x_i \sim p$ i.i.d., $y_i \sim q$ i.i.d., and $x_i$'s and $y_i$'s are independent, and $T_p$, $T_q$ as in \eqref{eq:LipT}.
We prove the three cases respectively, where for case (1), we prove the compactly supported case first, and the exponential-tail case as (1').
\\

(1) 
For any $f \in {\cal F}_{\Theta, \text{reg}}(B_R) $,
there is $x_0 \in B_R$,
$f(x_0) = 0$.
Then
$\forall x \in B_R(0)$, $|f(x)| \le \text{Lip}(f) \|x-x_0\| \le 2LR$, i.e., 
$\| f \|_{L^\infty( {B_R(0)})} \le 2LR$.
Thus ${\cal F}_{\Theta, \text{reg}}(B_R) $ is contained in the function space of
\[
{\cal F} := \{f, \, \text{Lip}_{\R^D}(f) \le L, \, \|f\|_{L^\infty( B_R(0) )} \le 2LR \}
\quad
\text{equipped with $\| \cdot \|_{L^\infty( B_R(0))}$}.
\]
By Lemma \ref{lemma:cover}, for any $r < 1 < 2R$, $t := 4Lr$, 
there exists a finite set $F$ in ${\cal F}_{\Theta, \text{reg}}(B_R)$ 
 which form an $t$-net that covers ${\cal F}_{\Theta, \text{reg}}(B_R) $ 
under the metric of $\| \cdot \|_{L^\infty( B_R(0))}$,
where 
\[
\text{Card}(F) \le (\frac{4R}{r})^{{\cal N}( B_R(0), r)}.
\]
The covering number of a Euclidean ball can be bounded by 
${\cal N}( B_R(0), r) \le \left( \frac{2R}{r} + 1\right)^D < \left( \frac{3R}{r}\right)^D$,
see e.g. Section 4.2 of \cite{vershynin2018high},
using $r < 1$.
Thus,
\[
\text{Card}(F) \le \exp \left\{   (3R)^D r^{-D}  \log \frac{4R}{r} \right\}.
\]

Given $r < \min\{ c_1'/4, 1\}$, 
for each $f \in F$, 
$\text{Lip}_{\R^D}(f) \le L$,
then Lemma \ref{lemma:sub-exp-xi} and \eqref{eq:Bern-1} apply to give concentration of the independent sums
over $x_i$'s and $y_i$'s respectively.
By a union bound, 
\begin{equation}\label{eq:badeventprob1}
\Pr[ \exists f \in F, \, |L_n[f] - L[f]| \ge t] 
\le \text{Card}(F) \cdot 4 e^{  - c_0' n \frac{t^2}{L^2} },
\end{equation}
where $\frac{t}{L} = 4r$ is chosen to $< c_1'$ to begin with.
The r.h.s of \eqref{eq:badeventprob1} is upper bounded by
\[
\exp \left \{
 \log 4 +  (3R)^D r^{-D}  \log \frac{4R}{r} - c_0' n (4r)^2 
 \right \},
\quad 
0< r < \min\{ c_1'/4, 1\},
\]
which 
\begin{align}
& \le \exp \left \{
- \frac{2}{3} n^{\frac{D}{D+2}} (\log n)^{\frac{2}{2+D}}  + \log 4
\right \},
\label{eq:badeventprob-2}
\\
& \text{
if $\gamma := \max\{ 3R, \frac{1}{4\sqrt{c_0'}} \}$, 
$r = \gamma (\frac{\log n}{n} )^{\frac{1}{2+D}}$,
and $(\log n)^{\frac{1}{2+D}} > \frac{4}{3} $.}
\end{align}
Note that as $n \to \infty$, $r \to 0$ thus the constraint of $r < \min\{ c_1'/4, 1\}$ is satisfied for sufficiently large $n$. 

We consider the good event where $|L_n[f] - L[f]| < t$ for all $f \in F$.
Since $F$ is a $t$-net that covers $ {\cal F}_{\Theta, \text{reg}}(B_R)$ with unions of closed $\| \cdot \|_{L^\infty( B_R)}$-balls around points in $F$,
for any $f \in {\cal F}_{\Theta, \text{reg}}(B_R) $,
there is an $f_0 \in F$ such that $\| f - f_0\|_{L^\infty(B_R)} \le t$.
Since supp($p+q$) $\subset B_R$,
this implies that $|L[f] - L[f_0]| \le t$ and $|L_n[f] - L_n[f_0]| \le  t$,
then 
\[
\sup_{f \in {\cal F}_{\Theta, \text{reg}}(B_R)} |L_n[f] - L[f]| \le 3t.
\]
This proves that with sufficiently large $n$, with probability $\ge 1-$ \eqref{eq:badeventprob-2},
\[
\sup_{f \in {\cal F}_{\Theta, \text{reg}}(B_R)} |L_n[f] - L[f]| 
\le 3 \cdot 4L \max\{ 3R, \frac{1}{4\sqrt{c_0'}} \} \left( \frac{\log n}{n} \right)^{\frac{1}{2+D}}.
\]
\\

(1')  The proof extends that in (1) by a truncation argument for the exponential tail of the densities.
We rename the $R$ in the statement as $R_0$ in below.

The following lemma gives the decay of the integration
 with the tail of sub-exponential densities 
 for any $f \in {\cal F}_{\Theta, \text{reg}}( B_{R_0}(0))$,
 proved in Appendix \ref{app:proofs}.
\begin{lemma}\label{lamma:truncate-Lftheta}
Suppose  $p$ is in ${\cal P}_{\text{exp}}$  with $c=1$,
function $f:\R^D \to \R$
has $\text{Lip}_{\R^D}(f) \le L$ and vanishes at some point $x_0$, $\|x_0\| < R_0$,
then, $C$ as in the definition of ${\cal P}_{\text{exp}}$,
\[
\int_{ \|x\| > R} |f| p  <  L( R_0 + 1+ R ) Ce^{-R}, 
\quad \forall R > R_0.
\]
\end{lemma}

We introduce a sequence of domains $\Omega_n := B_{R_n}(0)$ in $\R^D$, where $R_n = \alpha \log n$, $\alpha > 0$ is a constant to be determined. 
For $n$ samples of $x_i$'s and $y_i$'s,
since $p$, $q$ are in $ {\cal P}_{\text{exp}}$  with $c=1$,
\[
\Pr [ \exists i, \, \|x_i\| > R \text{ or } \exists i', \, \|y_{i'}\| > R] 
< 2n Ce^{-R}, 
\quad \forall R > 0.
\]
We call  \[
\text{(good event 1)}_n = \{ \| y_{i'}\|, \, \| x_i\| \le R_n, \,  \forall i, i'\}.
\]

For any large enough $n$ such that $R_n > \max\{ R_0, 1 \}$, 
the functional space ${\cal F}_{\Theta, \text{reg}}( B_{R_0}(0))$ lies inside 
\[
{\cal F}_n := \{f, \, \text{Lip}_{\R^D}(f) \le L, \, \|f\|_{L^\infty( B_{R_n} )} \le 2LR_n \}
\quad
\text{equipped with $\| \cdot \|_{L^\infty( B_{R_n})}$}.
\]
Similarly as in (1), 
for any $r< 1$ thus $< 2R_n$, $t: = 4Lr$,
${\cal F}_{\Theta, \text{reg}}( B_{R_0}(0))$ has a subset $F_n$ which is an $t$-net
that covers ${\cal F}_{\Theta, \text{reg}}( B_{R_0}(0))$ under $\| \cdot \|_{L^\infty( B_{R_n})}$,
and 
\[
\text{Card}(F_n) \le \exp \left\{   (3R_n)^D r^{-D}  \log \frac{4R_n}{r} \right\}.
\]
We choose
\[
r = \gamma_n (\frac{\log n}{n})^{\frac{1}{2+D}},
\,
\gamma_n = 3R_n \text{ ($\sim \log n > \frac{1}{4\sqrt{c_0'}}$ for large $n$)},
\]
then if $(\log n)^{\frac{1}{2+D}} > 4/3$,
\[
\text{(good event 2)}_n := \{ \forall f\in F_n,\, |L_n[f] - L[f]| < t\}
\]
fails with probability $\le$ \eqref{eq:badeventprob-2}.

Restricting to $\text{(good event 1)}_n$ and $\text{(good event 2)}_n$.
For any $f \in {\cal F}_{\Theta, \text{reg}}( B_{R_0}(0))$,
exists $f_0 \in F_n$ such that 
$\|f - f_0\|_{L^\infty( B_{R_n})} \le t$,
then $|L_n[f] - L_n[f_0]| \le t$.
Meanwhile, since both $f $ and $f_0$ are in ${\cal F}_{\Theta, \text{reg}}( B_{R_0}(0))$, 
Lemma \ref{lamma:truncate-Lftheta} applies to both (since $R_n > R_0$),
then
\begin{align*}
|L[f] - L[f_0]|
& \le \frac{1}{2} \left(  \int p |T_p \circ f -  T_p \circ f_0| +   \int q |T_q \circ f -  T_q \circ f_0| \right)
\\
& \le \frac{1}{2} \left(  \int p | f -  f_0| +   \int q |  f -   f_0| \right) 
\quad \text{(by \eqref{eq:LipT})}
\\
& \le \frac{1}{2} \left( 
t \int_{B_{R_n}} (p+q) 
+ \int_{\R^D \backslash B_{R_n}} (p+q) (|f| + |f_0|) \right) 
\\
& \le t + 2CL (R_0 + 1 + R_n) e^{-R_n},
\end{align*}
which means that 
\begin{align}
& \sup_{f \in {\cal F}_{\Theta, \text{reg}}(B_{R_0})} |L_n[f] - L[f]| 
\le 3 t + 2CL (R_0 + 1 + R_n) e^{-R_n} \\
&~~~
 = 3 \cdot 4L  \cdot 3R_n  \left( \frac{\log n}{n} \right)^{\frac{1}{2+D}} 
 +  2CL (R_0 + 1 + R_n) e^{-R_n} \\
 &~~~
 < \tilde{C} L \cdot \alpha \log n  \cdot (   (\log n/ n)^{\frac{1}{2+D}}  + n^{-\alpha}),
 \quad \text{(setting $R_n= \alpha \log n > R_0$)}
 \label{eq:estimation-bound(2)-1}
\end{align}
where $\tilde{C}$ is an absolute constant.
This happens with probability $\ge 1-  p_{\text{fail}}$,
and 
\begin{align*}
p_{\text{fail}}
& \le
\Pr[\text{ fail of $\text{(good event 1)}_n$}]+ \Pr[\text{ fail of $\text{(good event 2)}_n$ }]
\\
&  \le 2n Ce^{-R_n} + \text{\eqref{eq:badeventprob-2}} 
=  (2C) n^{-\alpha + 1}  + \text{\eqref{eq:badeventprob-2}}.
\end{align*}
To make $p_{\text{fail}} \to 0$,
one can set $\alpha = 1+\epsilon$ for some $\epsilon > 0$,
then in \eqref{eq:estimation-bound(2)-1}
 the $n^{-\alpha}$ term is dominated by the term of $(\log n/ n)^{\frac{1}{2+D}} \gg O( n^{-1/3})$. 

Putting together, we have that when $n$ is sufficiently large,
specifically,
$(\log n)^{\frac{1}{2+D}} > 4/3$ and $R_n = (1+\epsilon) \log n > \max\{ 1, R_0, 1/(12\sqrt{c_0'})\}$,
then with probability
$\ge 1- (2C)n^{-\epsilon} -  \text{\eqref{eq:badeventprob-2}}$,
the bound \eqref{eq:estimation-bound(2)-1} holds, which for large $n$ is 
\[
\sim L \log n  (\log n/ n)^{\frac{1}{2+D}},
\]
omitting the absolute constant in front.
\\

(2)
Since ${\cal M} \subset B_R$, similarly as in (1), enlarge the network function space to be 
\[
{\cal F} := \{f, \, \text{Lip}_{\R^D}(f) \le L, \, \|f\|_{L^\infty( {\cal M} )} \le 2LR \}
\quad
\text{equipped with $\| \cdot \|_{L^\infty( {\cal M})}$}.
\]
When applying Lemma \ref{lemma:cover}, the $t$-net $F$ can be chosen such that
\[
\text{Card}(F)
\le \exp \left\{   c({\cal M}) r^{-d} \log \frac{4R}{r} \right\},
\]
because the covering number of the manifold 
${\cal N}( {\cal M}, r) \le \frac{c({\cal M})}{r^d}$,
where $c({\cal M} )$ involves the intrinsic volume of ${\cal M}$ integrated over its Riemannian volume element, 
and it scales with the diameter of ${\cal M}$.
The rest of the proof is similar, which gives that, with sufficiently large $n$, with probability 
$\ge 1- \exp \left \{
- \frac{2}{3} n^{\frac{d}{d+2}} (\log n)^{\frac{2}{2+d}}  + \log 4
\right \}$,
\[
\sup_{f \in {\cal F}_{\Theta, \text{reg}}( {\cal M})} |L_n[f] - L[f]| 
\le 3 \cdot 4L \max\{ c({\cal M})^{1/d},  \frac{1}{4\sqrt{c_0'}}  \} \left( \frac{\log n}{n} \right)^{\frac{1}{2+d}}.
\]
\\

(3) We consider the same enlarged ${\cal F}$ as in (3) equipped with $\| \cdot \|_{L^\infty( {\cal M})}$,
which can be covered by the same $t$-net $F$ as there.
The large deviation of $|L_n[f] - L[f]|$ for any $f \in F$ is the same,
because though the densities $p$, $q$ are no longer supported on the manifold they still belong to the ${\cal P}_{\text{exp}}$ class,
 and so is the union bound. 
 The difference is in the control of $|L_n [f] - L_n[f_0]|$ and $|L[f] - L[f_0]|$ where $f$ is an arbitrary member in 
 ${\cal F}_{\Theta, \text{reg}}( \cal M)$
 and  $f_0 \in F$ such that $\| f - f_0\|_{L^\infty( {\cal M})} \le t$. 
 
 For $|L[f] - L[f_0]|$, 
 we can use the integral comparison Proposition \ref{prop:integral-swap}
  and the uniform approximation of $f$ by $f_0$ on the manifold.
 Specifically, similarly as in the proof of Proposition \ref{prop:step2-nearmanifold-pq},
 we have that 
 \begin{equation}\label{eq:diff-Lf-Lf0}
 |L[f] - L[f_0]| \le 
 2(2LR C_1({\cal M}) +L C_2({\cal M}) )c_1 \sigma
 + (1 + C_1({\cal M}) c_1 \sigma) t,
 \end{equation}
 where $C_1({\cal M})$, $C_2({\cal M})$ are manifold-atlas-dependent only constants defined in \eqref{eq:def-C1C2}.

For the empirical $L_n$, we need the concentration of the independent sum of the random variables $d(x_i, {\cal M})$ (and similarly $d(y_i, {\cal M})$),
which is the following lemma  proved in Appendix \ref{app:proofs}.
\begin{lemma}\label{lemma:conc-dxM}
Suppose $x_i$, $i=1,\cdots, n$ $\sim p$ i.i.d., $p \in {\cal P}_\sigma$ 
as defined in \eqref{eq:density-mannifold-exp-decay} with  $\sigma < \frac{1}{2}$
and also in $ {\cal P}_{\text{exp}}$  with $c=1$. Then 
\[
\Pr 
\left[ \frac{1}{n}\sum_{i=1}^n d(x_i, {\cal M})  > (c_1 + \tau) \sigma \right]
\le e^{ - c_0'' n \tau^2  },
\quad \forall  0 < \tau < 1, 
\]
where $c_0''$ is an absolute constant.
\end{lemma}
To proceed, observe that 
\begin{align*}
& \left| \frac{1}{n} \sum_{i=1}^n ( T_p \circ f(x_i) - T_p \circ f_0(x_i)  ) \right|
\le 
\frac{1}{n} \sum_{i=1}^n |f(x_i) - f_0(x_i)|
\quad \text{(by \eqref{eq:LipT})} 
\\
&
\le \frac{1}{n} \sum_{i=1}^n  |f(x_i) - f( (x_i)^*_{\cal M}) | 
+ |  f( (x_i)^*_{\cal M} ) - f_0 ( (x_i)^*_{\cal M})| 
+ |f_0 ( (x_i)^*_{\cal M}) - f_0 (x_i)|
\\
& ~~~~~~~~
\text{(for each $x_i$, $\exists (x_i)^*_{\cal M}\in {\cal M}$ and $\| x_i - (x_i)^*_{\cal M}\| = d(x_i, \cal M)$)}
\\
& \le \frac{1}{n} \sum_{i=1}^n  
(2 L d(x_i, M) + t )
\quad \text{(by that $\| f - f_0\|_{L^\infty({\cal M})} \le t$)}
\\
& = t + 2L  \left( \frac{1}{n} \sum_{i=1}^n   d(x_i, {\cal M}) \right).
\end{align*}
Similarly for the independent sums with $T_q \circ f(y_i)$, 
this gives that 
\begin{equation}\label{eq:diff-Lnf-Lnf0}
 |L_n[f] - L_n[f_0]| 
 \le t + L  \left( \frac{1}{n} \sum_{i=1}^n   d(x_i, {\cal M}) + \frac{1}{n} \sum_{i=1}^n   d(y_i, {\cal M}) \right).
\end{equation}
Putting together \eqref{eq:diff-Lf-Lf0}, \eqref{eq:diff-Lnf-Lnf0}, Lemma \ref{lemma:conc-dxM},
and the union bound over the $t$-net such that $\sup_{f \in F}| L_n [f] - L[f] | < t$,
choosing $\tau = c_1$, $t = 4Lr$, $r \sim (\log n /n)^{1/(2+d)}$ as in (3),
we have that, when $n$ is sufficiently large, with probability 
$\ge 1- \exp \left \{ - \frac{2}{3} n^{\frac{d}{d+2}} (\log n)^{\frac{2}{2+d}}  + \log 4 \right \} - \exp \{ - c_0'' c_1^2 n\}$,
\[
\sup_{f \in {\cal F}_{\Theta, \text{reg}}( {\cal M})} |L_n[f] - L[f]| 
\le \tilde{C}({\cal M}) \cdot L ( \sigma +  ( \log n /n )^{\frac{1}{2+d}}),
\]
where $\tilde{C}({\cal M})$ is a constant that depends on manifold-atlas only. 
\end{proof}

\subsubsection{Testing power analysis on the testing set}

\begin{proof}[Proof of Theorem \ref{thm1}]
The text before the theorem proves (1), collecting the approximation error results 
Propositions \ref{prop:step2-general-pq}, \ref{prop:step2-nearmanifold-pq}
and the estimation error result
Proposition \ref{prop:conc-supF-Ln}. 
Note that under the setting of Proposition \ref{prop:conc-supF-Ln}, 
which only imposes the extra assumption that network functions need to vanish at least at one point in a bounded domain
beyond Assumption \ref{assump:netLip},
the approximation results hold. 
Particularly,
in Proposition \ref{prop:step2-nearmanifold-pq}, 
the $\text{Lip}(f_{\text{con}})$ can be bounded by $L_\Theta$ which is a universal constant,
and all other constants depends on manifold-atlas, the target function $f_{\text{tar}}$.

(2) 
Because $T_n$ is the sum of two independent sums of a fixed Lipschitz function $\hat{f}_{\text{tr}}$
averaged on $x_i$'s and $y_i$'s respectively,
CLT applies to give the asymptotically normality, and it suffices to bound the variance to prove the claim.
Let $f=\hat{f}_{\text{tr}}$, under Assumption \ref{assump:netLip}, $\text{Lip}_{\R^D}(f) \le L_\Theta$.
Note that
$\text{Var}(\sqrt{n} T_n) =  \text{Var}_{x \sim p}(f(x)) + \text{Var}_{y \sim q}(f(y))$.
Since in all three settings in Proposition \ref{prop:conc-supF-Ln}, 
$p$ and $q$ are in ${\cal P}_{\text{exp}}$ with $c=1$,
then similar to Lemma \ref{lemma:sub-exp-xi},
we know that $( f(x_i)-\E_{x \sim p} f(x) )$
is 1D sub-exponential random variables satisfying \eqref{eq:sub-exp-xi}
where $L=L_\Theta$.  
This proves that Var($f(x_i)$) $\le c'' L_\Theta^2$, where $c''$ is an absolute constant.
Same argument applies to $f(y_i)$, 
and thus $\text{Var}(\sqrt{n} T_n) \le L_\Theta^2$ multiplied by an absolute constant.
\end{proof}

\begin{lemma}\label{lemma1}
 For any $f$ so that the integrals are defined, $T[f] \ge 4L[f]$.
 \end{lemma}

Remark:
The relaxation in Lemma \ref{lemma1} may not be sharp, 
particularly,
when $p$ and $q$ nearly non-overlap on their supports, 
$T[f^{*}] = 2 \text{SKL}(p,q)$ diverges to
infinity, while $L[f^{*}]$ remains bounded (by $2\log 2$). 
When $f$ is close to zero,
which as discussed above is the more relevant scenario for two-sample test,
the following lemma quantifies the tightness of the relaxation. 
Proofs of both lemmas are in Appendix \eqref{app:proofs}.

\begin{lemma} 
 For any $f$ s.t. $f^{2}$ is integrable w.r.t $p$ and
$q$,
\[
0\le\frac{1}{2}T[f]- 2 L[f]\le\int(p+q)\frac{f^{2}}{2}.
\]
\label{lemma2}
\end{lemma}

\section{Conclusion and discussion}

The paper proposes to use the difference of sample-averaged logit function on testing set as a two-sample statistic, once a classification network has been trained on a training set.
The proposed two-sample method empirically demonstrates improved performance over
 previous neural network two-sample tests using classification accuracy.
It also compares favorably to Gaussian kernel-based MMD in certain settings,
especially for higher dimensional data, 
e.g., distinguishing generated MNIST digits from authentic ones.
The proposed test has an advantage with larger datasets, 
as larger training sets make the optimization more stable,
and the batch-based algorithm has better scalability. 
Theoretically, we prove the consistency of the proposed test when the network is sufficiently parametrized,
and reduce the needed network complexity to scale intrinsically 
when $p$ and $q$ lie on or near to low-dimensional manifolds in ambient space.

The analysis in this paper covers approximation and estimation error,
and more understanding of network optimization is needed so as to better study network classification two-sample test methods. 
In particular, we have not systematically explored the influence of different network architectures.
As for methodology, a future direction is to extend to other training objectives than softmax, 
such as $f$-divergences \cite{kanamori2011f,nowozin2016fgan}, other GAN losses,
and testing power estimators \cite{gretton2012optimal,sutherland2016generative, liu2020learning}.

\section*{Acknowledgement}
The work is supported by NSF DMS-1818945.
AC is also partially supported by Sage Foundation Grant 2196.
XC is also partially supported by NSF (DMS-1820827), NIH (Grant
R01GM131642) and the Alfred P. Sloan Foundation.

\small
\bibliographystyle{plain}
\bibliography{mmd}

\normalsize

\appendix

\setcounter{figure}{0} \renewcommand{\thefigure}{A.\arabic{figure}}
\setcounter{table}{0} \renewcommand{\thetable}{A.\arabic{table}}
\setcounter{equation}{0} \renewcommand{\theequation}{A.\arabic{equation}}
\setcounter{theorem}{0} \renewcommand{\thetheorem}{A.\arabic{theorem}}

\section{Proofs of technical lemmas}\label{app:proofs}

\begin{proof}[Proof of Lemma \ref{lemma:sub-exp-xi}]
$\xi_i = g(x_i)$, where $g: = T\circ f$, $\text{Lip}_{\R^D}(g) \le L$. \[
| g(x_i) - \E_{x \sim p} g(x) |
\le \int_{\R^D} |g(x_i) - g(x)| p(x) dx
\le L \int \| x_i - x\| p(x) dx
\le L( \|x_i\| + \E_{x \sim p}\|x\|)  
\]
and by \eqref{eq:def-calP-subexp-RD}, $\E_{x \sim p}\|x\| = \int_0^\infty \Pr_{x \sim p} [ \|x\| > t] dt < C$ which is an absolute constant.
This means that the random variable $(\xi_i - \E \xi_i )/L $  in absolute value is upper bounded by 
$\| x_i\| + C$, where $y := \| x_i\|$ as a 1D random variable satisfies that $\Pr [ |y| > t ] < C e^{-t}$,
thus $(\xi_i - \E \xi_i )/L $ satisfies the sub-exponential tail claimed with other absolute constants $C'$ and $c'$.
\end{proof}

\begin{proof}[Proof of Lemma \ref{lemma-c3}]
Let $X \sim p$, 
then $d(X, {\cal M})$ is a non-negative random variable,
and  
\begin{align*}
& \int_{\R^D} d(x, {\cal M}) p(x) dx
 = \int_{\R^D} \left( \int_0^{d(x, {\cal M})} dt \right) p(x) dx \\
& =  \int_{0}^\infty \left(     \int_{\R^D} {\mathbf{1}}_{ \{0< t < d(x, {\cal M}) \} } p(x) dx    \right)dt 
= \int_{0}^\infty \Pr[ d(X, {\cal M}) > t] dt \\
& 
\le  \int_{0}^\infty c_1 e^{- \frac{t}{\sigma}}   dt 
 = {c_1} \sigma.
\end{align*}
Similarly,
\begin{align*}
& \int_{\R^D} d(x, {\cal M})^2 p(x) dx
 = \int_{\R^D} \left( \int_0^{d(x, {\cal M})} 2t dt \right) p(x) dx \\
& =  \int_{0}^\infty \left(     \int_{\R^D} {\mathbf{1}}_{ \{0< t < d(x, {\cal M}) \} } p(x) dx    \right) 2t dt 
= \int_{0}^\infty \Pr[ d(X, {\cal M}) > t] 2t dt \\
& 
\le  \int_{0}^\infty c_1 e^{- \frac{t}{\sigma}}  2t dt 
 = 2 {c_1} \sigma^2.
\end{align*}
\end{proof}

\begin{proof}[Proof of Lemma \ref{lemma:boundgx}]
By compactness and smoothness of ${\cal M}$, for any $x \in \R^D$, there exists $x^*_{\cal M} \in {\cal M}$ s.t. $\| x^*_{\cal M}-x \| = d(x, {\cal M})$.
Thus,
\begin{align*}
|g(x)| 
& \le |g(x^*_{\cal M})| +  |g(x) - g(x^*_{\cal M})|  \\
& \le \sup_{x' \in {\cal M}} |g(x')|  + \text{Lip}(\xi) \| x - x^*_{\cal M} \| \\
& =   \| g\|_{L^\infty({\cal M})} + \text{Lip}(\xi) \cdot d(x, {\cal M}).
\end{align*}
\end{proof}

\begin{proof}[Proof of Lemma \ref{lamma:truncate-Lftheta}]
Let Lip$(f)$$ \le L$,  $f(x_0) = 0$, $\|x_0\| < R_0$, then 
\[
|f(x) |\le  |f(x_0)| + L\|x-x_0\| \le L( R_0 + \|x\|), 
\quad \forall x, \, \|x \| > R,
\]
Thus,
\begin{align*}
& \int_{ \|x\| > R} |f(x)| p(x) dx   
 \le
L \int_{ \|x\| > R} ( R_0 + \|x\| )p(x) dx    
\\
& =
L \left(
(R_0 + R) \int_{ \|x\| > R} p + 
 \int_{ \|x\| > R}  (\|x\|-R) p \right),
\end{align*}
and 
\begin{align*}
\int_{ \|x\| > R}  (\|x\|-R) p 
&= \int_{\| x \| > R} \int_R^{\|x \|} dt p(x) dx
= \int_R^{\infty} \int_{\| x \| > t}  p(x) dx dt \\
&= \int_R^{\infty}  \Pr [ \|X\| > t]dt
<  \int_R^\infty Ce^{-t} dt
= Ce^{-R},
\end{align*}
then 
\[
\int_{ \|x\| > R} |f| p   
< L( R_0 + R + 1) Ce^{-R}.
\]
\end{proof}

\begin{proof}[Proof of Lemma \ref{lemma:conc-dxM}]
By definition of ${\cal P}_\sigma$, $d(x_i, {\cal M})$'s are i.i.d. non-negative sub-exponential random variables,
and $\E_{x \sim p} d(x, {\cal M}) < c_1 \sigma$ by Lemma \ref{lemma-c3}.
The claims follows by the  Bernstein's control of the positive tail
for independent sum of i.i.d. sub-exponential random variables.
\end{proof}

\begin{proof}[Proof of Lemma \ref{lemma1}]
By definition,
\begin{align*}
2 L[f] & =\int p\log\frac{2}{1+e^{-f}}+\int q\log\frac{2}{1+e^{f}}\\
 & =-\int p\log\frac{1+e^{-f}}{2}-\int q\log\frac{1+e^{f}}{2}\\
 & \le-\int p\log e^{-\frac{f}{2}}-\int q\log e^{\frac{f}{2}}\quad\text{(for any real number \ensuremath{\xi}, \ensuremath{\frac{1+e^{\xi}}{2}\ge e^{\frac{\xi}{2}}})}\\
 & =\int\frac{f}{2}(p-q)=\frac{1}{2}T[f].
\end{align*}
\end{proof}

\begin{proof}[Proof of Lemma \ref{lemma2}]
\begin{align*}
\frac{1}{2}T[f]- 2L [f] 
& =\int p\frac{f}{2}-\int q\frac{f}{2}-\int p\log\frac{2e^{f}}{1+e^{f}}-\int q\log\frac{2}{1+e^{f}}\\
 & =\int p\log\frac{1+e^{f}}{2e^{f/2}}+\int q\log\frac{1+e^{f}}{2e^{f/2}}\\
 & =\int(p+q)\log\frac{e^{-f/2}+e^{f/2}}{2},
\end{align*}
and by that $\frac{e^{x}+e^{-x}}{2}\le e^{x^{2}/2}$, 
\[
\frac{1}{2}T[f]- 2 L[f] 
\le\int(p+q)\log e^{f^{2}/2}=\int\frac{f^{2}}{2}(p+q).
\]
\end{proof}

\section{A Remark after Proposition \ref{prop:step2-general-pq}}\label{app:remark}

We give a remark on the trade-offs in the choice of $f_{\text{tar}}$ and regularity level in Proposition \ref{prop:step2-general-pq}.

There are two trade-offs in applying the result above.
First, generically (e.g., the two densities are absolutely continuous)
 one can make the constant $C$  arbitrarily close to $\text{JSD}(p,q)$, 
and as a result the regularity of $f_{\text{tar}}$ may worse, and then the needed network complexity will increase. 
Second, the regularity level $r$ can be chosen to be large given that $f_{\text{tar}} \in C_c^\infty(\Omega) $,
but the constant in the network complexity bound will grow with higher order $r$.

For example,
when $q$ is significantly large on a region where $q$ almost vanish,
the log density ratio $f^*$ will be very large in that region,
which creates a difficulty for network approximation if we set $f^*$ to be the target.
However, in this case $f^*$ can be replaced by a regularized surrogate $f_{\text{tar}}$ which will produce an a comparably large $L[f] > C$
with a significantly large $C > 0$.
When such region is large, i.e., $p$ and $q$ have distinct supports,
then the gap between $L[f_{\text{tar}}] $ and JSD$(p,q)$ may be larger if we choose to have smoother $f_{\text{tar}}$. 
But note that such situation is actually trivial case for two-sample testing since $p$ and $q$ already differ significantly,
and as a result $L[f_{\text{tar}}] $ can be made large even not close to  JSD$(p,q)$. 
The more interesting situation is the ``weak-separation" regime of $p$ and $q$,
that is, the departure of $q$ from $p$ is small and then $f^*$ is not far from zero.
We see in Section \ref{subsec:test-consist} that the critical regime for two-sample detection is when the magnitude of $(p-q)$ $\sim O(n^{-1/2})$
which is asymptotically 0 as $n$ increases.

\section{Optimization experiments in Section \ref{sec:1dexample}}

\begin{table}[t]
\begin{centering}
 \stackunder[5pt]{
\begin{tabular}[t]{  l | c c c c c  }
\hline
         $n_{\text{tr}}$   & 250 			&   500			&	   1000		& 	2000  		& 	4000		  \\
 \hline
 $H$=4 		        	&  0.0135    (0.0095) &  0.0169    (0.0105)	&  0.0134    (0.0097)	&  0.0165   (0.0128)	 &    0.0152    (0.0086) \\
 $H$=8		        	&  0.0209    (0.0127) &  0.0206    (0.0110)  &  0.0264  (0.0141)	&  0.0232    (0.0139)   &   0.0261   (0.0143)  \\
 $H$=16  	 		&  0.0255    (0.0122	& 0.0315    (0.0129)	& 0.0311    (0.0139)	&  0.0355    (0.0124)   &   0.0310   (0.0149)		\\
 $H$=32  	 		&  0.0309    (0.0105)	&  0.0339   (0.0107)	& 0.0388    (0.0096)   &  0.0386   (0.0123)   &   0.0410   (0.0094)	\\
 $H$=64  	 		& 0.0324    (0.0189)	& 0.0367    (0.0086)	& 0.0416    (0.0086)	&  0.0439    (0.0081)	   &   0.0447   (0.0058) 	\\
 $H$=128  	 	& 0.0334    (0.0100)	& 0.0422    (0.0046)	&  0.0458   (0.0032)   &   0.0481   (0.0031)   &	0.0485    (0.0030)	\\
 $H$=256  	 	& 0.0297    (0.0177)	&  0.0442   (0.0053) 	&  0.0478   (0.0022) 	&   0.0503    (0.0017)   & 	0.0511    (0.0008)	\\
 $H$=512  	 	& 0.0299    (0.0123)	&  0.0433    (0.0068)	&  0.0477   (0.0040)	 &   0.0504    (0.0011)  & 	0.0508    (0.0011)	\\
 \hline
\end{tabular}
}{Example 1. Data in $\R$.}

\vspace{10pt}
 \stackunder[5pt]{
\begin{tabular}[t]{  l | c c c c c  }
\hline
         $n_{\text{tr}}$   & 250 			&   500			&	   1000		& 	2000  		& 	4000		  \\
 \hline
H=4 & 0.0182 (0.0116) & 0.0208 (0.0111) & 0.0219 (0.0115) & 0.0212 (0.0103) & 0.0189 (0.0115)  \\
H=8 & 0.0245 (0.0115) & 0.0301 (0.0117) & 0.0335 (0.0101) & 0.0283 (0.0100) & 0.0299 (0.0116)  \\ 
H=16 & 0.0240 (0.0159) & 0.0357 (0.0085) & 0.0387 (0.0083) & 0.0360 (0.0102) & 0.0389 (0.0097) \\ 
H=32 & 0.0327 (0.0086) & 0.0386 (0.0061) & 0.0427 (0.0059) & 0.0418 (0.0069) & 0.0453 (0.0034) \\ 
H=64 & 0.0322 (0.0085) & 0.0395 (0.0051) & 0.0442 (0.0021) & 0.0461 (0.0027) & 0.0461 (0.0051) \\ 
H=128 & 0.0237 (0.0202) & 0.0400 (0.0066) & 0.0452 (0.0029) & 0.0475 (0.0024) & 0.0481 (0.0023) \\ 
H=256 & 0.0234 (0.0173) & 0.0388 (0.0081) & 0.0466 (0.0025) & 0.0488 (0.0020) & 0.0498 (0.0020) \\ 
H=512 & 0.0095 (0.0239) & 0.0398 (0.0049) & 0.0471 (0.0026) & 0.0498 (0.0016) & 0.0505 (0.0009) \\ 
 \hline
\end{tabular}
}{Example 2. Data near a 1D manifold in $\R^2$}
\caption{\label{tab:eg-1d-stdL}
Mean and standard deviation (in brackets) of $L[\hat{f}_{\text{tr}}]$,
over 40 replicas of training of the neural network,
in the experiment in Section \ref{sec:1dexample} and Fig. \ref{fig:eg-1d}.
}
\end{centering}
\end{table}

The experiment is merely to optimize the loss $L_{n,tr}$ using a dataset, and numerically compute the values of 
the obtained $L[\hat{f}_{tr}]$ using the analytical formula of the densities.

The training conducts Adam for $100 \cdot \frac{8000}{n_{tr}}$ epochs
to ensure that same number of samples are processed in the experiment when $n_{tr}$ changes. 
The batch size = 100, and learning rate 1e-3.

The numerical values of the mean and standard deviation of $L[\hat{f}_{tr}]$
for various $H$ and $n_{tr}$ is shown in Table \ref{tab:eg-1d-stdL},
where the values of the mean are plotted in Fig. \ref{fig:eg-1d}.

\section{Two-sample tests on 1D data in Section \ref{subsec:exp-1d}}\label{app:exp-1d-section}

\subsection{The different test methods}\label{app:different-tests}

We consider two types of alternative two-sample tests, which are

\begin{itemize}

\item
 ({\it net-acc}) The test based on classification accuracy \cite{lopez2016revisiting}.
The equivalent form as an IPM test is explained in \ref{app:net-acc-equiv}.

\item
 ({\it gmmd}) Gaussian kernel MMD.
The kernel bandwidth $\sigma$ in {\it gmmd} is set to be the median of the pairwise distances among all samples \cite{gretton2012kernel}.
\end{itemize}
Where the three tests,
{\it net-logit} (the proposed),
{\it net-acc}
and {\it gmmd}
all use the test set for two-sample problem,
the first two network-based methods are trained on the stand-along training set. 
One may observe that this comparison to kernel MMD is not fair:
First, kernel MMD with median-distance $\sigma$ does not use the training set,
thus it would be a more fair comparison if {\it gmmd} can use all the data samples without training-test splitting. 
Second, the median setting of $\sigma$ may not be optimal 
and can be improved by existing methods in literature. 
We thus consider three more variants of the Gaussian kernel MMD (GMMD) tests

\begin{itemize}
\item
({\it gmmd}+) GMMD using all samples.
{\it gmmd} with median-distance $\sigma$ which uses all the samples  without training-test splitting.

\item
({\it gmmd-ad}) GMMD using the test set only, 
but with adaptively selected bandwidth $\sigma$ on the training set.
The selection procedure is explained in \ref{app:gmmd-ad-detail}.

\item
 ({\it gmmd}++) GMMD on the whole data sets with post-selected $\sigma$.
The tests are conducted over a range of values of $\sigma$, 
which are $\{2^{-3}, \, 2^{-2}, \cdots 2^3\}$, 
and the best test power is post-selected.
Note that this value of kernel bandwidth choice is not available in an algorithm,
and the results are for theoretical comparison only. 
\end{itemize}

These three tests are included for a more complete comparison between network-based tests and kernel tests.
Training and testing split is half-and-half, and samples in $X$ and $Y$ are of the same number, in all cases.

\subsection{ Equivalent form of {\it net-acc} test}\label{app:net-acc-equiv}

Here we show that the {\it net-acc} test studied in \cite{lopez2016revisiting} is equivalent to 
using $\text{Sign}(f_\theta)$ instead of $f_\theta$ in  \eqref{eq:def-hatT-logratio} when $n_X = n_Y$, up to multiplying and adding constants. 
Specifically, by the definition of test statistic in \cite{lopez2016revisiting}, 
and recall that $|X_{te}|  = |Y_{te}| = \frac{1}{2} |{\cal D}_{te}|$,
$\text{Sign}( z) =1$ if $z \ge 0$ and -1 if $z <0$,

\begin{equation}\label{eq:def-hatT-netacc}
\begin{split}
\hat{T}_\text{net-acc}
& = \frac{1}{2} \left(
   \frac{1}{|X_{te}|} \sum_{x \in X_{te}} \mathbf{1}_{\{ f_\theta(x) \ge 0 \}}
+ \frac{1}{|Y_{te}|} \sum_{y \in Y_{te}}  \mathbf{1}_{\{ f_\theta(y) < 0 \}}
\right) \\
& = \frac{1}{2} \left(
   \frac{1}{|X_{te}|} \sum_{x \in X_{te}} \frac{1}{2} (1 + \text{Sign}( f_\theta(x) ) )
+ \frac{1}{|Y_{te}|} \sum_{y \in Y_{te}} \frac{1}{2} (1 - \text{Sign}( f_\theta(x) ) )
\right) \\
& = \frac{1}{2} 
+ \frac{1}{4} \left(
   \frac{1}{|X_{te}|} \sum_{x \in X_{te}}  \text{Sign}( f_\theta(x) )
- \frac{1}{|Y_{te}|} \sum_{y \in Y_{te}}   \text{Sign}( f_\theta(x) )
\right) .
\end{split}
\end{equation}

\subsection{Adaptive choice of $\sigma$ in {\it gmmd-ad}}\label{app:gmmd-ad-detail}
In the training phase, 
the algorithm computes the Gaussian kernel MMD discrepancy
\[
\hat{T}_{\text{MMD}}( X, Y) = 
\frac{1}{|X|^2}\sum_{x, x' \in X} k_{\sigma}( x, x')
+
\frac{1}{|Y|^2}\sum_{y, y' \in Y} k_{\sigma}( y, y')
-
\frac{2}{|X||Y|}\sum_{x\in X,\, y \in Y} k_{\sigma}( x, y)
\]
on the training set $X=X_{\text{tr}}$, $Y=Y_{\text{tr}}$,
for a range of values of the kernel bandwidth $\sigma$, i.e. $\sigma = \{2^{-3}, \cdots, 2^3\}$.
$k_\sigma(x,y) = \exp \{-\frac{|x-y|^2}{2\sigma^2} \}$ is the isotropic Gaussian kernel. 
A plot of MMD discrepancy as a function of varying $\sigma$ is given in 
Fig. \ref{fig:training-detail-synthetic}.
The $\sigma$ which maximizes the MMD discrepancy on the training set is then chosen to compute the test statistic on the test set.
The MMD test statistic also takes the form as $\hat{T}_{\text{MMD}}$ in \cite{gretton2012kernel,li2017mmd}.

Note that the MMD loss may not be the optimal objective to select $\sigma$,
and previous works have proposed to use the estimate of testing power as the optimization objective \cite{gretton2012optimal,sutherland2016generative, liu2020learning}.
We use the MMD loss to select $\sigma$ for simplicity.
This  method is also equivalent to the training process in \cite{li2017mmd} with only one trainable parameter which is the kernel bandwidth $\sigma$.
In experiments, 
{\it gmmd-ad}  improves the two-sample test performance over median choice 
on the 2D manifold dataset in Section \ref{subsec:exp-2d}
and the MNIST dataset in Section \ref{subsec:exp-mnist}.

\subsection{Training of neural networks}

In all the experiments,
the classifier network used by {\it net-acc} and {\it net-logit}
is a two-layer fully-connected neural network with 32 hidden nodes in each hidden layer,
and the bottom layer has the same dimension as the input data. 
The training of the network is conducted via Adam \cite{kingma2014adam}.
Specifically, 100 epochs of Adam with learning rate $10^{-3}$,
and batch size $100$ when the size of training set $> 100$. 
A typical plot of evolution of training loss and training error is given in Fig. \ref{fig:training-detail-synthetic}.
Training via SGD with momentum $0.9$ produces similar result.
The result is qualitatively the same when the number of hidden units varies from $16$ to $1024$.
We have not investigated the optimal choice of network architecture hyperparameters for the two-sample problem. 

We use fixed learning rate over a fixed number of epochs,
and it is entirely possible that our training procedure is over simplified and 
better usage of stochastic gradient descent method as studied in
\cite{bottou2010large,zeiler2012adadelta,sutskever2013importance,shamir2013stochastic,du2018gradient}
may lead to improved performance.

\begin{figure}
\hspace{-0.3cm}
\includegraphics[height=0.19\textwidth]{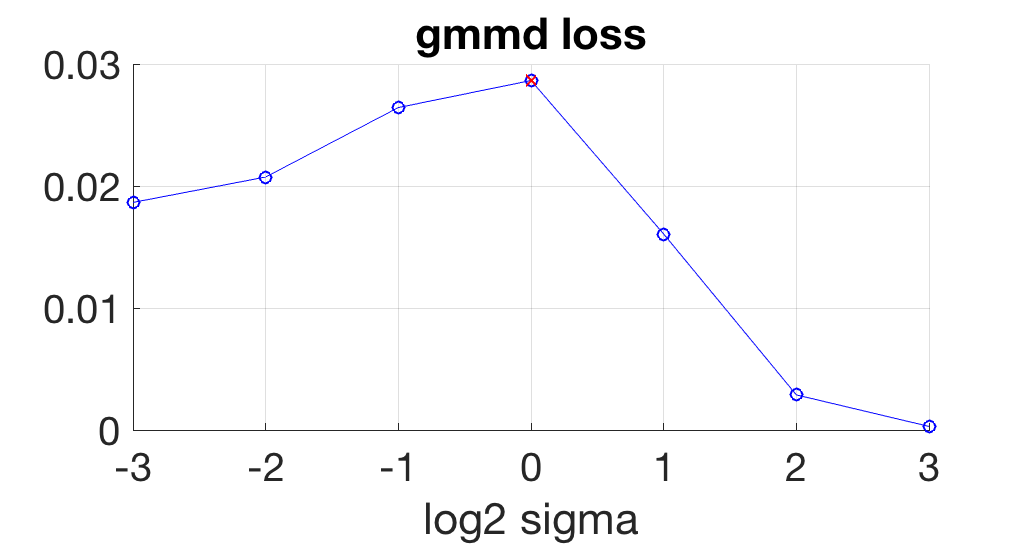}
\hspace{-0.5cm}
\includegraphics[height=0.2\textwidth]{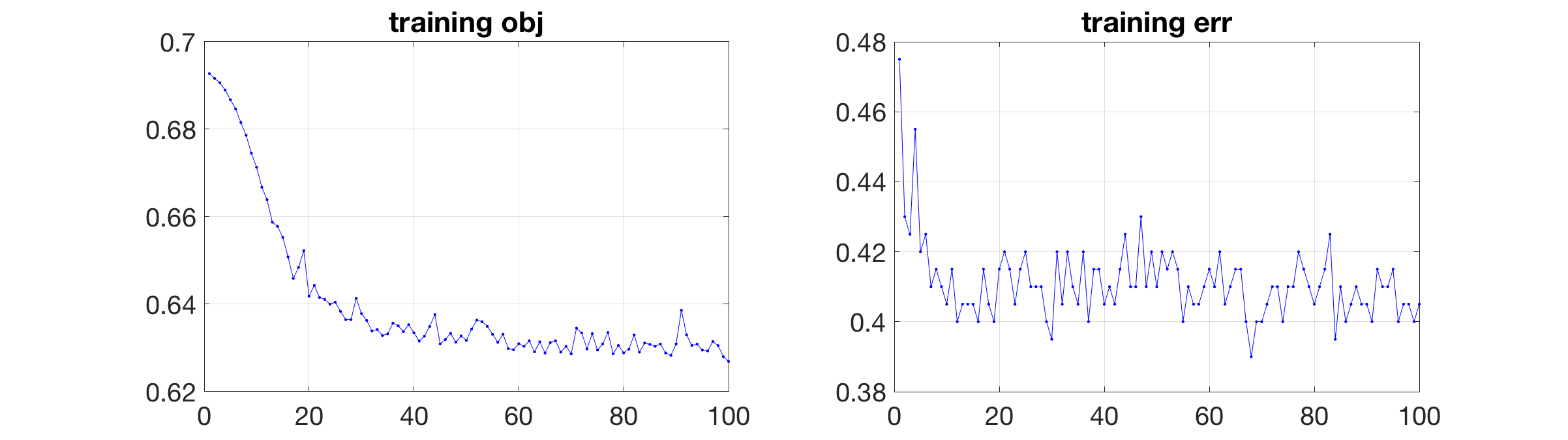}
\caption{
Left:
MMD discrepancy on trained set used by {\it gmmd-ad} to select kernel bandwidth $\sigma$.
Middle and right:
training of the classification network used in net-based tests.
On the synthetic 1D dataset.
}\label{fig:training-detail-synthetic}
\end{figure}

\subsection{Test power estimation}\label{app:exp-test-power-randomness}

\begin{figure}
\includegraphics[height=0.2\textwidth]{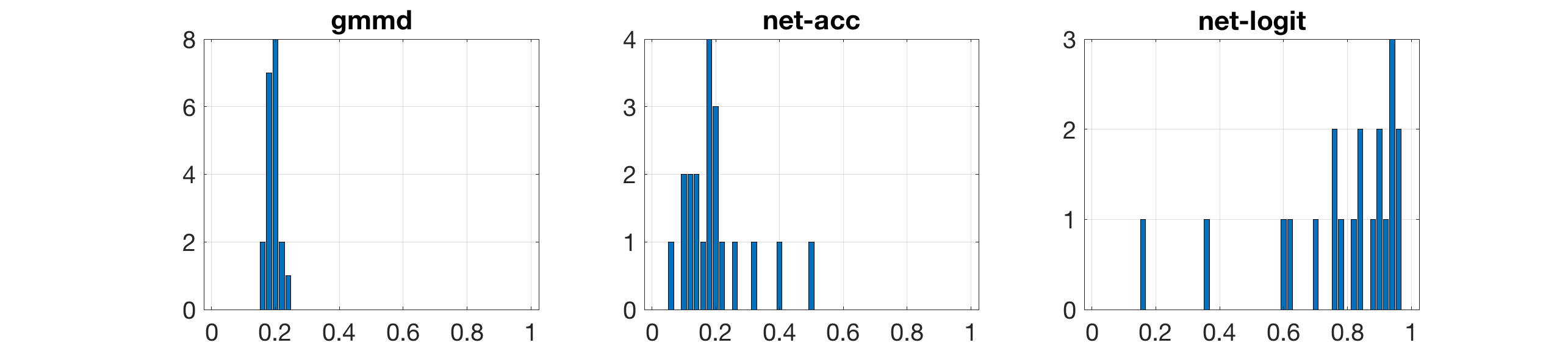}
\caption{
Histogram of estimated test power from 400 test runs
of {\it gmmd}, {\it net-acc} and {\it net-logit} over 20 replicas of training (no training for  {\it gmmd}),
on the example in Fig. \ref{fig:1dexample}.
}\label{fig:hist-power-test7}
\end{figure}

The experiments with tests (1)-(4) 
use $n_\text{run}=400$ test runs to estimate the power as the frequency of rejecting ${\cal H}_0$,
 and the whole experiments are repeated for $n_{\text{rep}}=20$ replicas 
 to compute mean and  standard deviation of the estimated power. 
The test with (5) and (6) uses 200 test runs to estimate the power,
since these gmmd methods demonstrate less variation in estimated power,
explained as below.

The way of computing the test power is empirical and has randomness:
for kernel mmd the variation is due to the finite number of runs ($n_\text{run}$ times),
and for network based tests there is extra variation due to the stochastic optimization of the network.
Thus we use experiment replicas to recored the variations of the test power.
The empirical distribution of the test power of the three methods over training replicas is given in \ref{fig:hist-power-test7}, 
corresponding to the experiment in Fig. \ref{fig:1dexample}.
The plots of the two net-based methods
indicate large variation of the power given by each trained network,
that is, the ``quality'' of the trained net to discriminate the two densities varies.
This instability is due to limited training samples
as well as the randomness in the optimization algorithm. 

We observe decreased power variation with larger training set,
and the trained network gives better two-sample test power.
This is consistent with the observation in Section \ref{sec:1dexample},
and indicates that larger training set benefits the network training.
However, 
as a price to pay, 
the testing set will be smaller given finitely many samples in total.

\subsection{Detailed experimental results on Eg. 3}\label{app:exp-1d-eg3-detail}

{\bf  Set-up.}
 In Eg.3, 
the number $\delta \in [0,1]$ controls the difference between the densities $p$ and $q$.
A plot for $p$ and $q$
with $\delta = 0.08$ is illustrated in Fig. \ref{fig:1dexample}.
200 training and 200 testing samples  are used, with half of the samples coming from $X$ and half from $Y$ in both the training and testing sets.

{\bf Test power.}
The table in Fig. \ref{fig:1dexample} lists the power for the three methods (1)(2)(3),
where {\it net-logit} gives significantly better average power $\sim$ 80\%, 
and the power of {\it net-acc} and {\it gmmd} are similar, both $\sim$ 20\%.
Table \ref{tab:1dexample-power} gives the full table of test power including that of the methods {\it gmmd}+ and {\it gmmd}++.
where {\it gmmd}+ achieves a test power of 47\%
and
{\it gmmd}++  
a power of 57\%, 
remaining inferior to {\it net-logit},
while both with small variation (std $\lesssim$ 2) and thus are more stable than net-based tests. 
Results with other values of $\delta$ and sample sizes are reported in Fig. \ref{fig:test5-eg3}.

The variation of the power is much larger for the two net-based tests,
as explained in \ref{app:exp-test-power-randomness} and  Fig. \ref{fig:hist-power-test7}.
We note that such large variation is due to the instability of network training at small training size, 
and is likely to be a limitation of the current net-based methods.

\begin{table}
\begin{centering}
\begin{tabular}[t]{ c | c c c| c c  }
\hline
                           	& {\it gmmd} 	&   {\it gmmd}+	&   {\it gmmd}++	& 	{\it net-acc}  		& {\it net-logit}  \\
 \hline
mean 	         	& 19.14		&	46.63		&  	57.29		& 19.98  			& 	78.09 		 	\\
std		         	&  1.95 		&	2.49			& 	1.578		& 10.43			&	 20.56		\\
median  	 		& 19.63		&	47.13		& 	57.38		& 17.63			& 	84.13	 		\\
 \hline
\end{tabular}
\caption{\label{tab:1dexample-power}
The mean, standard deviation (``std") and median of the test power of the various methods
computed from $n_{run}=400$ test runs over $n_{\text{rep}}=20$ replicas 
on Eg. 3 in Section \ref{subsec:exp-1d}.
The 
{\it gmmd}, {\it net-acc},  {\it net-logit} 
tests are computed  on 
$|X_\text{te}|= |Y_\text{te}|=100$ samples,
where {\it net-acc}  and  {\it net-logit} train a classification network on another training set of size $|X_\text{tr}|= |Y_\text{tr}|=100$.
{\it gmmd} only uses the test set and sets the kernel bandwidth $\sigma$ to be the median distance.
{\it gmmd+} and {\it gmmd++} accesses both the training and test sets,
where {\it gmmd+} uses the median distance as  $\sigma$,
and {\it gmmd++}  reports the best power over varying range of choices of $\sigma$,
as described in \ref{app:different-tests}.
The results of {\it gmmd}, {\it net-acc}, {\it net-logit} are also reported in Fig. \ref{fig:1dexample}. 
}
\end{centering}
\end{table}

\section{Two-sample tests on 2D data in Section \ref{subsec:exp-2d}}\label{app:detail-exp-2d}

The construction of $x_i \sim p$ and $y_j \sim q$ are as below:
$x_i= T(u_i)$, $y_j = T(v_j)$ where $T: \R^2 \to \R^3$ is a smooth mapping from unit square to the spherical surface given by 
\[
T(x_1, x_2) = \frac{1}{R}\left( x_1, x_2, \sqrt{R^2-x_1^2-x_2^2} \right),
\quad
R=1.5,
\]
and $u_i$, $v_j$ are i.i.d. copies of random variables $u$ and $v$ in $\R^2$ distributed as
\begin{eqnarray*}
& u = t_u + \eta_u,  
\quad
v = t_v + \eta_v,   
\quad
\eta_u, \, \eta_v \sim {\cal N}(0, \epsilon^2 I_2), \quad \epsilon = 0.05,
\\
& t_u  \sim \text{ uniformly on a quater circle in $[0,1]\times [0,1]$},
\\
& t_v  \sim \text{ the distribution of $t_u$ rotated around $(\frac{1}{\sqrt{2}}, \frac{1}{\sqrt{2}})$ by angle $\delta$},
\end{eqnarray*}
where the 4 random variables are all independent.

The other experimental set-ups are the same as in Section \ref{subsec:exp-1d}.

\section{Two-sample tests on MNIST data in Section \ref{subsec:exp-mnist}}\label{app:mnist}

\begin{figure}[t]
\begin{center}
\includegraphics[width=0.9\textwidth]{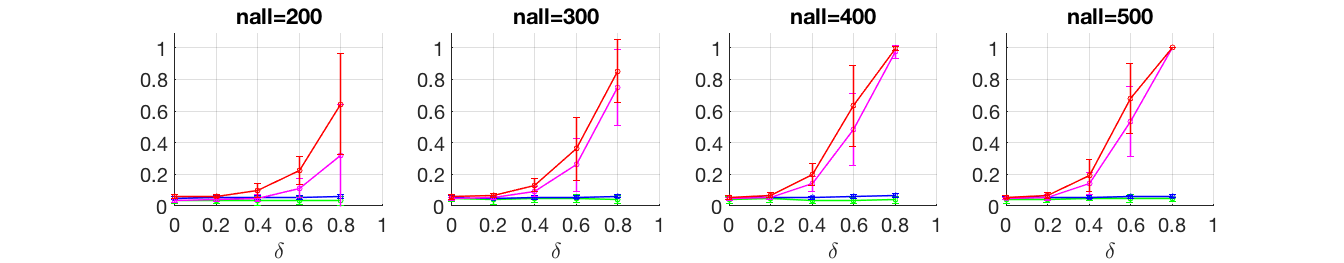}
\end{center}
\caption{
Same plot as Fig. \ref{fig:mnist-power}
with another pre-trained generative model 
which produces faked images that are closer to the authentic ones. 
}\label{fig:mnist-power-dim128}
\end{figure}

\begin{figure}
\hspace{-0.2cm}
\includegraphics[height=0.19\textwidth]{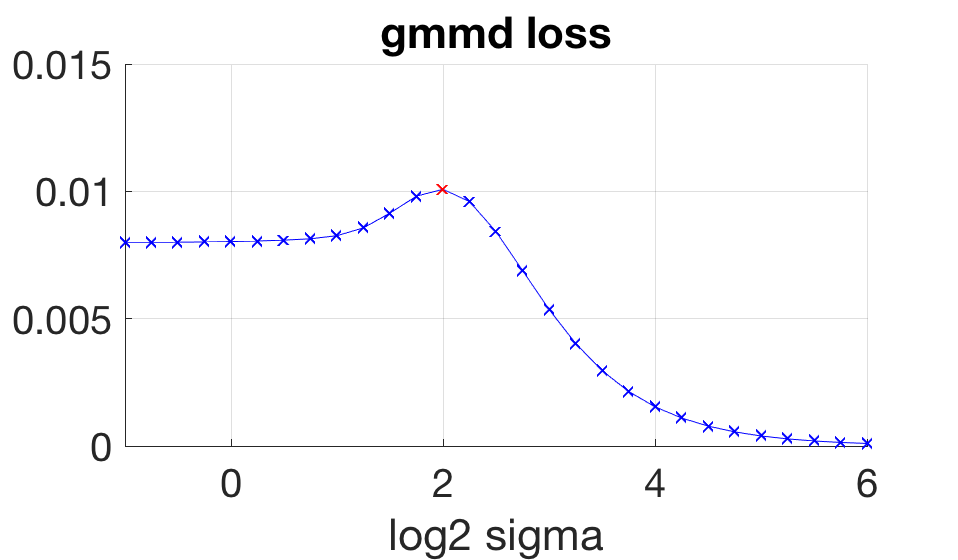}
\hspace{-0.5cm}
\includegraphics[height=0.2\textwidth]{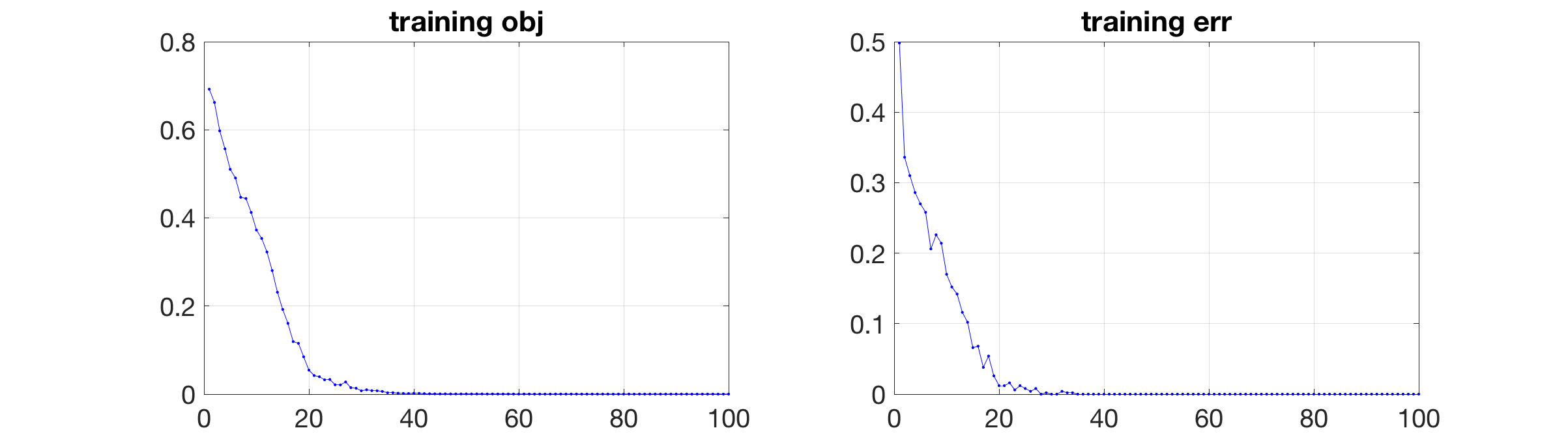}
\caption{
Same plot as Fig. \ref{fig:training-detail-synthetic}
on MNIST data.
}\label{fig:training-detail-mnist}
\end{figure}

The  classifier network used in the experiment 
is a convolutional neural network (CNN) with 2 convolutional layers. 

The pre-trained generated model is based on a convolutional auto-encoder:
\begin{itemize}
\item[]
c5x5x1x16 - re - ap 2x2 - c5x5x16x32 - re - ap 2x2 - fc128 - re  \\
- fc10 - re  ~~~ $\leftarrow$ {\bf code space $\R^{10}$} \\
- fc128 - re - ct 5x5x128x32  - re  \\
-  ct5x5x32x16 (upsample 2x2) - re  - ct5x5x16x1 (upsample 2x2) - Euclidean loss
\end{itemize}
where ``c'' stands for convolutional layer,  ``ct'' for transposed convolutional layers,
 ``re'' for Relu activation, and ``ap'' for average pooling.
The auto-encoder is trained on 50000 MNIST dataset for 20 epochs using Adam 
with learning rate decreasing from $10^{-3}$ to $10^{-6}$ and batch size $100$.

The sampling of generative model is conducted by adding a small isotropic Gaussian noise 
(``giggering'') 
to the 10-dimensional codes of authentic MNIST digits computed by an encoder,
and then mapping through the decoder to $\R^{784}$.

We also prepare another generative model
by removing the bottleneck layer in the above auto-encoder  architecture and retrain the model,
which gives smaller reconstruction error
and a higher-dimensional code space of $\R^{128}$.
The generative model is conducted in the same way by sampling
 in the code space using Gaussian noise of smaller 
variance per coordinate. 
This produces faked images that are closer to the authentic ones in Euclidean distance in $\R^{784}$, 
however less explore the ``manifold'' of $p_{\text{data}}$.
The test power of the four methods is shown in Fig. \ref{fig:mnist-power-dim128}.

The classification network used in {\it net-logit} is the following CNN
\begin{itemize}
\item[]
c5x5x1x16 - re - ap 2x2 \\
- c5x5x16x32 - re - ap 2x2 \\
 - fc128 - re  - fc2 - softmax loss
\end{itemize}
where dropout is used between the last 2 fully-connected layers.
The classification CNN is trained for 100 epochs using Adam  with learning rate  $10^{-3}$
and batch size $100$. 
A typical plot of evolution of training loss and training error is given in Fig. \ref{fig:training-detail-mnist}.

The procedure of adaptive selection of $\sigma$ by {\it gmmd-ad} is same as in Section \ref{subsec:exp-1d},
where the bandwidth search range is $\sigma = \{2^{-1}, \cdots, 2^6\}$.
The test power is evaluated on $400$ test runs and the training is repeated for $20$ replicas.

\end{document}